\documentclass[letterpaper, 10 pt, journal, twoside]{IEEEtran}  % Use this command for final RAL version

\IEEEoverridecommandlockouts                              % This command is only needed if
\usepackage{url}

\let\llncssubparagraph\subparagraph
\let\subparagraph\paragraph
\usepackage[compact]{titlesec}
\let\subparagraph\llncssubparagraph

\usepackage{comment}
\usepackage{siunitx}
\usepackage{relsize}
\usepackage{ifthen}

\usepackage[caption=false]{subfig}

\usepackage[vlined,ruled,linesnumbered]{algorithm2e}
\usepackage{graphics} % for pdf, bitmapped graphics files
\usepackage{rotating}
\usepackage{color}
\usepackage{enumerate}
\usepackage[T1]{fontenc}
\usepackage{psfrag}
\usepackage{epsfig} % for postscript graphics files
\usepackage{booktabs}
\usepackage{graphicx,url}
\usepackage{multirow}
\usepackage{array}
\usepackage{latexsym}
\usepackage{amsfonts}
\usepackage{amsmath}
\usepackage{amssymb}
\usepackage{xstring}
\usepackage[noend]{algorithmic}
\usepackage{multirow}
\usepackage{xcolor}
\usepackage{prettyref}
\usepackage{flexisym}
\usepackage{bigdelim}
\usepackage{breqn} % load this last
\usepackage{listings}
\let\labelindent\relax
\usepackage{enumitem}
\usepackage{xspace}
\usepackage{bm}
\graphicspath{{./figures/}}
\usepackage{tikz}
\usetikzlibrary{matrix,calc}

\usepackage{mdwlist}

\let\itemize\undefined
\makecompactlist{itemize}{stditemize}

\newrefformat{prob}{Problem\,\ref{#1}}
\newrefformat{def}{Definition\,\ref{#1}}
\newrefformat{sec}{Section\,\ref{#1}}
\newrefformat{sub}{Section\,\ref{#1}}
\newrefformat{prop}{Proposition\,\ref{#1}}
\newrefformat{app}{Appendix\,\ref{#1}}
\newrefformat{alg}{Algorithm\,\ref{#1}}
\newrefformat{cor}{Corollary\,\ref{#1}}
\newrefformat{thm}{Theorem\,\ref{#1}}
\newrefformat{lem}{Lemma\,\ref{#1}}
\newrefformat{fig}{Fig.\,\ref{#1}}
\newrefformat{tab}{Table\,\ref{#1}}

\newtheorem{theorem}{Theorem}
\newtheorem{proposition}[theorem]{Proposition}
\newcommand{\cf}{\emph{cf.}\xspace}

\newcommand{\bdmath}{\begin{dmath}}
\newcommand{\edmath}{\end{dmath}}
\newcommand{\beq}{\begin{equation}}
\newcommand{\eeq}{\end{equation}}
\newcommand{\bdm}{\begin{displaymath}}
\newcommand{\edm}{\end{displaymath}}
\newcommand{\bea}{\begin{eqnarray}}
\newcommand{\eea}{\end{eqnarray}}
\newcommand{\beal}{\beq \begin{array}{ll}}
\newcommand{\eeal}{\end{array} \eeq}
\newcommand{\beas}{\begin{eqnarray*}}
\newcommand{\eeas}{\end{eqnarray*}}
\newcommand{\ba}{\begin{array}}
\newcommand{\ea}{\end{array}}
\newcommand{\bit}{\begin{itemize}}
\newcommand{\eit}{\end{itemize}}
\newcommand{\ben}{\begin{enumerate}}
\newcommand{\een}{\end{enumerate}}

\newcommand{\calB}{{\cal B}}

\newcommand{\calE}{{\cal E}}

\newcommand{\calI}{{\cal I}}

\newcommand{\calL}{{\cal L}}

\newcommand{\calU}{{\cal U}}
\newcommand{\calV}{{\cal V}}

\newcommand{\setal}{~\emph{et~al.}\xspace}
\newcommand{\M}[1]{{\bm #1}} % Face for matrices
\renewcommand{\boldsymbol}[1]{{\bm #1}}
\newcommand{\hide}[1]{}

\newcommand{\tocheck}[1]{{\color{brown} #1}}
\newcommand{\grayout}[1]{{\color{gray} #1}}

\newcommand{\hiddenText}{{\color{gray} hidden text.}}
\newcommand{\hideWithText}[1]{\hiddenText}

\newcommand{\kron}{\otimes}

\newcommand{\subject}{\text{ subject to }}
\DeclareMathOperator*{\argmax}{arg\,max}
\DeclareMathOperator*{\argmin}{arg\,min}

\newcommand{\prob}[1]{{\mathbb P}\left(#1\right)}
\newcommand{\tran}{^{\mathsf{T}}}

\newcommand{\diag}[1]{\mathrm{diag}\left(#1\right)}
\newcommand{\trace}[1]{\mathrm{tr}\left(#1\right)}

\newcommand{\rank}[1]{\mathrm{rank}\left(#1\right)}

\newcommand{\ones}{{\mathbf 1}}
\newcommand{\zero}{{\mathbf 0}}
\newcommand{\eye}{{\mathbf I}}

\newcommand{\matTwo}[1]{\left[\begin{array}{cc}  #1  \end{array}\right]}

\newcommand{\Real}[1]{ { {\mathbb R}^{#1} } }

\newcommand{\at}[1]{^{(#1)}}

\newcommand{\setdef}[2]{ \{#1 \; {:} \; #2 \} }

\newcommand{\MA}{\M{A}}

\newcommand{\MG}{\M{G}}
\newcommand{\MM}{\M{M}}

\newcommand{\MP}{\M{P}}
\newcommand{\MQ}{\M{Q}}
\newcommand{\MU}{\M{U}}
\newcommand{\MR}{\M{R}}

\newcommand{\MV}{\M{V}}

\newcommand{\MH}{\M{H}}
\newcommand{\ML}{\M{L}}

\newcommand{\MX}{\M{X}}
\newcommand{\MY}{\M{Y}}

\newcommand{\MZ}{\M{Z}}

\newcommand{\vb}{\boldsymbol{b}}

\newcommand{\ve}{\boldsymbol{e}}

\newcommand{\vu}{\boldsymbol{u}}

\newcommand{\vxx}{\boldsymbol{x}} 
\newcommand{\vy}{\boldsymbol{y}}
\newcommand{\vw}{\boldsymbol{w}}
\newcommand{\vzz}{\boldsymbol{z}}

\newcommand{\vlambda}{\boldsymbol{\lambda}}

\newcommand{\scenario}[1]{{\smaller \sf#1}\xspace}

\newcommand{\CVX}{\scenario{CVX}}

\newcommand{\cvx}{{\sf cvx}\xspace}

\newcommand{\blue}[1]{{\color{blue}#1}}

\newcommand{\linkToPdf}[1]{\href{#1}{\blue{(pdf)}}}
\newcommand{\linkToPpt}[1]{\href{#1}{\blue{(ppt)}}}
\newcommand{\linkToCode}[1]{\href{#1}{\blue{(code)}}}
\newcommand{\linkToWeb}[1]{\href{#1}{\blue{(web)}}}
\newcommand{\linkToVideo}[1]{\href{#1}{\blue{(video)}}}
\newcommand{\award}[1]{\xspace} % {{\red{#1}}} % omit awards

\usepackage{float} 
\usepackage{hyperref}
\usepackage{graphicx}
\usepackage{epstopdf}
\usepackage{epsfig}
\usepackage{xr}

\newcommand{\eg}{\emph{e.g.,}\xspace}

\newcommand{\myparagraph}[1]{{\bf #1.}}

\newcommand{\todo}[1]{{(\color{red}TODO: #1)}}

\newcommand{\class}{x}
\newcommand{\classes}{\vxx}

\newcommand{\Classes}{\M{X}}

\newcommand{\Classesmp}{\vw} % minus/plus

\newcommand{\measuredClass}{\bar{x}}
\newcommand{\measured}{\vzz}
\newcommand{\nodeSet}{\calV}
\newcommand{\nrNodes}{N}

\newcommand{\unarySet}{\calU}
\newcommand{\binarySet}{\calB}

\newcommand{\classSet}{\calL}
\newcommand{\classesOpt}{\vxx^\star}
\newcommand{\nrClasses}{K}

\newcommand{\penaltyTerm}{\bar{\delta}}

\newcommand{\Sone}{S1}
\newcommand{\Stwo}{S2}

\newcommand{\fsdpopt}{f_\text{{\smaller \Sone}}^\star}

\newcommand{\fsdpoptTwo}{f_\text{{\smaller \Stwo}}^\star}
\newcommand{\fsdpround}{\hat{f}_\text{{\smaller \Sone}}^\star}
\newcommand{\fsdproundTwo}{\hat{f}_\text{{\smaller \Stwo}}^\star}

\newcommand{\MZopt}{\MZ^\star_{\Sone}}
\newcommand{\MZoptTwo}{\MZ^\star_{\Stwo}}

\newcommand{\MXrel}{\MX_{\Sone}}
\newcommand{\MXrelTwo}{\MX_{\Stwo}}

\newcommand{\MXopt}{\MX^\star}

\newcommand{\MXrounded}{\hat{\MX}_{\Sone}}
\newcommand{\MXroundedTwo}{\hat{\MX}_{\Stwo}}

\newcommand{\Stiefel}[2]{\text{St}(#1,#2)}

\newcommand{\omitted}[1]{}

\newcommand{\frob}{{\small \text{F}}}

\newcommand{\dars}{\scenario{DARS}}
\newcommand{\fuses}{\scenario{FUSES}}
\newcommand{\LBP}{\scenario{LBP}}
\newcommand{\TRW}{\scenario{TRW-S}}
\newcommand{\aexp}{\scenario{$\alpha$-exp}}
\newcommand{\openGM}{OpenGM2\xspace}
\newcommand{\CPLEX}{CPLEX\xspace}

\newcommand{\extendedVersion}[1]{}

\newcommand{\Pmrf}{P0}
\newcommand{\Pmp}{P1}
\newcommand{\Pzo}{P2}

\newcommand{\Smp}{S1}
\newcommand{\Szo}{S2}

\newcommand{\Rmp}{R1}
\newcommand{\Rzo}{R2}

\newcommand{\foptmp}{f^\star_{\Pmp}}
\newcommand{\foptzo}{f^\star_{\Pzo}}
\newcommand{\fsdpmp}{f^\star_{\Smp}}
\newcommand{\fsdpzo}{f^\star_{\Szo}}

\newcommand{\froundedmp}{\hat{f}_{\Smp}}
\newcommand{\froundedzo}{\hat{f}_{\Szo}}

\newcommand{\Fig}[1]{Fig.~\ref{#1}}
\newcommand{\Table}[1]{Table~\ref{#1}}

\newcommand{\percSubopt}{0.1\%\xspace}
\newcommand{\zoomTime}{18ms}
\newcommand{\claimTime}{faster than\xspace} %{as fast as}

\newcommand{\cumsum}{2 - \nrClasses}
\newcommand{\KpTwo}{\nrClasses - 2} 

\newcommand{\modelOne}{$512\!\times\!256$\xspace}
\newcommand{\modelTwo}{$1024\!\times\!512$\xspace}

\newcommand{\isExtended}[2]{#1} % #1 : long arxiv version
\renewcommand{\baselinestretch}{0.98}

\newcommand{\change}[2]{{\color{black}#2}} % \change{old version}{new version}

\begin{document}
%!TEX root = main.tex
%
\title{\LARGE \bf
	Accelerated Inference in Markov Random Fields \\ 
	via Smooth Riemannian Optimization
}
% \author{Siyi Hu$^{1}$ and Luca Carlone$^{2}$% <-this % stops a space
% 	\thanks{$^{1}$ Siyi Hu with the Department of Mechanical Engineering, Laboratory for Information and Decision Systems, Massachusetts Institute of Technology, Cambridge, MA 02139, USA
% 	{\tt\small siyi@mit.edu}}%
% 	\thanks{$^{2} $Luca Carlone is with Faculty of Aeronautics and Astronautics Deparment, Laboratory for Information and Decision Systems, Massachusetts Institute of Technology, Cambridge, MA 02139, USA
% 	{\tt\small lcarlone@mit.edu}}%
% }

\author{Siyi Hu and Luca Carlone% <-this % stops a space
	\thanks{S.\,Hu and L.\,Carlone are with the Laboratory for Information and Decision Systems, Massachusetts Institute of Technology, Cambridge, MA 02139, USA
	{\tt\small \{siyi,lcarlone\}@mit.edu}}%
	\thanks{This work was partially funded by MIT Lincoln Laboratory under the grant ``Fast Semantic Segmentation on Manifold''.}
}

\isExtended{}{
% Paper headers
\markboth{IEEE Robotics and Automation Letters. Preprint Version.}
{Hu and Carlone: Accelerated MRFs} %  \MakeLowercase{\textit{et al.}}
% Use only for final RAL version
}

\maketitle              % typeset the title of the contribution

\begin{abstract}
\emph{Markov Random Fields} (MRFs) are a popular model for %is the backbone of 
several pattern recognition and reconstruction 
problems in robotics and computer vision. 
% We propose a fast algorithm for maximum a-posteriori inference on Markov Random Fields (MRFs). 
Inference in MRFs is intractable in general and related work resorts to approximation algorithms. 
Among those techniques, semidefinite programming (SDP) relaxations have been shown to provide accurate estimates while 
scaling poorly with the problem size and being typically slow for practical applications.  
Our first contribution is to design a dual ascent method to solve standard SDP relaxations that 
takes advantage of the geometric structure of the problem to speed up computation. 
%enabling the use of the \emph{Riemannian Staircase}. % to improve scalability. 
This technique, named \emph{Dual Ascent Riemannian Staircase} (\dars), is able to solve large problem instances \emph{in seconds}. 
\change{Since our goal is to enable \change{}{SDP based} real-time inference on robotics platforms, we} 
{Our second contribution is to develop a second and faster approach.} 
The backbone of this second approach is a novel SDP relaxation combined with a fast and scalable solver 
based on smooth Riemannian optimization. We show that this approach, 
named \emph{Fast Unconstrained SEmidefinite Solver} (\fuses), can solve large problems \emph{in milliseconds}. 
 % while providing per-instance suboptimality guarantees.
% Our second and main contribution is the development of a no
% where the 
% inner iterations use a smooth Riemmannian optimization 
%
% We propose a novel SDP relaxation and develop a fast and scalable solver based on smooth Riemmannian optimization. 
% We show that this technique is as fast as local optimization methods and provides per-instance suboptimality guarantees.
% Leveraging the same machinery, we provide a solver for the standard SDP relaxation that can solve
Contrarily to local MRF solvers, \eg loopy belief propagation, our approaches do not require an initial guess. Moreover, 
we leverage recent results from optimization theory to provide per-instance sub-optimality guarantees. 
% As a further contribution, 
We demonstrate the proposed approaches in multi-class image segmentation problems. 
Extensive experimental evidence shows that 
(i) \fuses and \dars produce near-optimal solutions, attaining an objective within \percSubopt of the optimum,
%with a suboptimality gap between \tocheck{1-2\%} from the optimum,
(ii) \fuses and \dars are remarkably faster than general-purpose SDP solvers, 
and \fuses is more than two orders of magnitude faster than \dars while attaining similar solution quality,
(iii) \change{\fuses is \claimTime local search methods while being a global solver, and is a good candidate for real-time robotics applications.}{\fuses is \claimTime local search methods while being a global solver.}
% our approach outperforms local search methods (LBP) as well as 
% linear relaxations methods (roof duality), (ii) the accuracy of the proposed approach is comparable with the ones of standard SDP solvers, while the latter 
% imply an unacceptable computational burden, (iii) the computational effort of the proposed approach makes it a good candidate for real-time embedded applications.
% Moreover, we consider applications of our approach to the problem of multi-class semantic segmentation. 
%used for labeling or parsing of sequential data, such as natural language processing or biological sequences[1] and in computer vision
% \bit
% \item 1 sentence: what is the problem (high-level) and why its important
% \item 1 sentence: related work is lacking / has these limitators
% \item The contribution of this paper solve these limitations. List of contributions
% \eit
\end{abstract}

% \keywords{Perception, Vision and Sensor-based Control, Optimization and Optimal Control}
%!TEX root = main.tex

\section{Introduction}
\label{sec:intro}

% \bit
% \item what is the problem (high-level) and why its important
% \item brief review of related works
% \item related work is lacking / has these limitators
% \item The contribution of this paper solve these limitations. List of contributions
% \item (structure of the paper)
% \eit
% Markov Random Fields (MRFs) are a popular graphical model for reconstruction and recognition problems in computer vision and robotics.
\emph{Markov Random Fields} (MRFs) are a popular graphical model for reconstruction and recognition problems in computer vision and robotics, 
including 2D and 3D semantic segmentation, 
stereo reconstruction, image restoration and denoising, 
texture synthesis, object detection, and panorama stitching~\cite{Szeliski08pami-surveyMRF,Blake11book-MRF,Kappes15ijcv-energyMin}.
% These problems typically require to compute the most likely assignment of discrete variables that minimizes 
% a given cost function. 
% Inference over MRFs 
%  is intractable in general: the computation of the most likely  (maximum)
% is intractable in general and related work resorts to 
An MRF can be understood as a factor graph including only unary and binary factors, and where node variables are discrete labels. 
\change{The discrete nature of the variables makes \emph{maximum a posteriori} (MAP) inference in MRFs intractable in general, 
hence several MRF-based applications remain out of reach for real-time robotics.
Our motivating application in this paper is real-time semantic segmentation, which is crucial for the robot to 
understand the surrounding environment and execute high-level tasks.}{The discrete nature of the variables makes \emph{maximum a posteriori} (MAP) inference in MRFs intractable in general, 
and this clashes with the need for real-time inference that characterizes several robotics applications (e.g., semantic understanding, mapping).} 
\change{\emph{Therefore, we are interested in developing real-time MRF solvers 
that can support online operation at scale}.}{ % in large-scale environments. % understanding. 
%
% \begin{figure}[t]
% \begin{center}
%     \includegraphics[width=0.15\columnwidth]{MRFs}
%     \caption{MRFs}
%     \label{fig:MRFs}
% \end{center}
% \end{figure}

The literature on MRFs (reviewed in Section~\ref{sec:relatedWork}) is vast and includes methods based on graph cuts, message passing techniques, 
greedy methods, and convex relaxations, to mention a few. These approaches are typically approximation techniques, in the sense that they attempt to 
compute near-optimal MAP estimates efficiently (the problem is NP-hard in general, hence we do not expect to compute exact solutions in polynomial time).}
Among those, semidefinite programming (SDP) relaxations have been recognized to produce accurate approximations~\cite{Kumar08nips}.
On the other hand, the computational cost of general-purpose SDP solvers prevented widespread use of this technique beyond problems with few hundred  
variables~\cite{Keuchel03pami} (semantic segmentation typically involves thousands to millions of variables), and SDPs lost popularity in favor of 
computationally cheaper alternatives including move-making algorithms (based on graph cut) and message passing. 
Move-making methods~\cite{Boykov01pami-graphCut} require 
specific assumptions on the MRF and their performance typically degrades when these assumptions are not satisfied.
\change{}{Message passing methods~\cite{Wainwright05itis,Weiss01itis-beliefPropagation}, on the other hand, may not even converge, even thought they are observed to work very well in practice.}

% %%%%%%%%%%%%%%%%%%%%%%%%%%%%%%%%%%%%%%%%%%%%%%%%
%!TEX root = main.tex

\newcommand{\myhspace}{\hspace{-3mm}}

\newcommand{\mpw}{4.5cm}
\begin{figure}[t]
%\vspace{-5mm}
	\begin{center}
	\begin{minipage}{\textwidth}
	\hspace{-0.2cm}
%		\myhspace
%		\begin{minipage}{0.225\textwidth}%
%			\centering
%			\includegraphics[width=1.0\columnwidth]{overlay/bonnet_overlay} \\
%			(a) Bonnet 
%		\end{minipage}
%	\\
	\begin{tabular}{cc}%
		\myhspace
			\begin{minipage}{\mpw}%
			\centering%
			\includegraphics[width=0.95\columnwidth]{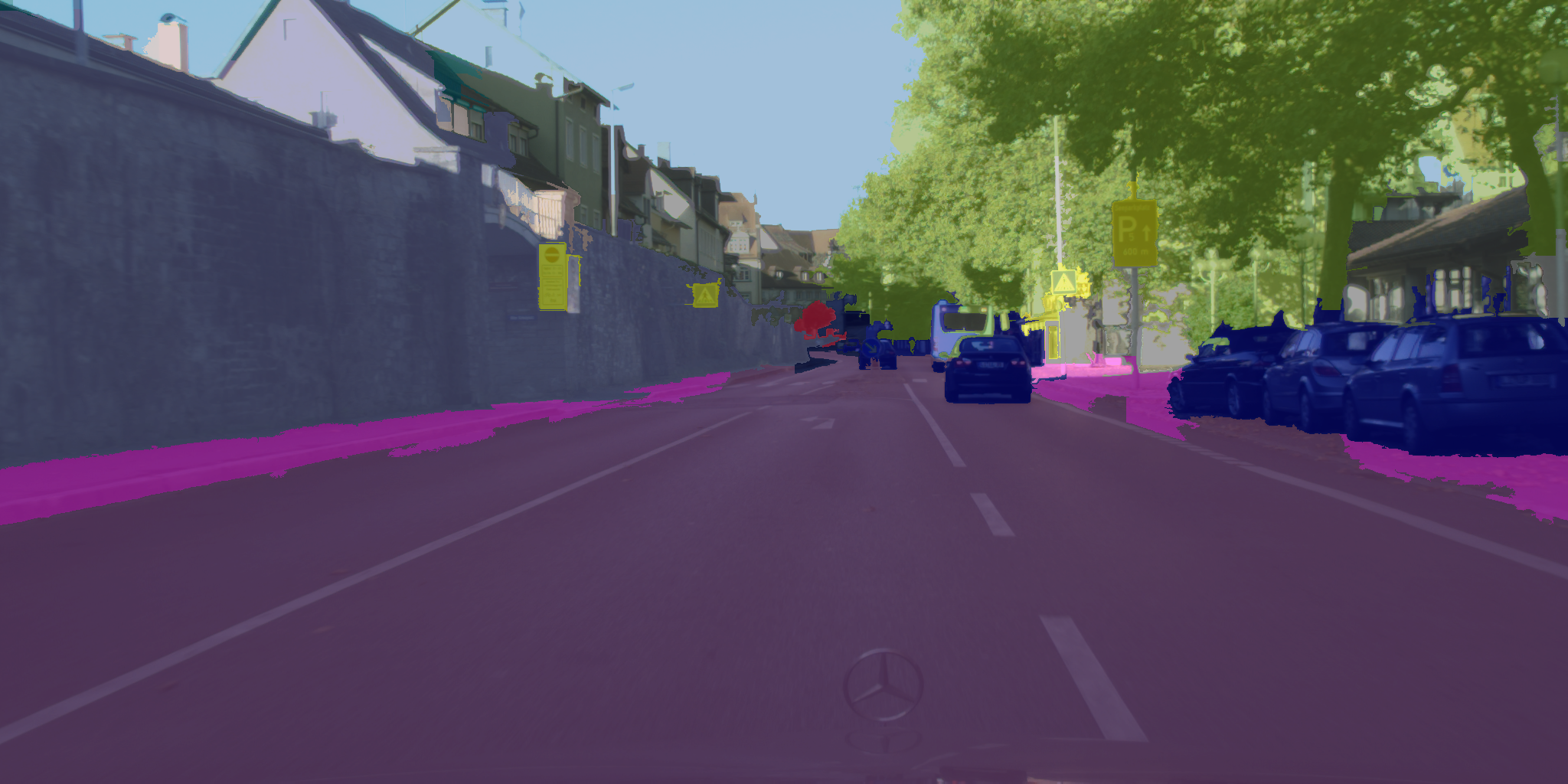} \\
			(a) \fuses  
			\end{minipage}
		& \myhspace
			\begin{minipage}{\mpw}%
			\centering%
			\includegraphics[width=0.95\columnwidth]{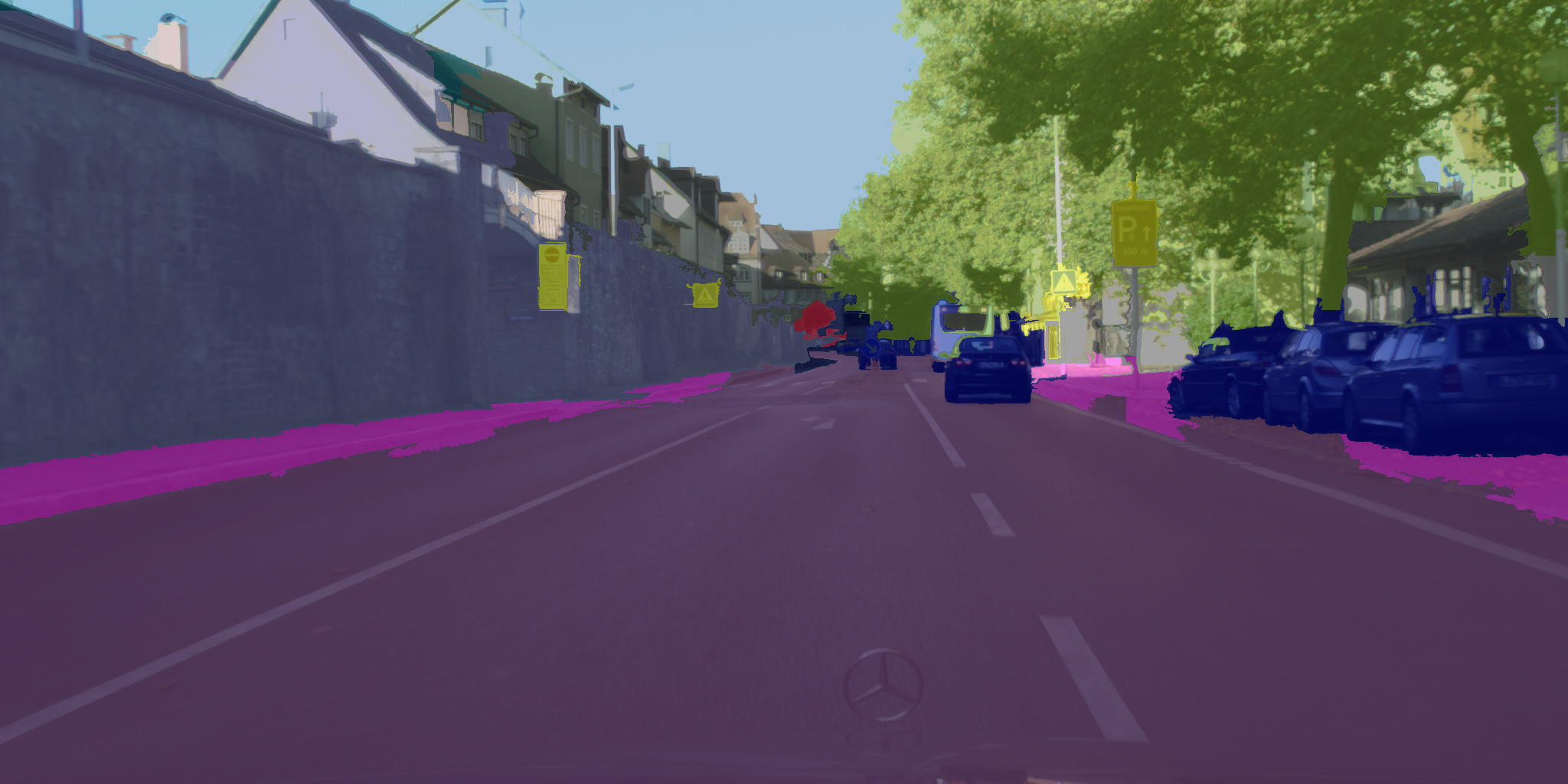} \\
			(b)  \dars  
			\end{minipage}
		\end{tabular}
	\end{minipage}%
	% \vspace{-2mm}
	\caption{Snapshots of the multi-label semantic segmentation computed by the proposed MRF solvers
	(a) \fuses and (b) \dars on the Cityscapes dataset. \fuses is able to 
 	 segment an image in 16ms (1000 superpixels).
	% segment the image in less than 20ms.}
	 \label{fig-segmentation-snapshots}}
	\vspace{-8mm} 
	\end{center}
\end{figure}

% %%%%%%%%%%%%%%%%%%%%%%%%%%%%%%%%%%%%%%%%%%%%%%%%

\myparagraph{Contribution} 
Our first contribution, presented in Section~\ref{sec:dars}, is to design a dual-ascent-based method to solve standard SDP relaxations that 
takes advantage of the geometric structure of the problem to speed up computation. 
In particular, we show that each dual ascent iteration can be solved using a fast low-rank SDP solver known as the 
\emph{Riemannian Staircase}~\cite{Boumal16nips}.
%enabling the use of the \emph{Riemannian Staircase}. % to improve scalability. 
% Upon convergence of the dual ascent iterations this technique attains the same objective as the standard SDP relaxation while 
% being more scalable. 
This technique, named \emph{Dual Ascent Riemannian Staircase} (\dars), is able to solve MRF instances with thousands of variables \emph{in seconds}, 
while general-purpose SDP solvers (e.g., \cvx~\cite{CVXwebsite})  are not able to provide an answer in reasonable time (hours) at that scale.

Our second contribution, presented in Section~\ref{sec:fuses}, 
\change{is an even faster SDP relaxation.
% This contribution is motivated by the fact that we are interested in real-time inference on robotics platforms, 
Despite being  remarkably faster than general-purpose SDP solvers,
 \dars is currently slow for real-world robotics applications, hence, we develop a \emph{Fast Unconstrained SEmidefinite Solver} (\fuses) that can solve large problems \emph{in milliseconds}.}{is a \emph{Fast Unconstrained SEmidefinite Solver} (\fuses) that can solve large problems \emph{in milliseconds}.} 
% Since our goal is to enable real-time inference on robotics platforms, we develop a second and faster approach. 
The backbone of this second approach is a novel SDP relaxation combined with the Riemannian Staircase method~\cite{Boumal16nips}. 
The novel formulation uses a more intuitive binary matrix (with entries in $\{0,1\}$), 
contrarily to related work that parametrizes the problem using a vector with entries in $\{-1,+1\}$. 
\change{}{\fuses does not require an initial guess for optimization (i.e., it is a global solver) and provides per-instance sub-optimality guarantees.}  
%parameterizing the problem using a vector with entries in $\{-1,+1\}$. 
% Moreover, the Riemannian Staircase leverages the fact that the variables in the resulting SDP relaxation 
%We show that this approach, named \emph{Fast Unconstrained SEmidefinite Solver} (\fuses), can solve large problems \emph{in milliseconds}. 
 % while providing per-instance suboptimality guarantees.
%  In this paper we propose a novel semidefinite relaxation for maximum-a-posteriori inference in MRFs and we leverage 
% recent results in optimization theory to develop a fast and scalable specialized solver for the resulting SDP. Extensive experiments show that
%  our approach outperforms state-of-the-art techniques (including move-making and message-passing methods) in terms of accuracy and speed.

Our third contribution is an extensive experimental evaluation. We test the proposed SDP solvers in semantic image segmentation problems 
and evaluate the corresponding results in terms of accuracy and runtime. 
We compare the proposed techniques against several related approaches, including move-making methods (\emph{$\alpha$-expansion}~\cite{Boykov01pami-graphCut}) 
and message passing 
(\emph{Loopy Belief Propagation}~\cite{Weiss01itis-beliefPropagation} 
and \emph{Tree-Reweighted Message Passing}~\cite{Wainwright05itis}).
\change{The results show that our MRF solver retains all the advantages of SDP relaxations (accuracy, no need for initial guess), %, , no assumption on the objective function), 
while being fast and scalable.}{ 
Upon convergence, \dars attains the same solution of standard SDP relaxations. \fuses, on the other hand, trades-off inference time for a mild loss in accuracy.} 
More specifically, our results show that
(i) \fuses and \dars produce near-optimal solutions, attaining an objective within \percSubopt of the optimum,
%with a suboptimality gap between \tocheck{1-2\%} from the optimum,
(ii) \fuses and \dars are remarkably faster than general-purpose SDP solvers (e.g., \CVX~\cite{CVXwebsite}), and
  \fuses is more than two orders of magnitude faster than \dars while attaining similar solution quality,
(iii) \change{}{\fuses is more than 2x faster than local search methods while being a global solver.}
%, and is a good candidate for real-time robotics applications.
% Section~\ref{sec:experiments} provides extensive experimental evidence showing that the proposed SDP solver outperforms 
% related techniques based on graph cut and message passing in real problem instances. \todo{this needs experimental results}.
% SDP known to work well but typically disregarded in evaluations (e.g.,~\cite{Szeliski08pami-surveyMRF,Kappes15ijcv-energyMin}) due to its computational cost.
% We show that it can be as fast or faster than the state of the art
% In a nutshell, our MRF solver retains all advantages of SDP relaxations (accuracy, no need for initial guess, no assumption on the cost function), 
% while being very fast and scalable. 
% Interestingly, the proposed approach draws connections between discrete inference in MRFs
%  and continuous optimization techniques for geometric understanding (e.g., Simultaneous Localization and Mapping, Structure from Motion), 
%  where we successfully proposed similar relaxations~\cite{Carlone15rssws2D-dualityPGO2D} and solvers~\cite{Rosen16wafr-sesync}.

\change{}{While the evaluation in this paper focuses on the MRF solver (rather than attempting to outperform state-of-the-art deep learning methods for semantic segmentation), we believe \fuses can be used \emph{in conjunction} with existing deep learning methods, as done in~\cite{Chen18pami-deepLab}, to refine the segmentation results. For this purpose, we released our implementation online at \url{https://github.mit.edu/SPARK/sdpSegmentation}.}

Before delving into the contribution of this paper,
Section~\ref{sec:preliminaries} provides preliminary notions on inference in MRFs, %and discusses standard semidefinite relaxations,
while we postpone the review of related work to Section~\ref{sec:relatedWork}. 
\isExtended{}{All proofs  are given in the supplemental material~\cite{Hu18tr-fuses}\change{.}{, together with extra experimental results and a more extensive literature review.}}

\section{Preliminaries}
\label{sec:preliminaries}

This section introduces standard notation for MRFs (\prettyref{sec:preMRF}) and provides necessary 
background on semidefinite relaxations (\prettyref{sec:standardSDPrelax}). % that will be instrumental to introduce our contribution.
 % (\prettyref{sec:contribution}),
 % and review technical aspects of related work (\prettyref{sec:relatedWork}).

\subsection{Markov Random Fields: Models and Inference}
\label{sec:preMRF}

A \emph{Markov Random Field} (MRF) is a graphical model in which \emph{nodes} are associated with  discrete labels we want to estimate,
% are associated with 
and \emph{edges} (or \emph{potentials}) represent given probabilistic constraints relating the labels of a subset of nodes.
Formally, for each node $i$ in the node set $\nodeSet \doteq \{1,\ldots,\nrNodes\}$ 
(where $\nrNodes$ is the number of nodes), we need to assign a label $\class_i \in \classSet$, 
where $\classSet \doteq \{1,\ldots,\nrClasses\}$ is the set of $\nrClasses$  possible labels. 
  If $\nrClasses = 2$ (i.e., nodes are classified into two classes) the corresponding model is called a \emph{binary} MRF.
   Here we consider $\nrClasses \geq 2$ possible labels, a setup  generally referred to as a \emph{multi-label} MRF. 
   % Node potenti
 % Moreover, we mainly focus on \emph{pairwise} MRFs, where each edge involves at most two nodes.

%%%%%%%%%%%%%%%%%%%%%%%%%%%%%%%%%%%%%%%%%%%%%%%%%%%%%%%%%%%%%%%%%%%%%%%%%%%%%%%%
\myparagraph{Maximum a posteriori (MAP) inference}
The MAP estimate is the most likely assignment of labels\isExtended{, i.e., the assignment
 of the node labels 
that attains the maximum of the posterior distribution of an MRF, or, equivalently, the minimum of the negative log-posterior.}{ in the MRF.}
% The minimization of the  negative log-posterior 
MAP estimation can be formulated as a discrete optimization problem over the labels $\class_i  \in \classSet$ 
with $i=1,\ldots,\nrNodes$~\cite{Szeliski08pami-surveyMRF}:
\beq
\label{eq:MRF1}
\min_{\substack{\class_i \in \classSet \\ i = 1,\ldots,\nrNodes}}  \;\; 
\sum_{i \in \unarySet} E_i(\class_i) + \sum_{(i,j)\in\binarySet} E_{ij}(\class_i,\class_j)
\tag{\Pmrf}
\eeq
where $\unarySet \subseteq \nodeSet$ is the set of \emph{unary potentials} (probabilistic constraints involving a single node),
$\binarySet \subseteq \nodeSet \times \nodeSet$ is the set of \emph{binary potentials} (involving a pair of nodes), 
and $E_i(\cdot)$ and $E_{ij}(\cdot)$ represent the negative log-distribution 
for each unary and binary potential, respectively (described below).
 For instance, in semantic segmentation each node in the MRF corresponds to a pixel (or superpixel) in the image, the unary potentials are dictated by pixel-wise classification from a classifier applied to the image, and the binary potentials enforce smoothness of the resulting 
segmentation~\cite{Zhu16jcvir}. The binary potentials (often referred to as \emph{smoothness priors}) are typically enforced between 
nearby (adjacent) pixels. 
% For pixel-wise segmentation this results in 4-connected MRFs (where each 
% pixel is connected to the two horizontal and the two vertical neighbors) or 8-connected MRFs (when also diagonal neighbors are considered).

%%%%%%%%%%%%%%%%%%%%%%%%%%%%%%%%%%%%%%%%%%%%%%%%%%%%%%%%%%%%%%%%%%%%%%%%%%%%%%%%
\myparagraph{MRF Potentials}
A typical form for the unary and binary potentials is:

\vspace{-0.7cm}
\begin{equation}
\begin{split}
% \label{eq:unary}
\label{eq:binary}
E_i(\class_i) &= 
\left\{
\begin{array}{ll}
0 & \text{ if } \class_i = \measuredClass_i  \\
\penaltyTerm_{i} & \text{otherwise}
\end{array}	
\right. 
\\
E_{ij}(\class_i,\class_j) &= 
\left\{
\begin{array}{ll}
0 & \text{ if } \class_i = \class_j  \\
\penaltyTerm_{ij} & \text{otherwise}
\end{array}	
\right.
\end{split}
\end{equation}
where $\measuredClass_i$ is a data-driven noisy measurement of the label of node $i$ (typically from a classifier), and $\penaltyTerm_{i}$ 
and $\penaltyTerm_{ij}$ are given scalars. 
Typically, it is assumed %$\eqpenaltyTerm_{i} = 0$ (no change in the cost if the node $i$ label matches the measured one), and 
$\penaltyTerm_{i} \geq 0$, i.e., choosing a label different from the measured one incurs a cost $\penaltyTerm_{i}$ in~\eqref{eq:MRF1}. 
% we incur a penalty 
% $\penaltyTerm_{i}$, which increases the cost in~\eqref{eq:MRF1})
% In words, if we choose a value of $\class_i$ different from $\measuredClass_i$ .
%
% Similarly, the binary potentials are modeled as:
%
% \beq
% \label{eq:binary}
% E_{ij}(\class_i) = 
% \left\{
% \begin{array}{ll}
% 0 & \text{ if } \class_i = \class_j  \\
% \penaltyTerm_{ij} & \text{otherwise}
% \end{array}	
% \right.
% \eeq
%
%which penalizes different labels at adjacent nodes. 
Similarly, for the binary potentials $E_{ij}(\cdot)$ it is typically assumed %$\eqpenaltyTerm_{ij} = 0$ (no change in the cost if the labels of adjacent nodes $i$ and $j$ match), and 
$\penaltyTerm_{ij} \geq 0$, i.e., label mismatch ($\class_i \neq \class_j$) incurs a cost of $\penaltyTerm_{ij}$ in the objective~\eqref{eq:MRF1}. 
\isExtended{
In this case the binary potentials are called \emph{attractive}, while they are referred to as \emph{repulsive} when $\penaltyTerm_{ij} < 0$ 
(i.e., the potentials encourage label mismatches)~\cite{Gallagher11cvpr}.}{}
\omitted{
The assumptions $\penaltyTerm_{i} \geq 0$, $\penaltyTerm_{ij} \geq 0$ (for all $i \in \calV$ and $(i,j) \in \calE$) are 
 necessary for some related techniques to work (see Section~\ref{sec:relatedWork}), %see e.g.,~\cite{Boykov01pami-graphCut}, 
but are not required for our method, hence here we do not take any assumption on $\penaltyTerm_{i}$ and $\penaltyTerm_{ij}$.
}
%on these scalars.
%$\eqpenaltyTerm_{i}, \penaltyTerm_{i}$

The MRF resulting from the choice of potentials in eq.~\eqref{eq:binary} 
is known as the \emph{Potts model}\isExtended{~\cite{Potts52-pottsModel}, which was first proposed
 in statistical mechanics to model interacting spins in ferromagnetic materials.}{.}
 When $\nrClasses\!=\!2$ (binary MRFs) the resulting model is known as the \emph{Ising model}~\cite[Section 1.4.1]{Blake11book-MRF}.
\subsection{Standard Semidefinite Relaxation}\label{sec:standardSDPrelax}
Semidefinite programming (SDP) relaxation has been shown to provide an effective approach to compute a 
good approximation of the global minimizer of~\eqref{eq:MRF1}~\cite{Keuchel03pami,Keuchel04dagm,Olsson08cviu}.
In this section we introduce a standard approach to obtain an SDP relaxation, for which we design a fast solver in Section~\ref{sec:dars}. % and comment on its limitations.

% The first step to design the SDP relaxation is to rephrase the problem in terms of binary optimization. 
In order to obtain an SDP relaxation, related works rewrite each node variable $\class_i \in \classSet \doteq \{1,\ldots,\nrClasses\}$ as a vector 
$\Classesmp_i \in \{-1,+1\}^\nrClasses$, such that $\Classesmp_i$ has a single entry equal to $+1$ (all the others are $-1$), and if the $j$-th entry of $\Classesmp_i$ is $+1$, then the corresponding node has label $j$. 
Moreover, they stack all vectors $\Classesmp_i$, $i =1,\ldots,\nrNodes$, in a single 
$\nrNodes\nrClasses$-vector $\Classesmp = [\Classesmp_1\tran \; \Classesmp_2\tran \; \ldots \; \Classesmp_\nrNodes\tran]\tran$. %\in \{-1,+1\}^{\nrNodes\nrClasses}
Using this reparametrization, the inference problem~\eqref{eq:MRF1} can be written in terms of the 
 vector $\Classesmp$ as follows (full derivation in~\isExtended{Appendix A}{\cite[Appendix A]{Hu18tr-fuses}}): 
\beq
\begin{array}{rl}
\label{eq:MRFmp2}
\min_{\Classesmp}  & \Classesmp\tran \MA \Classesmp + 2 \vb\tran \Classesmp
\\
\subject & \diag{\Classesmp\Classesmp\tran} = \ones_{\nrNodes\nrClasses}, %\quad k = 1,\ldots, \nrNodes\nrClasses 
\\
         & \vu_i\tran \Classesmp = \cumsum, \quad i=1,\ldots,\nrNodes   % \MU \Classesmp = (\nrClasses - 2) \ones_{\nrNodes}
\end{array}
\eeq
where $\MA$ and $\vb$ are a suitable symmetric matrix and a suitable vector collecting the coefficients of the 
binary terms and the unary terms in~\eqref{eq:binary}, respectively;
$\diag{\Classesmp\Classesmp\tran}$ is the diagonal of the matrix $\Classesmp\Classesmp\tran$, 
and $\vu_i \doteq \ve_i\tran \kron \ones_\nrClasses\tran$, where $\ve_i$ is an $\nrNodes$-vector which is 
all zero, except the $i$-th entry which is one, $\ones_\nrClasses$ is a $\nrClasses$-vector of ones, 
and $\kron$ is the Kronecker product.
 Intuitively, $\diag{\Classesmp\Classesmp\tran}$ contains
 the square of each entry of $\Classesmp$, hence $\diag{\Classesmp\Classesmp\tran} = \ones_{\nrNodes}$ imposes 
 that every entry of $\Classesmp$ has norm $1$, i.e., it belongs to $\{-1,+1\}$; 
  the constraint $\vu_i\tran \Classesmp = \cumsum$ writes in compact form 
  $\ones\tran \Classesmp_i = \cumsum$, which enforces each node to have a unique label (i.e., a 
  single entry in $\Classesmp_i$ can be $+1$, while all the others are $-1$). % and the symbol . 

Before relaxing problem~\eqref{eq:MRFmp2}, it is convenient to \emph{homogenize} the objective by reparametrizing the 
problem in terms of an extended vector $\vy \doteq [\Classesmp\tran 1]\tran$, where an entry equal to $1$ is concatenated to $\Classesmp$.
We can now rewrite~\eqref{eq:MRFmp2} in terms of $\vy$:
\beq
\begin{array}{rrl}
\label{eq:MRFmp3}
\foptmp =& \min_{\vy}  & \trace{\ML \vy\vy\tran} %  \vy \tran \ML \vy
\\
&\subject & \diag{\vy\vy\tran} = \ones_{\nrNodes\nrClasses+1}  \\
 &        & \trace{\MU_i \vy\vy\tran} = \cumsum, \quad i=1,\ldots,\nrNodes
         %\breve{\MU} (\vy\vy\tran) \ve_{\nrNodes+1} = (\nrClasses - 2) \ones_{\nrNodes}
\end{array}
\tag{\Pmp}
\eeq
where $\ML \doteq {\scriptsize \matTwo{\MA  & \vb \\ \vb\tran & 0}}$ 
and $\MU_i \doteq {\scriptsize \matTwo{\zero & \zero \\ \vu_i\tran & 0}}$.
In~\eqref{eq:MRFmp3}, we used the equality $\vy \tran \ML \vy = \trace{\ML \vy\vy\tran}$, 
and noted that since $\vy\vy\tran = {\scriptsize \matTwo{\Classesmp\Classesmp\tran  & \Classesmp \\ \Classesmp\tran & 1}}$, 
then  $\trace{\MU_i \vy\vy\tran} = \vu_i\tran \Classesmp$.
% in~\eqref{eq:MRFmp3}, we also rewrote the last equality in~\eqref{eq:MRFmp2} as a function of $\vy\vy\tran$ using the following chain of equalities: 
%  {\small $\vu_i\tran \Classesmp =  [\vu_i\tran \; 0] \cdot [\Classesmp\tran \; 1]\tran 
%  \overset{(a)}{=} 
%  [\vu_i\tran \; 0] \cdot (\vy\vy\tran) \cdot \ve_{\nrNodes\nrClasses+1} 
%  \overset{(b)}{=} 
%  \trace{\ve_{\nrNodes\nrClasses+1} \cdot [\vu_i\tran \; 0] \cdot (\vy\vy\tran)} 
%  \overset{(c)}{=}
%   \trace{\MU_i \vy\vy\tran}$}, where in (a) we used the fact that $\vy\vy\tran = {\scriptsize \matTwo{\Classesmp\Classesmp\tran  & \Classesmp \\ \Classesmp\tran & 1}}$
%   and $(\vy\vy\tran) \cdot \ve_{\nrNodes\nrClasses+1}$ picks the last column of $\vy\vy\tran$, 
%   in (b) we used the cyclic property of the trace, and in (c) we defined $\MU_i \doteq \ve_{\nrNodes\nrClasses+1} \cdot [\vu_i\tran \; 0]$.
 % we also noticed that since $\vy\vy\tran = {\scriptsize \matTwo{\Classesmp\Classesmp\tran  & \Classesmp \\ \Classesmp\tran & 1}}$
%  and where
% $\ve_{\nrNodes+1}$ is a vector with a single nonzero element in position $\nrNodes+1$, 
% and $\breve{\MU} \doteq [\MU \; \zero_{\nrNodes}]$.

So far we have only reparametrized problem~\eqref{eq:MRF1}, hence~\eqref{eq:MRFmp3} is still a MAP estimator. 
We can now introduce the SDP relaxation: 
% At this point, %to obtain a semidefinite relaxation of~\eqref{eq:MRFmp3}, 
 problem~\eqref{eq:MRFmp3} 
only includes terms in the form $\vy\vy\tran$, hence we can reparametrize it using a matrix  
$\MY \doteq \vy\vy\tran$. Moreover, we note that the set of matrices $\MY$ that satisfy $\MY \doteq \vy\vy\tran$ 
is the set of positive semidefinite ($\MY\succeq 0$)  rank-1 matrices ($\rank{\MY}=1$).
% any matrix that can be written as $\MY \doteq \vy\vy\tran$ must be 
% positive semidefinite and have rank equal to one. 
Rewriting~\eqref{eq:MRFmp3} using $\MY$ and dropping the 
non-convex rank-1 constraint, we obtain:
\beq
\begin{array}{rrl}
\label{eq:SDPstandard}
\fsdpmp =& \min_{\MY}  & \trace{\ML \MY} %  \vy \tran \ML \vy
\\
&\subject & \diag{\MY} = \ones_{\nrNodes\nrClasses+1}  
\\
&         & \trace{\MU_i \MY} = \cumsum, \quad i=1,\ldots,\nrNodes
         % \breve{\MU} \MY \ve_{\nrNodes+1} = (\nrClasses - 2) \ones_{\nrNodes}
\\
&         & \MY\succeq 0
\end{array}
\tag{\Smp}
\eeq
which is a (convex) semidefinite program and can be solved globally in polynomial time using interior-point methods\isExtended{~\cite{Boyd04book}. While the SDP relaxation~\eqref{eq:SDPstandard} is known to provide near-optimal approximations of the MAP estimate, 
interior-point 
methods are typically slow in practice and cannot solve problems with more than few hundred nodes in a reasonable time.}{.
Unfortunately, interior-point 
methods scale poorly in the dimension of the SDP, and are not able to solve problems with more than few hundred nodes in a reasonable time.}  
% Note that the derivation in this section applies to any MRFs with pairwise potentials, including fully connected MRFs 
% and MRFs with repulsive binary potentials.
%Note that the derivation works for fully connected MRFs with pairwise potentials.

%!TEX root = main.tex

\section{\dars: Dual Ascent Riemannian Staircase}
\label{sec:dars}

This section presents the first contribution of this paper: a dual ascent approach to efficiently solve large instances of 
the standard SDP relaxation~\eqref{eq:SDPstandard}. 
% Towards this goal we no
\omitted{ We first introduce our dual ascent strategy in Section~\ref{sec:dars-ascent};
we then introduce the Riemannian Staircase to efficiently solve each dual ascent iteration in Section~\ref{sec:dars-staircase}.}

%%%%%%%%%%%%%%%%%%%%%%%%%%%%%%%%%%%%%%%%%%%%%%%%%%%%%%%%%%%%%%%%%%%%%%%%%%%%%%%%%%%%%%%%%%%%%%%%%%%%%%%%%%%%%%%%%%%%%%%%%%%%%%%%%%%%%%%
\subsection{Dual Ascent Approach}
\label{sec:dars-ascent}

% The key insight of this section is that we would be able to solve very efficiently a version of problem~\eqref{eq:SDPstandard} 
% not including the equality constraint $\trace{\MU_i \MY} = \nrClasses - 2, \;\; i=1,\ldots,\nrNodes$. Therefore in this section 
% we propose a dual-ascent approach that aims at indirectly accounting for this linear constraint by including it in the objective function.
 % This section introduces a dual ascent strategy. 
The main goal of this section is to design a dual ascent method, where the subproblem to be solved at each iteration has 
a more favorable geometry, and can be solved quickly using the 
 Riemannian Staircase method introduced in Section~\ref{sec:dars-staircase}.
 % The main intuition behind this section is that while the constraints in~\eqref{eq:SDPstandard} do not have a favorable geometry to enable
 % fast solvers,  we can design a dual ascent method where each iteration can be solved quickly using the 
 % Riemannian Staircase method introduced in Section~\ref{sec:dars-staircase}.
Towards this goal, we 
 % The goal of this section is to ``remove'' the trace constraints $\trace{\MU_i \MY} = \nrClasses - 2$ from~\eqref{eq:SDPstandard}: the key advantage of doing so is that
 % the resulting problem without the trace constraint can be solved efficiently using the Riemannian Staircase (Section~\ref{sec:dars-staircase}). 
 % More formally, let us 
 rewrite~\eqref{eq:SDPstandard} equivalently as:
\beq
\begin{array}{rl}
\label{eq:DA1}
\min_{\MY}  & g(\MY)
\\
\subject & \trace{\MU_i \MY} = \cumsum, \quad i=1,\ldots,\nrNodes
\end{array}
\eeq
where the objective function is now $g(\MY) \doteq \trace{\ML \MY} + \calI(\diag{\MY} = \ones_{\nrNodes\nrClasses+1}) + \calI(\MY\succeq 0)$, where
$\calI(\cdot)$ is the indicator function which is zero when the constraint inside the parenthesis is satisfied and plus infinity otherwise. 
% Clearly,~\eqref{eq:DA1} is equivalent to~\eqref{eq:SDPstandard} and simply assigns an infinite cost to solutions violating 
% the diagonal and the semidefinite constraints.  

Under constraints qualification (e.g., the Slater's condition for convex programs~\cite[Theorem 3.1]{Vandenberghe96siam}), 
% Now we note that since we are dealing with convex problems, strong duality holds,\footnote{While not all convex problems exhibit strong duality, 
% our problem satisfies } hence 
we can obtain an optimal solution to~\eqref{eq:DA1} by computing a saddle-point of the Lagrangian function $\calL(\MY,\vlambda)$:
\beq
\label{eq:DA2}
\max_{\vlambda} \inf_{\MY}  \calL(\MY,\vlambda) = 
\max_{\vlambda} \inf_{\MY}  g(\MY) + \textstyle\sum_{i=1}^{\nrNodes} \vlambda_i (\trace{\MU_i \MY} + \KpTwo)
\eeq
where $\vlambda \in \Real{\nrNodes}$ is the vector of \emph{dual variables} and $\MY$ is the \emph{primal} variable.

The basic idea behind dual ascent~\cite[Section 2.1]{Boyd10fnt} is to solve the saddle-point problem~\eqref{eq:DA2} by alternating maximization steps 
with respect to the dual variables $\vlambda$ and minimization steps with respect to the primal variable $\MY$.

\myparagraph{Dual Maximization} 
The maximization of the dual variable is carried out via gradient ascent. In particular, at each iteration $t = 1,\ldots,T$ ($T$ is the maximum number 
of iterations), the dual 
ascent method fixes the primal variable and updates the dual variable $\vlambda$ as:
\beq
\label{eq:gradAscent}
\vlambda\at{t} = \vlambda\at{t-1} + \alpha \nabla \calL_\vlambda(\MY\at{t-1},\vlambda\at{t-1})
\eeq
where $\nabla \calL_\vlambda(\MY\at{t-1},\vlambda\at{t-1})$ is the gradient of the Lagrangian with respect to the dual variables, 
evaluated at the latest estimate of the primal-dual variables $(\MY\at{t-1},\vlambda\at{t-1})$, and 
$\alpha$ is a suitable stepsize. It is straightforward to compute the gradient with respect to the $i$-th dual variable as 
$\nabla \calL_{\vlambda_i}(\MY,\vlambda) = \trace{\MU_i \MY} +\KpTwo$. 
\isExtended{Intuitively, the second summand in~\eqref{eq:DA2} penalizes 
the violation of the constraint $\trace{\MU_i \MY} = \cumsum$ (for all $i$). Moreover, 
since the gradient in~\eqref{eq:gradAscent} grows with the amount of violation $\trace{\MU_i \MY} - \nrClasses + 2$,
the dual update~\eqref{eq:gradAscent} 
increases the penalty for constraints with large violation.}{}
% since  
% the gradient in~\eqref{eq:gradAscent} grows with the amount of violation $\trace{\MU_i \MY} - \nrClasses + 2$, the entries of 
% $\vlambda$ tend to grow quickly for nodes that violate the constraint, hence penalizing more and more such violation.

\myparagraph{Primal Minimization}
The minimization step fixes the dual variable to the latest estimate $\vlambda\at{t-1}$ and minimizes~\eqref{eq:DA2} 
with respect to the primal variable $\MY$: 
\beq
\label{eq:DA3}
\min_{\MY}  g(\MY) + \textstyle \sum_{i=1}^{\nrNodes} \vlambda\at{t-1}_i (\trace{\MU_i \MY} + \KpTwo)
\eeq
where we substituted ``$\inf$'' for ``$\min$'' since the objective cannot drift to minus infinity due to the 
implicit constraints imposed by the indicator 
functions in $g(\MY)$. 
Recalling the expression of $g(\MY)$, defining 
$\ML_{\lambda} \doteq \ML + \sum_{i=1}^{\nrNodes} \vlambda\at{t-1}_i \MU_i$, 
and moving again the  indicator functions to the constraints we write~\eqref{eq:DA3} more explicitly as:
% %
% \beq
% \begin{array}{rl}
% \label{eq:DA4}
% \min_{\MY}  & \trace{\ML \MY} + \sum_{i=1}^{\nrNodes} \vlambda\at{t-1}_i (\trace{\MU_i \MY} - \nrClasses + 2)
% \\
% \subject & \diag{\MY} = \ones_{\nrNodes\nrClasses+1}  \\
%          & \MY\succeq 0
% \end{array}
% \eeq
% % 
% where we moved again the indicator functions to the constraints. 
% Defining $\ML_{\lambda} \doteq \ML + \sum_{i=1}^{\nrNodes} \vlambda\at{t-1}_i \MU_i$ and dropping constant terms from the objective,~\eqref{eq:DA4} 
% becomes: % can be written as: 
%
\beq
\begin{array}{rrl}
\label{eq:primalDescent}
\MY\at{t} = & \argmin_{\MY}  & \trace{\ML_{\lambda} \MY}\\
& \subject & \diag{\MY} = \ones_{\nrNodes\nrClasses+1}  \\
&         & \MY\succeq 0
\end{array}
\eeq
where we dropped the constant terms $\sum_{i=1}^{\nrNodes} \vlambda\at{t-1}_i (\KpTwo)$ from the objective 
since they are irrelevant for the 
optimization.
The minimization step in the dual ascent is again an SDP, but contrarily to the standard SDP~\eqref{eq:SDPstandard}, 
problem~\eqref{eq:primalDescent} can be solved quickly using the Riemannian Staircase, as discussed in the following.

%%%%%%%%%%%%%%%%%%%%%%%%%%%%%%%%%%%%%%%%%%%%%%%%%%%%%%%%%%%%%%%%%%%%%%%%%%%%%%%%%%%%%%%%%%%%%%%%%%%%%%%%%%%%%%%%%%%%%%%%%%%%%%%%%%%%%%%
\subsection{A Riemannian Staircase for the Dual Ascent Iterations}
\label{sec:dars-staircase}

This section provides  a fast solver to compute a solution for the SDP~\eqref{eq:primalDescent}, that needs to be solved 
at each iteration of the dual ascent method of Section~\ref{sec:dars-ascent}.
\omitted{ While the dual maximization~\eqref{eq:gradAscent} is inexpensive 
and only requires computing the gradients for each dual variable $i=1,\ldots,\nrNodes$, developing a fast solver for~\eqref{eq:primalDescent} 
is crucial to make the dual ascent method practical.}

% ~\cite{Keuchel03pami} concluded that working in such high-dimensional state space is ``necessary'' to get rid of the non-convexity of the problem: 
% here we leverage recent results on low-rank semidefinite programming to show that we can compute optimal solutions while working in much smaller spaces.
% The SDP~\eqref{eq:primalDescent}  is a convex program and can be solved in polynomial time using general-purpose solvers, including interior-point methods~\cite{Boyd04book}. 
% Unfortunately, interior-point-method-based solvers are relatively slow and do not scale to problem with more that few hundred nodes. %
% %and the experiments in Section~\ref{sec:experiments}. 
% The high-computational cost of general-purpose SDP solvers is due to the need to manipulate and store the large and typically dense matrix $\MY$. On the other hand, we empirically observe that the solution of the SDP~\eqref{eq:primalDescent} is typically low rank 
%  (we will formalize this point in Proposition~\ref{prop:boumal}). 
% % This suggests that the matrix $\MZ$ can be represented with  
% This allows leveraging recent results in low-rank SDP solvers~\cite{Burer03mp,Boumal16nips}. 

% To solve the SDP~\eqref{eq:primalDescent} efficiently, 
We use of the Burer-Monteiro method~\cite{Burer03mp}, which replaces the matrix $\MY \in \Real{(\nrNodes\nrClasses+1) \times (\nrNodes\nrClasses+1)}$ in~\eqref{eq:primalDescent} with a rank-$r$ product $\MR \MR\tran$ with $\MR \in \Real{(\nrNodes\nrClasses+1) \times r}$: % rewriting~\eqref{eq:newSDP} as:     
\beq
\label{eq:RRT1}
\begin{array}{rcl}
& \min_{\MR} & \trace{\ML_{\lambda} \MR \MR\tran}  \\
&\subject   &  \diag{\MR \MR\tran} = \ones_{\nrNodes\nrClasses+1}
\end{array}
\eeq
Note that the constraint $\MY \succeq 0$ in~\eqref{eq:primalDescent}  becomes redundant after the substitution, since $\MR \MR\tran$ is always positive semidefinite, hence it is dropped. 
% While it seems we took a step backward by reparametrizing $\MY$ as $\MR \MR\tran$ (the new parametrization seems similar to the original parametrization $\vy \vy\tran$ in~\eqref{eq:MRFmp3}, 
% and~\eqref{eq:RRT1} is again a nonconvex problem), we remark that $\MR$ is an overparametrization whenever $r > 1$ ($\MR$ is a matrix, while $\vy$ is a vector), which will be useful in Proposition~\ref{prop:boumal}.
%This observation will be fundamental for XXX.

Following Boumal\setal~\cite{Boumal16nips} we note that the constraint set in~\eqref{eq:RRT1} describes a smooth manifold, and in particular a product of Stiefel manifolds. To make this apparent, we recall that 
the (transposed) \textit{Stiefel manifold} is defined as~\cite{Boumal16nips}:
\beq
\Stiefel{d}{p} = \{\MM\in\Real{d\times p}\colon \MM \MM\tran = \eye_d\}
\eeq
Then, we observe that 
$\diag{\MR \MR\tran} = \ones_{\nrNodes\nrClasses+1}$ can be written as 
$\MR_i \MR_i\tran = 1$, $i=1,\ldots,\nrNodes\nrClasses+1$ (where $\MR_i$ is the $i$-th row of $\MR$), which is equivalent to saying that 
$\MR_i \in \Stiefel{1}{r}$ for $i=1,\ldots,\nrNodes\nrClasses+1$. 
This observation allows concluding that the matrix $\MR$ 
belongs to the product manifold $\Stiefel{1}{r}^{\nrNodes\nrClasses+1}$.
% \beq
% \StiefelProduct = 
% \overbrace{\Stiefel{1}{r} \times \Stiefel{1}{r} \times \ldots \Stiefel{1}{r}}^{\nrNodes \text{ times } } \times \Stiefel{\nrClasses}{r}
% \eeq
%
Therefore, we can rewrite~\eqref{eq:RRT1} as an unconstrained optimization on manifold: % $\StiefelProduct$:
\beq
\label{eq:RRT2}
\min_{\MR \in \Stiefel{1}{r}^{\nrNodes\nrClasses+1}} \trace{\ML_{\lambda}  \MR \MR\tran} 
\tag{\Rmp}
\eeq
The formulation~\eqref{eq:RRT2} is non-convex (the product of Stiefel manifolds describes a non-convex set), but one can find \emph{local} minima efficiently using iterative methods~\cite{Boumal16nips,Rosen16wafr-sesync}.
While it might seem that little was gained 
(we started with an intractable problem and we ended up with another non-convex problem), the following remarkable result from 
 Boumal\setal~\cite{Boumal16nips} ties back \emph{local solutions} of~\eqref{eq:RRT2} to \emph{globally optimal solutions} of the SDP~\eqref{eq:primalDescent}.

 \begin{proposition}[Optimality Conditions for~\eqref{eq:RRT2}, Corollary 8 in~\cite{Boumal16nips}]
 \label{prop:boumal}
 If $\MR  \in \Stiefel{1}{r}^{\nrNodes\nrClasses+1}$ is a (column) rank-deficient second-order critical point of problem~\eqref{eq:RRT2}, then $\MR$ is a global optimizer of~\eqref{eq:RRT2}, and 
 $\MY^\star \doteq \MR\MR\tran$ is a solution of the semidefinite relaxation~\eqref{eq:primalDescent}.
 \end{proposition}

The previous proposition ensures that when
local solutions (second-order critical points) of~\eqref{eq:RRT2} are rank deficient, then they can be 
mapped back to global solutions of~\eqref{eq:primalDescent}, 
hence providing a way to solve~\eqref{eq:primalDescent} efficiently via~\eqref{eq:RRT2}.

 The catch is that one has to choose the rank $r$ large enough to obtain rank-deficient solutions. %, which is required by the optimality condition in Proposition~\ref{prop:boumal}. 
Related work~\cite{Boumal16nips} therefore proposes the \emph{Riemannian staircase} method, where one solves~\eqref{eq:RRT2} for increasing values of $r$ till a rank-deficient solution is found. % Similarly to related problems~\cite{Rosen16wafr-sesync}, in Section~\ref{sec:experiments} we will notice 
% In our tests few steps of the staircase
% that empirically a single step of the staircase ($r = 2$ in our case) is sufficient to find a rank-deficient solution
Boumal\setal~\cite{Boumal16nips} also provide theoretical results
ensuring that rank-deficient solutions are found for small $r$ (more details in Section~\ref{sec:experiments}). 
% bounds on the
% maximum rank required to obtain a rank-deficient solution. % are given in~\cite{Boumal16nips}). 

%%%%%%%%%%%%%%%%%%%%%%%%%%%%%%%%%%%%%%%%%%%%%%%%%%%%%%%%%%%%%%%%%%%%%%%%%%%%%%%%%%%%%%%%%%%%%%%%%%%%%%%%%%%%%%%%%%%%%%%%%%%%%%%%%%%%%%%
\subsection{\dars: Summary, Convergence, and Guarantees}
\label{sec:dars-guarantees}

We name \dars (\emph{Dual Ascent Riemannian Staircase}) the approach resulting from the combination of dual ascent and the Riemannian Staircase.
\dars starts with an initial guess for the dual variables (we use $\vlambda = \zero_{\nrNodes}$), and then alternates two steps: 
(i) the \emph{primal minimization} where a solution for~\eqref{eq:primalDescent} is obtained using the Riemannian Staircase~\eqref{eq:RRT2} 
(in practice this is solved using iterative methods, such as the Truncated Newton method);
(ii) the \emph{dual maximization} were the dual variables are updated using the gradient ascent update~\eqref{eq:gradAscent}.

\myparagraph{Rounding} 
 Upon convergence, \dars produces a matrix $\MY^\star = \MR\MR\tran$. 
When deriving the standard SDP relaxation~\eqref{eq:SDPstandard} we dropped the rank-1 constraint, hence $\MY^\star$
 cannot be written in general as $\MY^\star = \hat{\vy}\hat{\vy}^\star$. 
 The process of computing a feasible solution $\hat{\vy}$ for the original problem~\eqref{eq:MRFmp3}  is called \emph{rounding}.
 A standard approach for rounding consists in computing a rank-1 approximation of $\MY$ (which can be done via singular value decomposition) 
 and rounding the entries of the resulting vector in $\{-1;+1\}$. 
We refer to  $\hat{\vy}$ as the \emph{rounded estimate} and we call $\froundedmp$ the objective value attained by $\hat{\vy}$ in~\eqref{eq:MRFmp3}.
 %Specifically, calling the rank-1 approximation $\vy \in \Real{\NK}$, and recalling that
 % More specifically, calling $\vy$
 % but since this matrix is not necessarily 
 % is not necessarily feasible for the original problem~\eqref{eq:MAP-MRF} 

\myparagraph{Convergence} While dual ascent is a popular optimization technique, few convergence results are available in the literature. For instance, 
dual ascent is known to converge when the original objective is strictly convex~\cite{Tseng90siam}. 
Currently, % we do not have theoretical results to prove convergence in \dars, although 
we observe that \dars converges when the stepsize $\alpha$ in~\eqref{eq:gradAscent} is sufficiently small. We prove the following per-instance performance guarantees.
\begin{proposition}[Guarantees in \dars]
\label{prop:dars:guarantees}
If the dual ascent iterations converge to a value $\vlambda^\star$ (i.e., the dual iterations reach a solution where the gradient in~\eqref{eq:gradAscent} is zero) then the following properties hold:
\begin{itemize}
	\item let $\MR^\star$ be a (column) rank-deficient second-order critical point of problem~\eqref{eq:RRT2} 
	with $\ML_{\lambda} \doteq \ML + \sum_{i=1}^{\nrNodes} \vlambda^\star_i \MU_i$, 
	then the matrix $\MY^\star \doteq (\MR^\star)(\MR^\star)\tran$ is an optimal solution for the standard SDP relaxation~\eqref{eq:SDPstandard};
	\item let $\fsdpmp$ be the (optimal) objective value attained by $\MY^\star$ in the standard SDP relaxation~\eqref{eq:SDPstandard}, 
	$\foptmp$ be the optimal objective of~\eqref{eq:MRFmp3}, and $\froundedmp$ the objective attained by the rounded solution $\hat{\vy}$, then it holds $\froundedmp - \foptmp \leq \froundedmp - \fsdpmp$.
	% then $\fsdpmp \leq \fopt$, where $\fopt$ be the optimal objective attained by the MAP estimator~\eqref{eq:MRFmp3}.
\end{itemize}
\end{proposition}
% Note that having a lower bound for the optimal objective $\fopt$ allows understanding how suboptimal a given solution is.
% For instance, if  
The proof of Proposition~\ref{prop:dars:guarantees} is given in \isExtended{Appendix B}{the supplemental material~\cite[Appendix B]{Hu18tr-fuses}}.
The first claim in Proposition~\ref{prop:dars:guarantees} ensures that when the dual ascent method converges, it produces an optimal solution
for the standard SDP relaxation~\eqref{eq:SDPstandard}. The second claim states that we can compute an upper-bound on how far the \dars' solution is from optimality ($\froundedmp - \foptmp$) using the rounded objective $\froundedmp$ and the relaxed objective $\fsdpmp$.
%!TEX root = main.tex

\section{\fuses: Fast Unconstrained \\ SEmidefinite Solver} % // Inference on MRFs \\via Smooth Riemannian Optimization
\label{sec:fuses}

In this section we propose a more direct way to obtain a semidefinite relaxation and a remarkably faster solver.
 While \dars is already able to compute an approximate MAP estimate in seconds for large problems, the approach 
 presented in this section requires two orders of magnitude less time to compute a solution of comparable \change{ }{(but slightly inferior)} quality.
 We first present a binary $\{0,1\}$ (rather than $\{-1,+1\}$) matrix formulation (\prettyref{sec:fuses-formulation}) and  
derive an SDP relaxation (\prettyref{sec:fuses-sdp}). We then present a Riemannian staircase approach 
to solve the resulting SDP in real time (\prettyref{sec:fuses-staircases}) 
and discuss performance guarantees (\prettyref{sec:fuses-guarantees}). 
% Finally, we recall rounding strategies in~\prettyref{sec:rounding} and performance guarantees in~\prettyref{sec:guarantees}.
  %!TEX root = main.tex

%%%%%%%%%%%%%%%%%%%%%%%%%%%%%%%%%%%%%%%%%%%%%%%%%%%%%%%%%%%%%%%%%%%%%%%%%%%%%%%%
\subsection{Matrix Formulation}
\label{sec:fuses-formulation}

In this section we rewrite the node variables $\class_i \in \classSet \doteq \{1,\ldots,\nrClasses\}$  as 
an $\nrNodes \times \nrClasses$ binary matrix $\Classes\in \{0,1\}^{\nrNodes\times\nrClasses}$ that is 
such that if an entry in position $(i,j)$ is equal to $1$, then node $i$ has label $j$ and is zero otherwise. 
In other words, the $i$-th row of $\Classes$ is a binary vector that describes the label of node $i$ and 
has a single entry equal to $1$ in position $j$, where $j$ is the label assigned to the node.
This is  a more intuitive parametrization of the problem and indeed leads to a more elegant matrix 
formulation, given as follows.

\begin{proposition}[Binary Matrix Formulation of MAP-MRF]
\label{prop:binaryMat}
% The multi-class image segmentation problem is to assign each pixel $i$ a label $\vec{x}_i\in\{0,1\}^\nrClasses$, where $\nrClasses$ is the number of classes. $\vec{x}_i$ is a vector of size $\nrClasses$ with the $k^{th}$ entry being 1 and the rest being 0 to indicate label $k\in\{1,...,\nrClasses\}$ is assigned to this pixel. Therefore, the labels for all $\nrNodes$ pixels in an image can be expressed as an $\nrNodes\times \nrClasses$ matrix $X$:
% \[
% X=\begin{bmatrix}\vec{x}_1\tran\\\vec{x}_2\tran\\ \vdots \\ \vec{x}_N\tran\end{bmatrix}
% \]
Let $\MG  \in \Real{\nrNodes \times \nrClasses}$ and $\MH \in \Real{\nrNodes \times \nrNodes}$ be defined as follows:
\bea
\begin{split}
\label{eq:GHdefs}
\MG \doteq &
\left\{
\begin{array}{ll}
\MG_{i} =  -\penaltyTerm_{i} \ve_{\measuredClass_i}\tran, & \text{ if } i \in \unarySet\\
\MG_{i} = \zero_{\nrClasses}\tran, & \text{ otherwise}
\end{array}
\right. 
% \matTwo{
% -\penaltyTerm_{1} \ve_{\measuredClass_2}\tran \\
% -\penaltyTerm_{2} \ve_{\measuredClass_2}\tran \\
% \vdots \\
% -\penaltyTerm_{\nrNodes} \ve_{\measuredClass_\nrNodes}\tran
%  },
\\
%\\ %\text{ and }
% \begin{array}{c}
\MH \doteq &
\left\{
\begin{array}{ll}
\MH_{ij} =  -\penaltyTerm_{ij},&  \text{ if } (i,j) \in \binarySet\\
\MH_{ij} = 0, &\text{ otherwise}
\end{array}
\right. \hspace{-3mm}
% \\
% (\MH \in \Real{\nrNodes \times \nrNodes})
% \end{array}
\end{split}
\eea
where $\MG_i$ is the $i$-th row of $\MG$, $\MH_{ij}$ is the entry of $\MH$ in row $i$ and column $j$, 
$\penaltyTerm_{i}$ and $\penaltyTerm_{ij}$ are the coefficients defining the MRF, \cf eq.~\eqref{eq:binary}, 
and  
% represents the unary cost for node $i$, \cf eq.~\eqref{eq:unary},
$\ve_{\measuredClass_i}$ is a vector with a unique nonzero entry equal to 1 in position $\measuredClass_i$ ($\measuredClass_i$ is the measured label for node $i$).
% and $\penaltyTerm_{ij}$ represents the pairwise costs between nodes $i$ and $j$, \cf eq.~\eqref{eq:binary}.
% Let $\MH$ be an $\nrNodes \times \nrNodes$ matrix.
% The $(i,j)^{th}$ entry in $\MH$ is set to $\MH_{ij}=-\frac{1}{2}\penaltyTerm_{ij}$, where . 
% Let $\MG$ be an $\nrNodes\times \nrClasses$ matrix,
% where the $i$-th row describes the measured label of pixel $i$ and is \todo{finish}.
%  and $\vec{g}_i = -\vec{c}_i$, where vector $\vec{c}_i$ is the predicted pixel assignment to pixel $i$:
% \[
% G=\begin{bmatrix}\vec{g}_1\tran\\\vec{g}_2\tran\\ \vdots \\ \vec{g}_N\tran\end{bmatrix}
% \]
Then the MAP estimator~\eqref{eq:MRF1} can be equivalently written as: %\textit{Problem 1 (Maximum-likelihood estimation for CRF)}
\beq
\label{eq:binaryMat}
\begin{array}{rl}
\min_{\MX} & \trace{\MX\tran \MH \MX} + \trace{\MG \MX\tran} \\
\subject   & \diag{\MX \MX\tran} = \ones_{\nrNodes} \\% \MX \ones_{\nrClasses} = \ones_{\nrNodes}  \\
           & \MX\in\{0,1\}^{\nrNodes\times \nrClasses}
\end{array}
\eeq
%% and the second  enforces the matrix $\MX$ to be binary.
\end{proposition}

The equivalence between~\eqref{eq:MRF1} and~\eqref{eq:binaryMat} is proven in \isExtended{Appendix C}{\cite[Appendix C]{Hu18tr-fuses}}.
We note that the constraint $\diag{\MX \MX\tran} = \ones_{\nrNodes}$ in~\eqref{eq:binaryMat} (contrarily to the $\diag{\cdot}$ 
constraint in~\eqref{eq:MRFmp2})  imposes that each node has a unique label when $\MX\in\{0,1\}^{\nrNodes\times \nrClasses}$.
\subsection{Novel Semidefinite Relaxation}
\label{sec:fuses-sdp}

This section presents a semidefinite relaxation of~\eqref{eq:binaryMat}.
Towards this goal, 
% of the semidefinite relaxation of~\eqref{eq:binaryMat}, 
% we apply two small  changes ~\eqref{eq:binaryMat}. First, we substitute the constraint $\MX \ones_{\nrClasses} = \ones_{\nrNodes}$ 
% with the constraint:
% \beq 
% \label{eq:diagConstraint}
% \diag{\MX \MX\tran} = \ones_{\nrNodes}
% \eeq
% The substitution does not alter the problem since, for a binary matrix $\MX$, 
% they are both enforcing each row of $\MX$ to have a single non-zero element.\footnote{The constraint $\diag{\MX \MX\tran} = \ones_{\nrNodes}$ 
% enforces the (squared) norm of each row to be $1$, which for binary $\MX$ is equivalent to a unique $1$ in each row.} 
%  Second, 
 % somehow similarly to standard SDP relaxations, 
 we first homogenize the cost by lifting the problem to work on a larger variable:
 \beq 
 \label{eq:homogenization2}
 \MV \doteq \matTwo{\MX \\ \eye_\nrClasses} \;\;\; \left(\text{note: } \MV \MV\tran = \matTwo{\MX \MX\tran & \MX \\ \MX\tran & \eye_\nrClasses}\right)
\eeq
where $\eye_\nrClasses$ is the $\nrClasses\!\!\times\!\!\nrClasses$ identity matrix.
The reparametrization is given as follows.

\begin{proposition}[Homogenized Binary Matrix Formulation]
\label{prop:homBinaryMat}
Let us define %the matrix %$\MQ \in \Real{(\nrNodes+\nrClasses)\times(\nrNodes+\nrClasses)}$ be defined as 
$\MQ \doteq {\scriptsize \matTwo{\MH & \frac{1}{2} \MG \\
\frac{1}{2} \MG\tran & \zero}} \in \Real{(\nrNodes+\nrClasses)\times(\nrNodes+\nrClasses)}$. 
%
% \beq
% \MQ \doteq \matTwo{
% \MH & \frac{1}{2} \MG \\
% \frac{1}{2} \MG\tran & \zero
% }
% \eeq
%
Then the MAP estimator~\eqref{eq:binaryMat} can be rewritten as: %\textit{Problem 1 (Maximum-likelihood estimation for CRF)}
\beq
\label{eq:homBinaryMat}
\begin{array}{rcl}
\foptzo = & \min_{\MV} & \trace{\MV\tran \MQ \MV}  \\
&\subject   & \diag{[\MV \MV\tran]_{tl}} = \ones_{\nrNodes}  \\
&	       & [\MV \MV\tran]_{br} = \eye_{\nrClasses}  \\
&           & \MV\in\{0,1\}^{(\nrNodes+\nrClasses)\times(\nrClasses)}
\end{array}
\tag{\Pzo}
\eeq
where $[\MV \MV\tran]_{tl}$ denotes the $(\nrNodes \times \nrNodes)$ top-left block of the matrix $\MV \MV\tran$, \cf~\eqref{eq:homogenization2}, 
(the corresponding constraint rewrites the first constraint in~\eqref{eq:binaryMat}), %  and enforces a single ``1'' in each row of $\MV$
%and the corresponding constraint corresponds to~\eqref{eq:diagConstraint} (enforces a single ``1'' in each row of $\MV$), 
and where $[\MV \MV\tran]_{br}$ denotes the $(\nrClasses \times \nrClasses)$ bottom-right block of $\MV \MV\tran$, \cf~\eqref{eq:homogenization2}. 
\omitted{which is constrained to be the identity matrix according 
to~\eqref{eq:homogenization2}.} 
\end{proposition}

At this point it is straightforward to derive a semidefinite relaxation, by noting that 
$\trace{\MV\tran \MQ \MV} = \trace{\MQ \MV \MV\tran}$ and by observing that $\MV \MV\tran$ 
is a $(\nrNodes+\nrClasses)\times(\nrNodes+\nrClasses)$ symmetric positive semidefinite matrix of rank $\nrClasses$.

% \todo{unfortunately, we relaxed more than the rank constraint here}
\begin{proposition}[Semidefinite Relaxation]
\label{prop:newSDP}
The following SDP is a convex relaxation of the MAP estimator~\eqref{eq:homBinaryMat}:
\beq
\vspace{-0.2cm}
\label{eq:newSDP}
\begin{array}{rcl}
\fsdpzo \doteq & \min_{\MZ} & \trace{\MQ \MZ}  \\
&\subject   & \diag{[\MZ]_{tl}} = \ones_{\nrNodes}  \\
&	       & [\MZ]_{br} = \eye_{\nrClasses}  \\
&           & \MZ \succeq 0
\end{array}
\tag{\Szo}
\eeq
where $[\MZ]_{tl}$ and $[\MZ]_{br}$ are the $(\nrNodes \times \nrNodes)$ top-left block 
and the $(\nrClasses \times \nrClasses)$ bottom-right block
of the matrix $\MZ$, respectively, and we dropped the rank-$\nrClasses$ constraint for $\MZ$.
\end{proposition}

\subsection{Accelerated Inference via the Riemannian Staircase}
\label{sec:fuses-staircases}

% ~\cite{Keuchel03pami} concluded that working in such high-dimensional state space is ``necessary'' to get rid of the non-convexity of the problem: 
% here we leverage recent results on low-rank semidefinite programming to show that we can compute optimal solutions while working in much smaller spaces.
We now present a fast specialized solver to solve the SDP~\eqref{eq:newSDP} in real time and for large problem instances.
% The SDP~\eqref{eq:newSDP} is a convex program and can be solved in polynomial time using general-purpose solvers, including interior-point methods~\cite{Boyd04book}. 
% Unfortunately, interior-point-method-based solvers are relatively slow and do not scale to problem with more that few hundred nodes, see e.g.,~\cite{Keuchel03pami}. %
% %and the experiments in Section~\ref{sec:experiments}. 
% Since superpixel semantic segmentation typically involves thousands of variables (while pixel-level segmentation 
% may involve millions), this section presents a specialized solver that can solve~\eqref{eq:newSDP} in milliseconds, even for MRFs with thousands of variables.
% The high-computational cost of general-purpose SDP solvers is due to the need to manipulate and store the large and typically dense matrix $\MZ$. On the other hand, we obtained the SDP~\eqref{eq:newSDP} by relaxing a rank-$\nrClasses$ constraint and we empirically observe
% that the solution 
% % and we observe experimentally that the solutions 
% $\MZ$ of~\eqref{eq:newSDP} is still low rank (we will formalize this point in Proposition~\ref{prop:boumal}). 
% % This suggests that the matrix $\MZ$ can be represented with  
% This allows leveraging recent results in low-rank SDP solvers~\cite{Burer03mp,Burer04mp,Boumal16nips}. 
Similarly to Section~\ref{sec:dars-staircase}, we use the Burer-Monteiro method~\cite{Burer03mp}, 
% to solve the SDP~\eqref{eq:newSDP} efficiently. 
% The Burer-Monteiro method~\cite{Burer03mp}, 
which replaces the matrix $\MZ$ in~\eqref{eq:newSDP} with a rank-$r$ product $\MR \MR\tran$: % rewriting~\eqref{eq:newSDP} as:     
\beq
\label{eq:RRT21}
\begin{array}{rcl}
& \min_{\MR} & \trace{\MQ \MR \MR\tran}  \\
&\subject   & \diag{[\MR \MR\tran]_{tl}} = \ones_{\nrNodes}  \\
&	       & [\MR \MR\tran]_{br} = \eye_{\nrClasses}  
\end{array}
\eeq
where $\MR \in \Real{\nrNodes \times r}$ (for a suitable rank $r$), and 
where the constraint $\MZ \succeq 0$ in~\eqref{eq:newSDP}  becomes redundant after the substitution, and is dropped. 
% While it seems we took a step backward by reparametrizing $\MZ$ as $\MR \MR\tran$ (the new parametrization seems similar to the original parametrization $\MV \MV\tran$and~\eqref{eq:RRT} is again a nonconvex problem), we remark that $\MR$ is  an overparametrization whenever $r > \nrClasses$ ($\MR$ has more columns than $\MV$), which will be useful in Proposition~\ref{prop:boumal}.
% %This observation will be fundamental for XXX.

% Following Boumal\setal~\cite{Boumal16nips} 
Similarly to Section~\ref{sec:dars-staircase},
we note that the constraint set in~\eqref{eq:RRT21} describes a smooth manifold, and in particular a product of Stiefel manifolds. Specifically, we observe that
% that the top part of 
% To make this apparent, we first recall that 
% the (transposed) \textit{Stiefel manifold} is defined as~\cite{Boumal16nips}:
% %
% \beq
% \Stiefel{d}{p} = \{\MM\in\Real{d\times p}\colon \MM \MM\tran = \eye_d\}
% \eeq
% %
% Then, we observe that 
$\diag{[\MR \MR\tran]_{tl}} = \ones_{\nrNodes}$ can be written as 
$\MR_i \MR_i\tran = 1$, $i=1,\ldots,\nrNodes$, which is equivalent to saying that 
$\MR_i \in \Stiefel{1}{r}$ for $i=1,\ldots,\nrNodes$. 
Moreover, denoting with $\MR_b$ the block matrix including the last 
$\nrClasses$ rows of $\MR$, the constraint $[\MR \MR\tran]_{br} = \eye_{\nrClasses}$ can be written as $\MR_b \MR_b\tran = \eye_\nrClasses$, which 
is equivalent to saying that $\MR_b \in \Stiefel{\nrClasses}{r}$.
The two observations above allow concluding that the matrix $\MR$ 
belongs to the product manifold $\Stiefel{1}{r}^\nrNodes \times \Stiefel{\nrClasses}{r}$.
% \beq
% \StiefelProduct = 
% \overbrace{\Stiefel{1}{r} \times \Stiefel{1}{r} \times \ldots \Stiefel{1}{r}}^{\nrNodes \text{ times } } \times \Stiefel{\nrClasses}{r}
% \eeq
%
Therefore, we can rewrite~\eqref{eq:RRT21} as an \emph{unconstrained} optimization on manifold: % $\StiefelProduct$:
\beq
\label{eq:RRT22}
\min_{\MR \in \Stiefel{1}{r}^\nrNodes \times \Stiefel{\nrClasses}{r}} \trace{\MQ \MR \MR\tran} 
\tag{\Rzo}
\eeq
The formulation~\eqref{eq:RRT22} is non-convex but one can find local minima efficiently using iterative methods~\cite{Boumal16nips,Rosen16wafr-sesync}. 
We can again adapt the result from Boumal\setal~\cite{Boumal16nips} to conclude that rank-deficient local solutions of~\eqref{eq:RRT22} can be mapped back to global solutions of the semidefinite relaxation~\eqref{eq:newSDP}. 

 \begin{proposition}[Optimality Conditions for~\eqref{eq:RRT22}, Corollary 8 in~\cite{Boumal16nips}]
 \label{prop:boumal2}
 If $\MR \in  \Stiefel{1}{r}^\nrNodes \times \Stiefel{\nrClasses}{r}$ is a (column) rank-deficient second-order critical point of problem~\eqref{eq:RRT22}, then $\MR$ is a global optimizer of~\eqref{eq:RRT22}, and 
 $\MZ^\star \doteq \MR\MR\tran$ is a solution of the semidefinite relaxation~\eqref{eq:newSDP}.
 \end{proposition}

% The previous proposition provides a way to 
% % ensures that when
% % local solutions (second-order critical points) of~\eqref{eq:RRT22} are rank deficient, then they can be 
% % mapped back to global solutions of~\eqref{eq:newSDP}, 
% % hence providing a way to 
% solve~\eqref{eq:newSDP} efficiently via nonlinear optimization on~\eqref{eq:RRT22}. 
Similarly to Section~\ref{sec:dars-staircase}, we can adopt a Riemannian staircase method, where one solves~\eqref{eq:RRT22} for increasing values of $r$ till a rank-deficient solution is found. %Similarly to related problems~\cite{Rosen16wafr-sesync}, in Section~\ref{sec:experiments} we will notice that 
In \emph{all tests} we performed, \change{a single step of the staircase ($r = \nrClasses+1$ in this case) was sufficient to find a rank-deficient solution.}{at most two steps of the staircase (initialized at $r = \nrClasses+1$) were sufficient to find a rank-deficient solution.} 

  \omitted{%!TEX root = main.tex

%%%%%%%%%%%%%%%%%%%%%%%%%%%%%%%%%%%%%%%%%%%%%%%%%%%%%%%%%%%%%%%%%%%%%%%%%%%%%%%%
\subsection{Rounding}
\label{sec:rounding}

This section discusses how to compute a discrete label assignment from the solution of the SDPs~\eqref{eq:newSDP} and~\eqref{eq:newSDP2}.
Let us call $\MZopt$ and $\MZoptTwo$ the optimal solutions of~\eqref{eq:newSDP} and~\eqref{eq:newSDP2};
when using the Riemannian staircase method of Section~\ref{sec:riemannianStaircase}, then $\MZopt \doteq (\MR^\star)(\MR^\star)\tran$ 
where $\MR^\star$ is the rank-deficient solution produced by solving~\eqref{eq:RRT2}.
We note that we substituted $\MZ$ for the product $\MV\MV\tran$ in~\eqref{eq:homogenization2} hence \emph{ideally} 
we would expect to read the MAP estimate $\MXopt$ from the top-right $\nrNodes \times \nrClasses$ submatrix of $\MZopt$ and $\MZoptTwo$.
Unfortunately, this is only the case when the relaxation is tight (Proposition~\ref{prop:tightness2}) which is typically not the case in practice, 
while in general the top-right submatrix contains fractional (rather than binary) entries. 

Let us call $\MXrel$ and $\MXrelTwo$ the top-right $\nrNodes \times \nrClasses$ submatrices of $\MZopt$ and $\MZoptTwo$, respectively.
These matrices are not feasible for the original MAP estimation problem~\eqref{eq:binaryMat}, since they are not binary. 
Therefore, we have to apply \emph{rounding} strategies to recover a binary estimate from $\MXrel$ and $\MXrelTwo$.
We review four rounding strategies in the following.

\myparagraph{Naive rounding}
We already observed that the $i$-th row of $\MXrel$ (resp. $\MXrelTwo$) describes the label assignment of node $i$.
Naive rounding takes the largest element (in absolute value) of the $i$-th row of $\MXrel$ (resp. $\MXrelTwo$) to be the label assignment for node $i$; 
more formally, if the largest entry (in absolute value) of row $i$ is in position $j$ (column index), the naive rounding 
assigns node $i$ to belong to the $j$-th class.
% Therefore a naive rounding would parse each row, and for the $i$-th row it would assign the largest entry (in absolute value) of the row to the 
% the label of node $i$. 

\myparagraph{Winner-Take-All strategy~\cite{Schellewald05cvpr}}
We describe this strategy for rounding~\eqref{eq:newSDP}, while we remark that the very sample approach applies to~\eqref{eq:newSDP2}. 
The winner-take-all rounding strategy looks for the binary matrix $\MXrounded$ closest to $\MXrel$:
\beal
\label{eq:win1}
\min_\MXrounded &  \| \MXrounded - \MXrel \|^2_\frob \\
\subject &  \MXrounded \in \{0,1\}^{\nrNodes \times \nrClasses}, \quad \MXrounded \ones_{\nrClasses}=\ones_{\nrNodes}
\eeal
where $\| \cdot \|_\frob$ denotes the Frobenius norm of a matrix. 
We note that $\| \MXrounded - \MXrel \|^2_\frob = \trace{(\MXrounded - \MXrel) (\MXrounded - \MXrel)\tran} = - 2 \trace{\MXrel\tran \MXrounded} + \text{constant}$. 
Hence~\eqref{eq:win1} is equivalent to:
 \beal
\label{eq:win2}
\max_\MXrounded &  \trace{\MXrel\tran \MXrounded} \\
\subject &  \MXrounded \in \{0,1\}^{\nrNodes \times \nrClasses}, \quad \MXrounded \ones_{\nrClasses}=\ones_{\nrNodes}
\eeal
We further note that~\eqref{eq:win2} separates into $\nrNodes$ independent problems, where the $i$-th  problem only involves 
the $i$-th row of $\MXrounded$ (i.e., the label of the $i$-th node), that we denote with $\hat{\vxx}_i$:
 \beal
\label{eq:win3}
\max_{\hat{\vxx}_i} &  [\MXrel]_i\tran \hat{\vxx}_i \\
\subject &  \hat{\vxx}_i \in \{0,1\}^{\nrNodes}, \quad \ones_{\nrClasses}\tran \hat{\vxx}_i = 1
\eeal
The feasible set of~\eqref{eq:win3} (for each $i=1,\ldots,\nrNodes$) is the set of vectors with a single nonzero element equal to $1$, 
hence the solution of~\eqref{eq:win3}  can be easily computed by assigning the nonzero entry to the largest element of $[\MXrel]_i\tran$.
This strategy is equivalent to naive rounding for~\eqref{eq:newSDP2}, but it may produce different results for~\eqref{eq:newSDP}, since $\MXrel$ 
 may have negative entries. 

\myparagraph{Sampling~\cite{Schellewald05cvpr}}
Due the constraints in~\eqref{eq:newSDP2}, $\MXrelTwo$ satisfies $\MXrelTwo \geq 0$ and $\MXrelTwo \ones_{\nrClasses}=\ones_{\nrNodes}$ 
hence each row $[\MXrelTwo]_i$ of $\MXrelTwo$ has nonnegative entries that sum up to one. Therefore, we can interpret the entries of $[\MXrelTwo]_i$ 
as a probability distribution and sample the $i$-th label from such distribution. Sampling-based rounding randomly samples labels for all nodes and 
repeats the sampling $T$ times, choosing the rounded estimate as the one that attains the smallest cost. 
This strategy cannot be directly applied to~\eqref{eq:newSDP} since $\MXrel$ may have negative entries and rows do not sum up to 1. 

\myparagraph{Randomized projection~\cite{Goemans95acm}} %~\cite[Section 3.3]{Keuchel03pami}}
% Randomized projection methods generate random vectors to project the fractional solutions $\MXrel$ and $\MXrelTwo$ before rounding.
The most popular method in this family, 
is the \emph{randomized hyperplanes} method, % a popular rounding  method for 
commonly applied to binary segmentation problems, thanks to to the seminal result of Goemans and Williamson~\cite{Goemans95acm}, who 
provide suboptimality guarantees for the rounded result. 
In the multi-class case, a randomized project method~\cite{Wang15cvpr}  works as follows: it  
factorizes the matrix $\MZopt$ as $\MZopt = \M{\Psi} \; \M{\Psi}\tran$ with $\M{\Psi} \in \Real{\nrClasses \times d}$ where $d$ is the rank of $\MZopt$.
The factorization can be obtained using singular value decomposition. 
Then, the randomized project method samples a projection matrix $\MP \in \Real{d \times \nrClasses}$ with entries drawn from a normal distribution with 
zero mean and unit variance. Finally, the method selects the labels from the product $\M{\Psi} \MP$ in a winner-take-all manner. 
The sampling procedure is repeated $T$ times and the rounded estimate is chosen as the one attaining the smallest cost.}
  \isExtended{%!TEX root = main.tex

%%%%%%%%%%%%%%%%%%%%%%%%%%%%%%%%%%%%%%%%%%%%%%%%%%%%%%%%%%%%%%%%%%%%%%%%%%%%%%%%
\subsection{\fuses: Summary, Convergence, and Guarantees}
\label{sec:fuses-guarantees}

We name \fuses (\emph{Fast Unconstrained SEmidefinite Solver}) the approach presented in this section. 
Contrarily to \dars, \fuses is extremely simple and only requires solving the rank-restricted problem~\eqref{eq:RRT22}, which 
can be solved using iterative methods, such as the Truncated Newton method.
Besides its simplicity, \fuses is guaranteed to converge to the solution of the SDP~\eqref{eq:newSDP} for increasing values of the rank $r$ 
(Proposition~\ref{prop:boumal2}).\omitted{, while in practice we observe that choosing $r = \nrClasses+1$ always produces a rank-deficient critical point, 
which is optimal according to  Proposition~\ref{prop:boumal2}.}

\myparagraph{Rounding} Upon convergence, \fuses produces a matrix $\MZ^\star$. 
Similarly to \dars, we obtain a rounded solution by computing a rank-K approximation of $\MZ^\star$ and rounding the corresponding matrix in
 $\{0,1\}$ (i.e., we assign the largest element in each row to 1 and we zero out all the others).
 We denote with $\hat{\MX}$ the resulting estimate and we 
 call $\froundedzo$ the objective value attained by $\hat{\MX}$ in~\eqref{eq:binaryMat}.

Since the SDP~\eqref{eq:newSDP} is a relaxation of the MAP estimator~\eqref{eq:homBinaryMat}, it is straightforward to prove the following 
proposition. 
% We provide the following per-instance performance guarantee:
\begin{proposition}[Guarantees in \fuses]
\label{prop:fuses:guarantees}
% Let $\MR^\star$ denote a (column) rank-deficient second-order critical point of problem~\eqref{eq:RRT22}, 
% 	then the matrix $\MZ^\star \doteq (\MR^\star)(\MR^\star)\tran$ is an optimal solution for the SDP relaxation~\eqref{eq:newSDP}. 
	Let $\fsdpzo$ be the optimal objective attained by $\MZ^\star = (\MR)(\MR)\tran$ in~\eqref{eq:newSDP}, 
 $\foptzo$ be the optimal objective of~\eqref{eq:homBinaryMat}, and $\froundedzo$ be the objective attained by the rounded solution $\hat{\MX}$, then 
 $\froundedzo - \foptzo \leq \froundedzo - \fsdpzo$.
\end{proposition}

Again, we can use Proposition~\ref{prop:fuses:guarantees} to compute how far the solution computed by \fuses is from the optimal objective 
attained by the MAP estimator.

}{%!TEX root = main.tex

%%%%%%%%%%%%%%%%%%%%%%%%%%%%%%%%%%%%%%%%%%%%%%%%%%%%%%%%%%%%%%%%%%%%%%%%%%%%%%%%
\subsection{\fuses: Summary, Convergence, and Guarantees}
\label{sec:fuses-guarantees}

We name \fuses (\emph{Fast Unconstrained SEmidefinite Solver}) the approach presented in this section. 
Contrarily to \dars, \fuses is extremely simple and only requires solving the rank-restricted problem~\eqref{eq:RRT22}, which 
can be solved using iterative methods. \omitted{, such as the Truncated Newton method.}
Besides its simplicity, \fuses is guaranteed to converge to the solution of the SDP~\eqref{eq:newSDP} for increasing values of the rank $r$ 
(Proposition~\ref{prop:boumal2}).\omitted{, while in practice we observe that choosing $r = \nrClasses+1$ always produces a rank-deficient critical point, 
which is optimal according to  Proposition~\ref{prop:boumal2}.}

\myparagraph{Rounding} Upon convergence, \fuses produces a matrix $\MZ^\star$. 
Similarly to \dars, we obtain a rounded solution by computing a rank-K approximation of $\MZ^\star$ and rounding the corresponding matrix in
 $\{0,1\}$ (i.e., we assign the largest element in each row to 1 and we zero out all the others).
 We denote with $\hat{\MX}$ the resulting estimate and we 
 call $\froundedzo$ the objective value attained by $\hat{\MX}$ in~\eqref{eq:binaryMat}.

Since the SDP~\eqref{eq:newSDP} is a relaxation of the MAP estimator~\eqref{eq:homBinaryMat}, it is straightforward to prove~\cite[Proposition 7]{Hu18tr-fuses}
that $\froundedzo - \foptzo \leq \froundedzo - \fsdpzo$, where $\fsdpzo$ is the optimal objective of~\eqref{eq:newSDP}, 
 and $\foptzo$ is the optimal objective of~\eqref{eq:homBinaryMat}. %, and $\froundedzo$ be the objective attained by the rounded solution $\hat{\MX}$. 
}
  \omitted{%!TEX root = main.tex

%%%%%%%%%%%%%%%%%%%%%%%%%%%%%%%%%%%%%%%%%%%%%%%%%%%%%%%%%%%%%%%%%%%%%%%%%%%%%%%%
\section{Learning MRF Potentials \\ for Semantic Segmentation}
\label{sec:kkt}

\todo{complete this section}

inverse KKT \url{http://journals.sagepub.com/doi/abs/10.1177/0278364917745980}

\grayout{
\bit
\item typical approaches sets $\lambda$ heuristically 
\item \cite{Szeliski08pami-surveyMRF} stresses that near optimal solutions do not necessarily mean more accurate results: problem with cost function design
\item we propose a more grounded way..
\eit
}}
%!TEX root = main.tex

% % %%%%%%%%%%%%%%%%%%%%%%%%%%%%%%%%%%%%%%%%%%%%%%%%
% \input{fig-segmentation-snapshots2}
% % %%%%%%%%%%%%%%%%%%%%%%%%%%%%%%%%%%%%%%%%%%%%%%%%

\section{Experiments}
\label{sec:experiments}

This section evaluates the proposed \change{MRF solvers}{approaches, \fuses and \dars,} on semantic segmentation problems, comparing their performance against \change{the state of the art.}{several state-of-the-art MRF solvers.}
\omitted{
Our experiments show that
(i) \fuses and \dars produce near-optimal solutions, attaining an objective within \tocheck{2\%} of the optimum,
%with a suboptimality gap between \tocheck{1-2\%} from the optimum,
(ii) \fuses is more than two orders of magnitude faster than \dars while attaining similar solution quality,
(iii) \fuses is as fast as local search methods while being a global solver. 
}

%%%%%%%%%%%%%%%%%%%%%%%%%%%%%%%%%%%%%%%%%%%%%%%%%%%%%%%%%%%%%%%%%%%%%%
\subsection{\fuses and \dars: Implementation Details} 
We implemented \fuses and \dars in C++ using Eigen's sparse matrix manipulation and leveraging the optimization suite developed in~\cite{Rosen16wafr-sesync}. Sparse matrix manipulation is crucial for speed and memory reasons, since the involved matrices are very large. For instance in \dars, the matrix
$\ML_{\lambda}$ in~\eqref{eq:RRT2} has size $(\nrNodes\nrClasses+1) \times (\nrNodes\nrClasses+1)$ where typically $\nrNodes > 10^3$ and $\nrClasses > 20$.
\omitted{Similarly in \fuses, the smaller matrix $\MQ$ in~\eqref{eq:RRT22} has size $(\nrNodes+\nrClasses) \times (\nrNodes+\nrClasses)$ and is sparse. }
% Allocating and manipulating dense matrices of this size would be impractical for computational and memory reasons. 
We initialize the rank $r$ of the Riemannian Staircase to be $r=2$ for \dars and $r = \nrClasses+1$ for \fuses (this is the smallest rank for which we expect a rank-deficient solution). The Riemannian optimization problems~\eqref{eq:RRT2} and~\eqref{eq:RRT22} are solved iteratively using the \emph{truncated-Newton trust-region} \change{}{(TNT)}  method. We refer the reader to~\cite{Rosen17irosws-SEsyncImplementation} for a description of the implementation of a truncated-Newton trust-region method. 
% (cite "Trust-Region Methods" by Conn, Gould, and Toint). 
% The sparse matrix class is used to encode data matrix $\MQ$ to save computation speed and memory space. 
% The first Riemannian staircase by default initializes $r=K+1$ such that one iteration is sufficient to find rank-$K$ solution with optimality guarantee. The quadratic unconstrained optimization problem on manifold is solved iteratively using truncated-Newton trust-region algorithm (cite "Trust-Region Methods" by Conn, Gould, and Toint).
\omitted{ 
Although the solution of~\eqref{eq:RRT2} and~\eqref{eq:RRT22} is independent of the initialization, we generally initialize $\MR$ using the unary potentials of the MRF to reduce convergence time. 
 }As in~\cite{Rosen17irosws-SEsyncImplementation}, we use the Lanczos algorithm to check that~\eqref{eq:RRT2} and~\eqref{eq:RRT22} converged to rank-deficient second-order critical points, which are optimal according to Proposition~\ref{prop:boumal} and Proposition~\ref{prop:boumal2}, respectively.
% We use spectra library (cite here?) to compute the minimum eigenvalue of $\MZ = \MR \MR\tran$ to check optimality at the end of each Riemannian staircase.
If the optimality condition is not met, the algorithm proceeds to the next step of the Riemannian staircase, repeating the optimization with the rank $r$  increased by 1. \change{}{In all experiments, \fuses found an optimal solution within the first two iterations of the staircase, while we}  
%(initialized at $r = \nrClasses+1$)}, while we 
 observed that the rank in \dars (initially $r=2$) sometimes increases to $6-8$.
%, hence meeting the assumptions of Proposition~\ref{prop:dars:guarantees},
 % In practice, for all experiments we run, only one iteration is needed. Therefore, after finding the minimum, we simply apply winter-take-all rounding scheme to return the final labelings. 
% We also implemented another version of FUSES conbining vector formulation (\todo{add vector formulation to section III}) and dual ascent, denoted as FUSES2DA in this section. The implementation is almost the same, except the data matrix $\MQ$ is updated in each dual ascent iteration. For this formulation, we also falso ind that using random initialization leads to faster convergence.
% \todo{ mention implementation details, C++, sparse eigen, libraries we used (not many details needed) - choice of parameters (nr iterations, etc.) - ~10 lines} 
\isExtended{}{
\change{}{In \dars, we limit the number of dual ascent iterations to $T=1000$, and we terminate iterations when the gradient in~\eqref{eq:gradAscent} has norm smaller than $0.5$. Using a constant stepsize $\alpha = 0.005$ ensured convergence in all tests. The interested reader can find a list of all relevant parameters of our implementation in~\cite{Hu18tr-fuses}.}} 
\change{}{Our implementation has been made available online at \url{https://github.mit.edu/SPARK/sdpSegmentation}.}

\isExtended{
\change{}{\myparagraph{Parameter Choice} 
The proposed techniques are based in the Riemannian staircase method and they look for rank-deficient solutions for increasing values of the rank $r$. As mentioned above, 
we use the initial value of $r=2$ for \dars and $r = \nrClasses+1$ for \fuses. 
After solving the rank-constrained SDP for a given value of $r$, each technique checks if the resulting solution is rank deficient (in which case an optimal solution is found), or the techniques moves to the next step of the staircase ($r \leftarrow r+1$). The thresholds used to check that the Riemannian optimization converged (gradient norm and relative decrease in the objective) and that
the resulting solution is rank deficient (eigenvalue tolerance) are given in Table~\ref{table:TNT_params}.
The table also reports parameters governing the trust-region method (maximum number of iterations, initial radius and parameters $\alpha_1$ and $\alpha_2$ the decide how to change the radius), as well
% The Riemannian optimization problems~\eqref{eq:RRT2} and~\eqref{eq:RRT22} are solved iteratively using the \emph{truncated-Newton trust-region} \change{}{(TNT)}  method.
% We use $1\text{e}{-4}$ eigenvalue tolerance to varify the Riemannian optimization finds rank-deficient second-order critical points for \fuses and for each dual ascent iteration of \dars. 
% The Riemannian staircase termination conditions are listed in Table~\ref{table:Riemannian_params}. The thresholds used to check if a solution is rank deficient are given in
% Within TNT iteration, we use a maximum of 2000 conjugate gradient (CG) iterations to compute the minimum. The relevant parameters are listed in Table~\ref{table:TNT_params}. The CG iteration is initialized with trust-region radius $\delta_0=1$ and stops when relative decrease in function value is sufficiently large as defined by $\kappa_{fgr}$ and $\theta$ or the step size is too small. The trust-region radius scales down by $\alpha_1$ or up by $\alpha_2$ based on the relative decrease in function value evaluated using parameters $\eta_1$ and $\eta_2$. The TNT iteration terminates if $\delta$ is below tolerance.
as the Conjugate Gradient (CG) solver used within TNT.
In \dars, we also limit the maximum number of dual ascent iterations to $T=1000$, and we terminate iterations when the gradient in~\eqref{eq:gradAscent} has norm smaller than $0.5$. We adopted a stepsize $\alpha = 0.005$ for the dual ascent iterations.}

%!TEX root = main.tex
\begin{table}[h]
	\centering
	\begin{tabular}{|l|c|c|}
	\hline
	Parameters & \dars & \fuses \\
	\hline
	\hline
	Initial rank &  $2$ & $\nrClasses+1$ \\
	\hline
	Gradient norm tolerance &  {$1\text{e-}{3}$} & {$1\text{e-}{2}$}\\
	%\hline
	%Eigenvalue tolerance & $1\mathrm{e}{-4}$ &  $1\mathrm{e}{-4}$\\
	\hline
	Eigenvalue tolerance & \multicolumn{2}{|c|}{$1\text{e-}{2}$} \\
	\hline
	Relative decrease in function value tolerance &  \multicolumn{2}{|c|}{$1\text{e-}{5}$}\\
	\hline
% 	\end{tabular}
% 	\caption{Riemannian staircase parameters.\label{table:Riemannian_params}}
% \end{table}
% \begin{table}[h]
% 	\centering
% 	\begin{tabular}{|l|c|c|}
% 		\hline
% 		Parameters & \fuses and \dars\\
% 		\hline
		Max TNT iterations & \multicolumn{2}{|c|}{$500$} \\
		%Max truncated CG iterations & 2000 \\
		%\hline
		%Stepsize tolerance & $1\text{e-}{6}$  \\
		\hline
		Initial trust-region radius ($\delta_0$) & \multicolumn{2}{|c|}{$1$} \\
		\hline
		%$\delta$ tolerance & $1\text{e-}{6}$  \\
		%\hline
		%$\kappa_{fgr}$(CG stopping condition) &  0.1 \\
		%\hline
		%$\theta$(CG stopping condition) & 0.5 \\
		%\hline
		Trust region decrease factor ($\alpha_1$) & \multicolumn{2}{|c|}{$0.25$}  \\
		\hline
		Trust region increase factor ($\alpha_2$)  & \multicolumn{2}{|c|}{$2.5$}  \\
		\hline
		Max GC iterations & \multicolumn{2}{|c|}{$2000$} \\
		% \hline
		% Acceptable CG step $\eta_1$ () & \multicolumn{2}{|c|}{$0.5$}\\
		\hline
		Successful CG step ($\eta$) & \multicolumn{2}{|c|}{$0.9$}  \\
		\hline
		Max dual ascent iterations ($T$) & 1000 & -   \\
		\hline
		Dual ascent gradient tolerance & 0.5 & -   \\
		\hline
		Dual ascent gradient stepsize & 0.005 & -   \\
		\hline
	\end{tabular}
	\caption{Parameters used in \dars and \fuses.\label{table:TNT_params}}
\end{table}
%\begin{table}[h]
%	\centering
%	\begin{tabular}{|l|c|}
%		\hline
%		Parameters & \dars\\
%		\hline
%		Max dual accent iterations & 1000 \\
%		\hline
%		$\alpha_0$ (initial step size) & 0.05 \\
%		\hline
%		Equality violation tolerance & 10(?) \\
%		\hline
%	\end{tabular}
%	\caption{Dual accent parameters for \dars}
%\end{table}
}{}

%%%%%%%%%%%%%%%%%%%%%%%%%%%%%%%%%%%%%%%%%%%%%%%%%%%%%%%%%%%%%%%%%%%%%%
\subsection{Setup, Compared Techniques, and Performance Metrics}

% \todo{- Describe the cityscape datasets (5 lines) - describe testing setup - performance metrics}
\myparagraph{Setup} We evaluate \fuses and \dars using the \emph{Cityscapes dataset}~\cite{Cordts16cvpr-cityscapes}, 
which contains a large collection of images of urban scenes 
with pixel-wise semantic annotations. The annotations include 30 semantic classes (e.g., road, sidewalk, person, car, building, vegetation). 
%{of which 19 classes are used for evaluations and the others ignored}.
We first extract superpixels from the images using OpenCV (we obtain around 1000 superpixels per image, unless specified otherwise). Then, the unary terms are obtained using the \change{}{\modelOne pretrained model from} \emph{Bonnet}~\cite{Milioto18arxiv-bonnet}, which uses a CNN to obtain pixel-wise segmentation. \change{}{We restrict our evaluation to 19 (out of 30) classes to be consistent with \emph{Bonnet}. Moreover, we create the unary terms in the MRF by picking the three most likely labels (averaged across all pixels in a superpixel) returned by Bonnet.}
% \change{}{We restrict our evaluation to 19 classes that }
 \change{}{Bonnet achieves 52.65\% accuracy on the Cityscapes dataset, while the accuracy drops after restricting the labeling to the superpixels (see Tables~\ref{table-statistics-512-1K} and~\ref{table-statistics-512-2K}).}
 %(the slight difference with respect to reported $52.3\%$ accuracy might be due to resizing of images)} (\change{Bonnet only uses 20 classes for classification purposes}{}); 
 % the unary potential for each superpixel is set based on the \change{majority of labels for the corresponding set of pixels (we set $\penaltyTerm_{i} = 1$). }{combined probability of all pixels. We filter out labels with less than 15\% probability and then pick up to three most likely labels. We normalize the weights such that $\penaltyTerm_{i}$ sum up to $1$ for each label. } 
Bonnet returns noisy labels for each superpixel and the role of the MRF is to refine the segmentation by encouraging smoothness of nearby labels. 
\change{In practice, since CNNs are typically inaccurate at the boundary between different objects, we expect the use of superpixels and MRF to improve the segmentation results. }{}
%(see \Fig{fig-segmentation-snapshots}).
%  and superpixel algorithm is good at tracing the edges, combining the results of both techiques through a CRF is expected to improve the segmentation results.
% In
% \cite{Milioto18arxiv-bonnet}
% (cite here), which contains fine labeled images taken in various cities. The images are first pre-processed using superpixel segmentation algorithms in OpenCV (cite here?). Then the unary potentials are obtained by first passing original image through Bonnet (cite here), a convolutional neural network (CNN). 
% The unary potential for each superpixel is set based on the majory of labelings. 
% Since CNN is generally bad at determining the edges of different labels and superpixel algorithm is good at tracing the edges, combining the results of both techiques through a CRF is expected to improve the segmentation results.
The binary potentials are modeled as $\penaltyTerm_{ij} = \lambda_1 + \lambda_2\exp(-\beta||c_i-c_j||^2_2)$~\cite[Section 7.2]{Blake11book-MRF}, where $c_i$ denotes the average color vector in superpixel $i$, $\lambda_1$ and $\lambda_2$ are parameters to tune, and
$\beta = (2\langle\|c_i-c_j\|^2_2\rangle)^{-1}$ where "$\langle\cdot\rangle$" represents the sample mean. In our tests, we set \change{$\lambda_1 = 0.04$, $\lambda_2 = 0.22$, and $\beta = 0.00024$}{$\lambda_1 = 0.02$, $\lambda_2 = 0.04$, and $\beta = 0.000173$; $\lambda_1$ and $\lambda_2$ 
are tuned to maximize the accuracy of the optimal solution of~\eqref{eq:MRF1} (computed from CPLEX, see below) on the training data.}
% parameters are  (we tuned these parameters to maximize accuracy of  on the training data)}.
\omitted{, at which the exact solution achieves over $99\%$ accuracy in the Cityscapes' Munster dataset. Snapshots of the Cityscapes segmentation are given in~\Fig{fig-segmentation-snapshots}.}
%In this analysis, we focus on the Lindau dataset in Cityscapes.
\omitted{
The parameters of the binary potentials in the CRF are trained using the Munster dataset in Cityscapes validation set, which constains 174 images. Here, The training set is not used, since Bonnet is pre-trained on these datasets. The test set is not used for this paper because the ground-truth labels are not publicly available. In addition, the main purpose of this paper is to compare the performance of FUSES against other state-of-the-art inference algorithm in optimizing a CRF. Therefore, using validation set is still fair among all techniques.
For evaluation,  we use the Lindau dataset in Cityscapes, which is another validation set containing 59 images with similar color variance across neigheboring superpixel pairs.
}

\myparagraph{Compared techniques}
We compare the proposed techniques against three state-of-the-art methods: 
\emph{$\alpha$-expansion}~\cite{Boykov01pami-graphCut} (label: \aexp).
\emph{Loopy Belief Propagation}~\cite{Weiss01itis-beliefPropagation}  (label: \LBP)
and \emph{Tree-Reweighted Message Passing}~\cite{Wainwright05itis} (label: \TRW).
% FUSES segmentation results are compared with $/alpha$-expansion, loopy belief propagation (LBP) and tree-reweighted messagin passing (TRW-S). 
\omitted{\aexp is an expansion-move-based method, while  \LBP and \TRW are message-passing methods.} 
%These algorithms are the best optimization algorithms to the best of our knowledge. 
We use the implementation of these methods available in the newly-released  \openGM library~\cite{OpenGM2website}.

\myparagraph{Performance metrics}
We evaluate the results in terms of suboptimality, accuracy, and CPU time. 
We measure the suboptimality using three metrics: the \emph{percentage of optimal labels}, the \emph{percentage relaxation gap}, 
and the \emph{percentage rounding gap}. The optimal labels are those that agree with the optimal solution of~\eqref{eq:MRF1}.
The relaxation gap is $(\foptmp - \fsdpmp) / \foptmp$ for \dars, and $(\foptzo - \fsdpzo) / \foptzo$ for \fuses. 
The rounding gap is $(\froundedmp - \foptmp) / \foptmp$ for \dars, and $(\froundedzo - \foptzo) / \foptzo$ for \fuses. 
We compute the optimal labels (and the corresponding optimal objective) using a commercial tool for integer programming, \CPLEX~\cite{CPLEXwebsite}. 
The runtime of \CPLEX increases exponentially in the problem size hence we can only use it offline for benchmarking the proposed solvers.
%on moderate-sized problem instances.
We measure the accuracy using the~\emph{Intersection over Union} (IoU) metric\change{~\cite{GarciaGarcia17arxiv}}{~\cite{Cordts16cvpr-cityscapes}}, 
and record the CPU time for each compared technique. 
% The performance is evaluated based on pixel-wise accuracy, percentage of suboptimal labels and runtime. Pixel-wise accuracy is computed using \todo{add here}. To compute percentage of suboptimal labels, we first solve the CRF exactly using mixed integer solver in CPLEX (cite). Then the labels computed by each algorithm is compared against the exact solution. 

%%%%%%%%%%%%%%%%%%%%%%%%%%%%%%%%%%%%%%%%%%%%%%%%
%!TEX root = main.tex

% \begin{figure}[h!]
% 	\begin{center}
% 		\includegraphics[width=1\columnwidth]{comparison1}
% 		\caption{Convergence plot 1}
% 		\label{fig:comparison1}
% 	\end{center}
% \end{figure}

\begin{figure}[t]{}
%\vspace{-3mm}
	\begin{center}
	\begin{tabular}{cc}%
	\hspace{-0.5cm}
			\begin{minipage}{4.5cm}%
			\centering%
			\includegraphics[width=1.1\columnwidth,trim=0mm 0mm 200mm 0mm,clip]{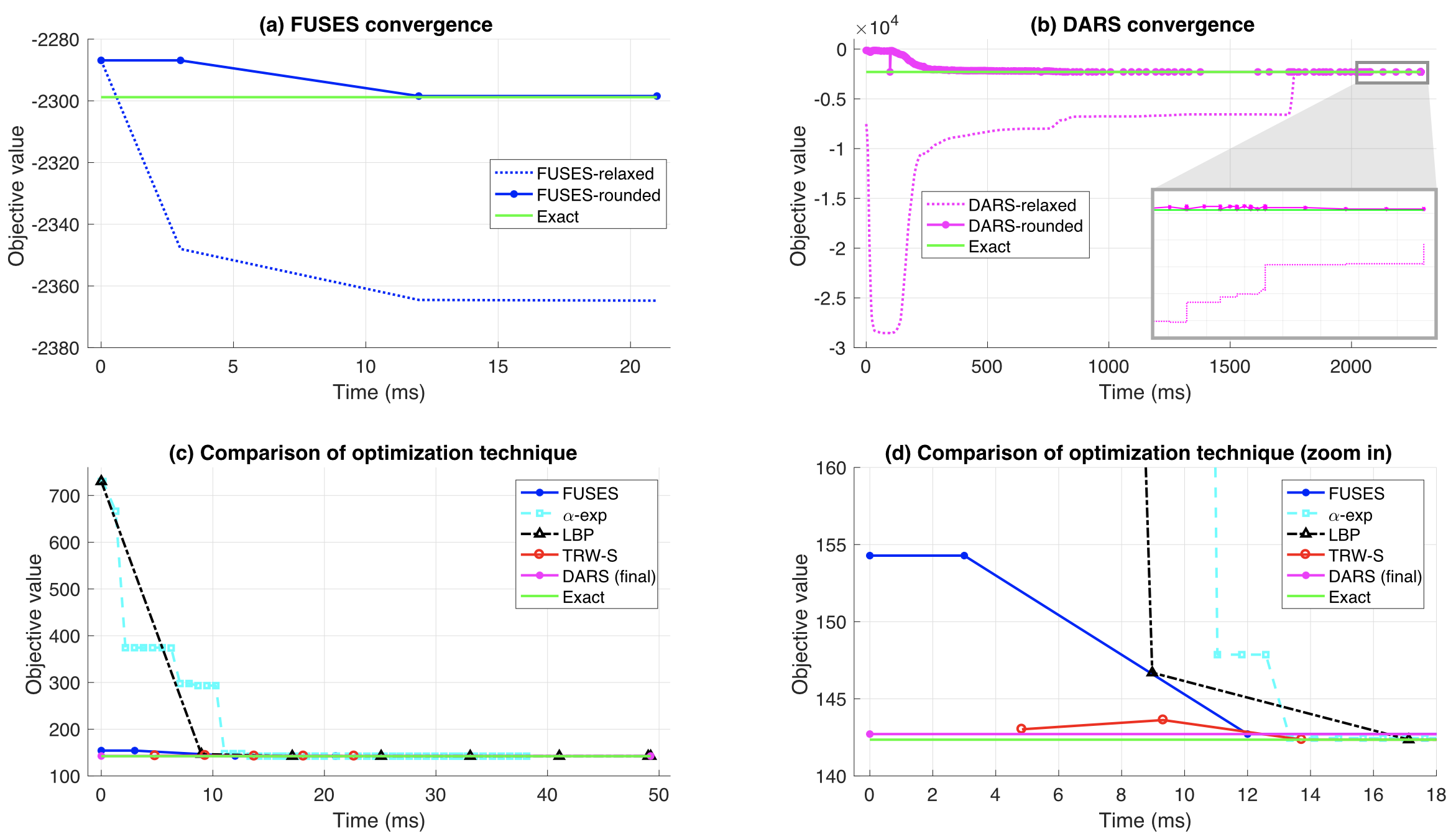}
			\end{minipage}
		& 
		\hspace{-0.6cm}
			\begin{minipage}{4.5cm}%
			\centering%
			\includegraphics[width=1.1\columnwidth,trim=210mm 0mm -10mm 0mm,clip]{convergence_v22}
			\end{minipage}
		\end{tabular}
		%\includegraphics[width=1.1\columnwidth,trim=30mm 2mm 0mm 10mm,clip]{convergence}
		%\vspace{-1cm}
		%\includegraphics[width=1\columnwidth,trim=0mm 0mm 0mm 0mm,clip]{convergence3}
		% \includegraphics[width=1.1\columnwidth,trim=0mm 0mm 0mm 0mm,clip]{convergence_v22}
		%\vspace{-8mm}
		\caption{Convergence results for a single test image: % in the Cityscapes datasets:  
		objective value over time (in milliseconds) for all the compared techniques. (a) Objective vs. time for~\eqref{eq:homBinaryMat}; 
		(b) Objective vs. time for~\eqref{eq:MRFmp3}; (c)-(d) Objective vs. time for~\eqref{eq:MRF1}.\label{fig-convergence-example}}
		\vspace{-8mm}
	\end{center}
\end{figure}
%%%%%%%%%%%%%%%%%%%%%%%%%%%%%%%%%%%%%%%%%%%%%%%%

%%%%%%%%%%%%%%%%%%%%%%%%%%%%%%%%%%%%%%%%%%%%%%%%%%%%%%%%%%%%%%%%%%%%%%
\subsection{Semantic Segmentation Results}\label{sec:exp-results}
% \todo{test against OpenGM for increasing size of the datasets, increasing number of unary terms, and attractive/repulsive potentials - Figure convergence (for a single image) - Table statistics (for the dataset in cityscape) - Suboptimality in Fuses and Fuses2DA}
 \Fig{fig-convergence-example} shows a typical execution of the algorithms for a single image in the Cityscapes \omitted{lindau}dataset.
% Let us start analyzing the results in \Fig{fig-convergence-example}. 
\Fig{fig-convergence-example}(a) shows the convergence of \fuses, reporting 
the relaxed objective attained by iteratively solving~\eqref{eq:RRT22} (\fuses-relaxed),
		the objective of the corresponding rounded estimate at each iteration (\fuses-rounded), and the optimal cost attained by \CPLEX (Exact).
		The approach converges in few milliseconds, and the corresponding rounded estimate settles near the optimal objective.
\Fig{fig-convergence-example}(b) shows the convergence of \dars, reporting 
the relaxed objective attained by~\eqref{eq:RRT2} (\dars-relaxed),
		the objective of the corresponding rounded estimate (\dars-rounded), and the optimal cost from \CPLEX (Exact).
\omitted{The relaxed cost does not decrease monotonically; indeed the dual ascent method constrains the estimate more and more, hence the 
relaxed objective tends to increase and approach the optimal objective.}\dars' relaxed cost does not decrease monotonically.
Moreover, its convergence time is around two orders of magnitude slower than 
\fuses. %, moreover \dars' relaxed cost does not decrease monotonically. % (\fuses settles around 30ms, while \dars requires more than 1s).

\Fig{fig-convergence-example}(c) shows all the compared techniques, while \Fig{fig-convergence-example}(d) provides a zoomed-in 
view restricted to the first \zoomTime. We only report the final cost for \dars, whose convergence is much slower than all the other methods. 
From \Fig{fig-convergence-example}(c)-(d) we note that \aexp, \LBP, an \TRW perform well in segmentation problems \change{. 
While not providing any optimality guarantee (\LBP and \TRW may not even converge to a local optimum), these techniques return near-optimal 
solutions in all the tested images.}{ and indeed return near-optimal 
solutions in all the tested images.} 
\aexp and \LBP have longer convergence tails but typically obtain a smaller value than \fuses and \dars.
 \TRW also requires more time to terminate but attains a near-optimal objective in few iterations. 
 \fuses is farther from optimal (see also Tables~\ref{table-statistics-512-1K}-\ref{table-statistics-512-2K}), but it is the only technique that does not require any initial guess. 
 % We remark that 
 \fuses attains an objective comparable to the one of \dars, 
 while being much faster. % than \dars. % In particular, \fuses' CPU time is comparable (if not smaller) than all evaluated techniques.

 %%%%%%%%%%%%%%%%%%%%%%%%%%%%%%%%%%%%%%%%%%%%%%%%
\change{\input{table-statistics2}}{%!TEX root = main.tex
\begin{table}[h]
	\vspace{-0.2cm}
	\centering
	\resizebox{\columnwidth}{!}{%
	\begin{tabular}{|c|c|c|c|c|c|}
		\hline
		\multirow{3}{*}{Method} & \multicolumn{3}{c|}{Suboptimality} & \multirow{2}{*}{Accuracy}  & \multirow{2}{*}{Runtime} \\
		\cline{2-4} & Optimal & Relax &  Round &   &   \\
		&  Labels (\%) &  Gap (\%) & Gap (\%) & (\% IoU) &  (ms)  \\
		\hline
		\fuses & $99.17$ & $2.584$ & $0.047$ & $49.46$ & $16.15$ \\
		\hline
		\dars & $99.68$  & $0.210$ & $0.010$  & $49.43$ & $683.77$ \\
		\hline
		\aexp & $99.93$ & - & $8.08\text{e-}4$ & $49.38$ & $65.78$ \\
		\hline
		\LBP & $99.99$ & - & $1.74\text{e-}4$ & $49.38$ & $106.48$ \\
		\hline
		\TRW & $100.00$ & - & $-8.31\text{e-}6$ & $49.36$ & $41.64$ \\
		\hline
		\emph{Bonnet} (SP)  & - & - & -& $48.08$ & - \\
		%\hline
		%\emph{bonnet} (P)  & - & - & -& $52.65$ & - \\
		\hline
	\end{tabular}
	}
	\caption{Performance on the Cityscapes dataset (1000 superpixels).\label{table-statistics-512-1K}}
\isExtended{}{\vspace{-2mm}}
\end{table}

\begin{table}[h]
	\vspace{-0.4cm}
	\centering
	\resizebox{\columnwidth}{!}{%
	\begin{tabular}{|c|c|c|c|c|c|}
		\hline
		\multirow{3}{*}{Method} & \multicolumn{3}{c|}{Suboptimality} & \multirow{2}{*}{Accuracy}  & \multirow{2}{*}{Runtime} \\
		\cline{2-4} & Optimal & Relax &  Round &   &   \\
		&  Labels (\%) &  Gap (\%) & Gap (\%) & (\% IoU) &  (ms)  \\
		\hline
		\fuses & $99.17$  & $2.331$ & $0.050$  & $51.37$ & $40.20$ \\
		\hline
		\dars & $99.68$  & $0.163$ & $0.011$  & $51.27$ & $1700.45$ \\
		\hline
		\aexp & $99.93$  & - & $1.02\text{e-}3$  & $51.22$ & $145.48$ \\
		\hline
		\LBP & $99.99$  & - & $1.57\text{e-}4$  & $51.23$ & $250.58$ \\
		\hline
		\TRW & $100.00$  & - & $-3.31\text{e-}6$  & $51.21$ & $99.19$ \\
		\hline
		\emph{Bonnet} (SP)  & - & - & -& $50.28$ & - \\
		%\hline
		%\emph{bonnet} (P)  & - & - & -& $52.65$ & - \\
		\hline
	\end{tabular}
	}
	\caption{Performance on the Cityscapes dataset (2000 superpixels).\label{table-statistics-512-2K}}
	\vspace{-7mm}
\end{table}
}
%%%%%%%%%%%%%%%%%%%%%%%%%%%%%%%%%%%%%%%%%%%%%%%%

Tables~\ref{table-statistics-512-1K} and~\ref{table-statistics-512-2K}
%\Table{table-statistics-512-1K} provides 
report statistics describing the performance of the compared techniques on the \change{Cityscapes' Lindau dataset over \change{59}{500} images 
(we use approximately 1000 superpixels). }{\emph{Cityscapes dataset, when using 1000 and 2000 superpixels, respectively}.} 
\omitted{
we chose this dataset for testing, due to its similarity in the color variance $\beta$ to the Munster dataset we used for training.}% (This dataset is chosen because of similarity in color variance to our training dataset). 
We show the percentage of optimal labels (``Optimal Labels'' column), the relaxation gap (``Relax Gap'' column), and the rounding gap (``Round Gap'' column). 
%while \fuses is two orders of magnitude faster than \dars.  
%We note that 
The tables show that \fuses and \dars have comparable suboptimality (typically larger than the other compared techniques). 
\fuses and \dars produce optimal assignments for most of the nodes in the MRF, 
and attain a rounded cost within \percSubopt of the optimum.
 The IoU (``Accuracy'' column) shows that all the techniques have comparable accuracy (around \change{$34\%$}{$49.4\%$ for 1000 superpixels}). 
 % \change{}{The standard deviations are within $0.1\%$ and therefore not shown here.} 
 % hence the MRF 
 % and improve  
 \change{All the compared techniques produce more accurate results than the %improve the accuracy of the
 	% the average accuracy of the 
 CNN-based segmentation produced by Bonnet, which has IoU equal to $30.11\%$ on this dataset. 
 %, which we use as unary potential
}{All the compared techniques outperform the Bonnet solution restricted to the superpixels ($48.08\%$ in Table~\ref{table-statistics-512-1K}) by a small margin of $1.4\%$. However, their accuracy is inferior to the original Bonnet solution ($52.65\%$, not restricted to superpixels). This is due to the fact that the MRF solution heavily depends on the quality of the superpixels and on the model used for the binary terms. While improving these aspects is outside the scope of this work (we focus on solving the MRF, rather than building it from data), in \isExtended{Section~\ref{sec:enhanced}}{the supplemental material}  we provide extra results to show that having more accurate superpixels and binary terms can boost the IoU above $70\%$.  
} 

% produce less accurate results than bonnet mostly due to superpixel segmentation. To show that the MRF works properly. We computed bonnet superpixel accuracy from largest unary potential for each superpixel and Table~\ref{table-statistics-512-1K} shows an increase of $1.4\%$. }
 % Note that the accuracy depends on the parameters of the MRF 
 % ($\penaltyTerm_{i}$ and $\penaltyTerm_{ij}$) besides depending on the solver. \change{}{For readers intersted in improving pixel-wise accuracy, we provided additional experiments in section D to address possible improvements.}
\change{}{\fuses is the fastest MRF solver (more than 2x faster than \TRW) and can compute a solution in milliseconds, while 
not relying on any initial guess. 
%In particular, \fuses is more than 2x faster than \TRW.
\Table{table-statistics-512-2K} also shows that 
%\change{even with 2000 superpixels, the advantages of \fuses remain.}{
\fuses scales better than other techniques.}
\Fig{fig-segmentation-snapshots} shows qualitative segmentation results obtained using the proposed techniques. 
We also attempted to use a general-purpose SDP solver,~\cvx~\cite{CVXwebsite}, for our evaluation: with only 200 superpixels, \cvx requires more than 50 minutes 
to solve~\eqref{eq:SDPstandard}, while for 1000 superpixels it crashes due to excessive memory usage.
% \cvx is unable to solve the same SDP as \dars due to memory limitations. With only 200 superpixels, it requires 50 minutes computation time.
% In particular, the average accuracy of the CNN-based segmentation produced by Bonnet is $30.11\%$ in our tests, which is lower than the one attained by the MRF solvers in 
% \Table{table-statistics-lindau-1000} and \Table{table-statistics-lindau-2000}.

\Fig{fig:relaxation-gap}(a) shows the relaxation gap for \fuses and \dars for increasing number of nodes; we control the number of nodes 
by controlling how many superpixels each image is divided into. The relaxation gap decreases for increasing number of nodes, which 
is a desirable feature since one typically solves large problems (>1000 nodes). The relaxation gap in \fuses is slightly larger: in hindsight, 
we traded-off suboptimality for fast computation. \Fig{fig:relaxation-gap}(b) shows the relaxation gap for \fuses and \dars for
 increasing number of labels; we artificially reduce the number of labels in Cityscapes for this test. 
 The quality of both relaxations does not degrades significantly for increasing number of labels.
 \omitted{Both relaxations exhibit 
 a smaller relaxation gap when there are very few labels and then settle to a constant value for increasing number of labels.} 

%%%%%%%%%%%%%%%%%%%%%%%%%%%%%%%%%%%%%%%%%%%%%%%%
%!TEX root = main.tex

% \todo{compare fuses and fuses2DA in terms of time and relaxation gap for increasing number of classes and number of nodes}
\begin{figure}[h!]
\vspace{-4mm}

% \begin{centering}
%\includegraphics[width=0.7\columnwidth,trim=0mm 14mm 0mm 0mm,clip]{scalability}
\includegraphics[width=\columnwidth]{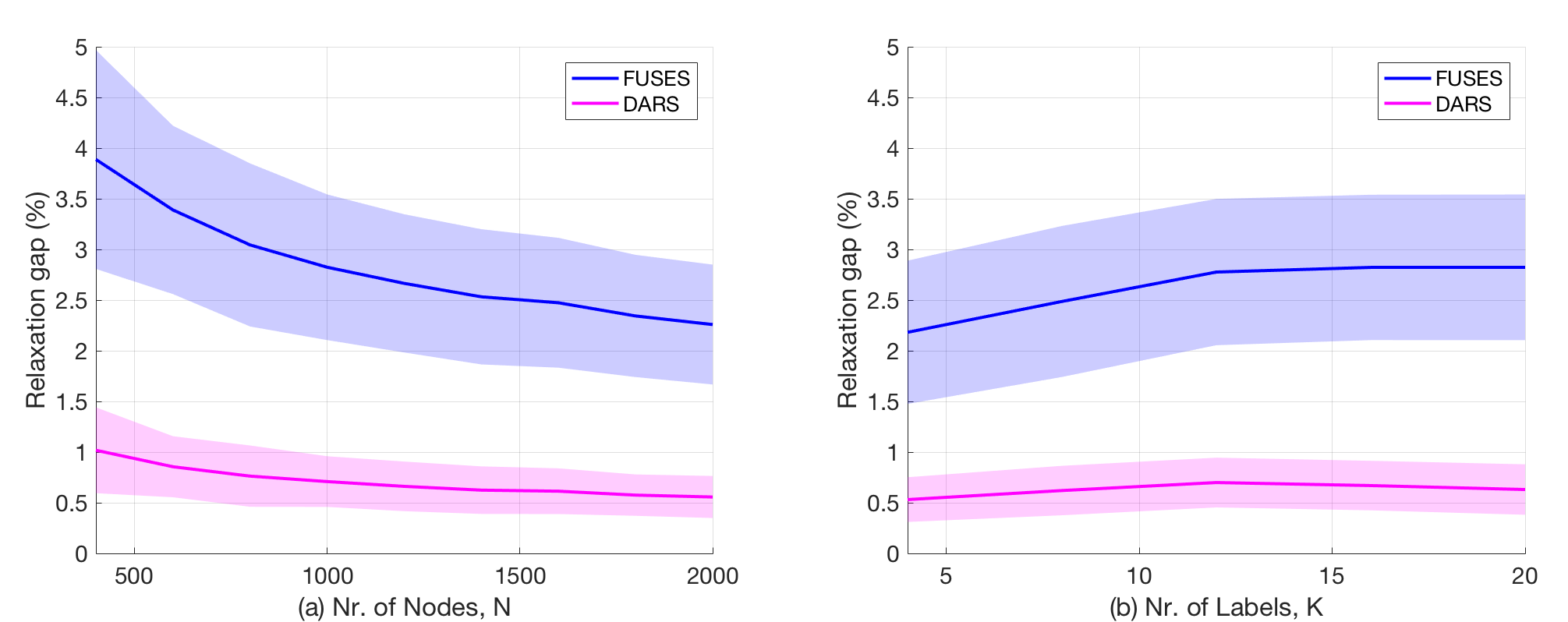}
% \begin{minipage}{\textwidth}
% \begin{tabular}{cc}
% %\hspace{2cm}(a) Nr. of Nodes, $\nrNodes$ & \hspace{2.4cm} (b) Nr. of Labels, $\nrClasses$
% \end{tabular}
% \end{minipage}
\vspace{-6mm}
\caption{Relaxation gap for \fuses and \dars for (a) increasing number of nodes and (b) increasing number of labels. 
The shaded area describes the 1-sigma standard deviation.
\label{fig:relaxation-gap}}
% \end{centering}
% \vspace{-7mm}
\end{figure}
% \todo{This is using lindau dataset}
%%%%%%%%%%%%%%%%%%%%%%%%%%%%%%%%%%%%%%%%%%%%%%%%

\isExtended{
%%%%%%%%%%%%%%%%%%%%%%%%%%%%%%%%%%%%%%%%%%%%%%%%%%%%%%%%%%%%%%%%%%%%%%
\change{}{
\subsection{Semantic Segmentation: Results with Enhanced MRFs}\label{sec:enhanced}
The main focus of this work is to design a fast MRF solver. 
However, when used for semantic segmentation, the accuracy (IoU) of the resulting solution may be implicitly limited by the choice of unary and binary potentials, as well as by the choice of nodes (superpixels). 
For instance, a superpixel including pixels belonging to different classes implicitly induces errors in the segmentation and the MRF cannot do anything to mitigate that issue (each superpixel can only be assigned a single label). This section presents an extra set of results 
%can be only as good as the 
to test the performance of the MRF solution for increasing quality of the unary potentials, binary potentials, and superpixels. While the tests for improved unary factors are realistic (we use a slower but more accurate model from Bonnet), the tests with improved binary potentials and superpixels rely on an idealized (and unrealistic) setup; in particular we create improved binary potentials and superpixels using the knowledge of the ground truth labels. While this setup is not implementable in practice, we believe it provides interesting insights on the potential performance of MRF-based segmentation (and \fuses) when the construction of the MRF is less naive than the one in Section~\ref{sec:exp-results}. For these extra results, we omit \dars, which is currently slow compared to the other techniques.   
% However, for readers interested in fast semantic segmentation, this section provides directions for improvements. We drop \dars here, because \fuses is able to achieve similar accuracy but is much faster.

%%%%%%%%%%%%%%%%%%%%%%%%%%%%%%%%%%%%%%%%%%%%%%%%
%!TEX root = main.tex
\begin{table}[h]
	\vspace{-0.2cm}
	\centering
	\resizebox{\columnwidth}{!}{%
	\begin{tabular}{|c|c|c|c|c|c|}
		\hline
		\multirow{3}{*}{Method} & \multicolumn{3}{c|}{Suboptimality} & \multirow{2}{*}{Accuracy}  & \multirow{2}{*}{Runtime} \\
		\cline{2-4} & Optimal & Relax &  Round &   &   \\
		&  Labels (\%) &  Gap (\%) & Gap (\%) & (\% IoU) &  (ms)  \\
		\hline
		\fuses & $99.50$ & $2.31$ & $0.025$ & $60.96$ & $16.85$ \\
		\hline
		\aexp & $99.97$ & - & $6.22\text{e-}4$ & $60.98$ & $51.90$ \\
		\hline
		\LBP & $99.99$ & - & $1.69\text{e-}4$ & $60.99$ & $60.59$ \\
		\hline
		\TRW & $100.00$ & - & $8.66\text{e-}6$ & $60.99$ & $19.94$ \\
		\hline
		\emph{Bonnet} (SP)  & - & - & -& $60.85$ & - \\
		\hline
		\emph{Bonnet} (P)  & - & - & -& $61.23$ & - \\
		\hline
	\end{tabular}
	}
	\caption{Performance on the Cityscapes dataset \hspace{5cm}(1000 superpixels, \emph{Bonnet} \modelTwo model).\label{table-statistics-1024-1K}}
	% \vspace{-5mm}
\end{table}

\begin{table}[h]
	\vspace{-0.4cm}
	\centering
	\resizebox{\columnwidth}{!}{%
	\begin{tabular}{|c|c|c|c|c|c|}
		\hline
		\multirow{3}{*}{Method} & \multicolumn{3}{c|}{Suboptimality} & \multirow{2}{*}{Accuracy}  & \multirow{2}{*}{Runtime} \\
		\cline{2-4} & Optimal & Relax &  Round &   &   \\
		&  Labels (\%) &  Gap (\%) & Gap (\%) & (\% IoU) &  (ms)  \\
		\hline
		\fuses & $99.43$ & $2.036$ & $0.034$ & $61.36$ & $36.43$ \\
		\hline
		\aexp & $99.96$ & - & $5.80\text{e-}4$ & $61.31$ & $112.01$ \\
		\hline
		\LBP & $99.99$ & - & $1.20\text{e-}4$ & $61.32$ & $166.62$ \\
		\hline
		\TRW & $100.00$ & - & $4.64\text{e-}6$ & $61.32$ & $57.52$ \\
		\hline
		\emph{Bonnet} (SP)  & - & - & -& $61.16$ & - \\
		\hline
		\emph{Bonnet} (P)  & - & - & -& $61.23$ & - \\
		\hline
	\end{tabular}
	}
	\caption{Performance on the Cityscapes dataset \hspace{5cm}(2000 superpixels, \emph{Bonnet} \modelTwo model).\label{table-statistics-1024-2K}}
	\vspace{-5mm}
\end{table}

%%%%%%%%%%%%%%%%%%%%%%%%%%%%%%%%%%%%%%%%%%%%%%%%
\myparagraph{Improved Unary Potentials} One way to improve the accuracy of the MRF segmentation is to use a better CNN model for the unary potentials. In Section~\ref{sec:exp-results}, we ran experiments with the \emph{Bonnet} \modelOne model which strikes a good balance between accuracy and computational cost. 
In this section, we report extra results using a more accurate model, \emph{Bonnet} \modelTwo, which requires four times more operations but achieves $61.23\%$ accuracy (pixel-wise).   
% We also tested our solver on \emph{Bonnet} 1024x512 model which requires four times more operations but achieves $61.23\%$ accuracy after recovering the original size. 
We used the Bonnet results to create improved unary potentials in the MRF and repeated the tests in Section~\ref{sec:exp-results}, keeping the parameters as in the previous tests.
 The results are reported in Tables~\ref{table-statistics-1024-1K} and~\ref{table-statistics-1024-2K}. The tables report both the pixel-wise accuracy of Bonnet (``Bonnet (P)'') and the accuracy of the induced superpixel segmentation (``Bonnet (SP)'').
%keeping the parameters as in the previous  kept the same parameters from the previous setup.
Table~\ref{table-statistics-1024-1K} confirms that \fuses is fast but slightly less optimal than the other  techniques. Table~\ref{table-statistics-1024-2K} shows that with 2000 superpixels the MRF solution outperforms (by a small margin) the pixel-wise accuracy of Bonnet ($61.23\%$). Moreover, the accuracy of all techniques is approximately $10\%$ higher than the results in Tables~\ref{table-statistics-512-1K} and~\ref{table-statistics-512-2K}.

% we are able to not only outperform the superpixel version of Bonnet (accuracy $61.23$)
% In addition, without changing the parameters, we still improve the accuracy slightly. Table~\ref{table-statistics-1024-2K} shows that having more superpixels alone improve accuracy and we can further improve accuracy to $61.36\%$, higher than \emph{bonnet}'s accuracy.

%%%%%%%%%%%%%%%%%%%%%%%%%%%%%%%%%%%%%%%%%%%%%%%%
%!TEX root = main.tex
\begin{table}[h]
	\vspace{-0.2cm}
	\centering
	\resizebox{\columnwidth}{!}{%
	\begin{tabular}{|c|c|c|c|c|c|}
		\hline
		\multirow{3}{*}{Method} & \multicolumn{3}{c|}{Suboptimality} & \multirow{2}{*}{Accuracy}  & \multirow{2}{*}{Runtime} \\
		\cline{2-4} & Optimal & Relax &  Round &   &   \\
		&  Labels (\%) &  Gap (\%) & Gap (\%) & (\% IoU) &  (ms)  \\
		\hline
		\fuses & $99.06$ & $0.947$ & $0.131$ & $62.67$ & $23.97$ \\
		\hline
		\aexp & $99.97$ & - & $7.39\text{e-}3$ & $63.33$ & $46.41$ \\
		\hline
		\LBP & $99.89$ & - & $3.36\text{e-}3$ & $63.26$ & $114.87$ \\
		\hline
		\TRW & $100.00$ & - & $7.51\text{e-}6$ & $63.32$ & $38.07$ \\
		\hline
		\emph{Bonnet} (SP)  & - & - & -& $60.85$ & - \\
		\hline
		\emph{Bonnet} (P)  & - & - & -& $61.23$ & - \\
		\hline
	\end{tabular}
	}
	\caption{Performance on the Cityscapes dataset \hspace{5cm}(1000 superpixels, \emph{Bonnet} \modelTwo model, improved binary).\label{table-statistics-1024-1K-Binary}}
	\vspace{-5mm}
\end{table}
\begin{table}[h]
	\vspace{-0.2cm}
	\centering
	\resizebox{\columnwidth}{!}{%
	\begin{tabular}{|c|c|c|c|c|c|}
		\hline
		\multirow{3}{*}{Method} & \multicolumn{3}{c|}{Suboptimality} & \multirow{2}{*}{Accuracy}  & \multirow{2}{*}{Runtime} \\
		\cline{2-4} & Optimal & Relax &  Round &   &   \\
		&  Labels (\%) &  Gap (\%) & Gap (\%) & (\% IoU) &  (ms)  \\
		\hline
		\fuses & $98.93$ & $0.944$ & $0.156$ & $63.28$ & $45.87$ \\
		\hline
		\aexp & $99.97$ & - & $5.45\text{e-}4$ & $64.13$ & $91.94$ \\
		\hline
		\LBP & $99.88$ & - & $4.16\text{e-}3$ & $64.02$ & $296.94$ \\
		\hline
		\TRW & $100.00$ & - & $0.00$ & $64.09$ & $98.62$ \\
		\hline
		\emph{Bonnet} (SP)  & - & - & -& $61.16$ & - \\
		\hline
		\emph{Bonnet} (P)  & - & - & -& $61.23$ & - \\
		\hline
	\end{tabular}
	}
	\caption{Performance on the Cityscapes dataset \hspace{5cm}(2000 superpixels, \emph{Bonnet} \modelTwo model, improved binary).\label{table-statistics-1024-2K-Binary}}
	\vspace{-5mm}
\end{table}

%%%%%%%%%%%%%%%%%%%%%%%%%%%%%%%%%%%%%%%%%%%%%%%%
\myparagraph{Improved Binary Potentials} This section evaluates the performance of the MRF-based semantic segmentation techniques when computing the binary potentials from the ground truth labels. This is not a realistic setup and is used to obtain an ``upper bound'' on the performance of the MRF solvers when the binary potentials are accurate. 
% in practical application one 
% The other direction is to utilize generate better binary potentials. Our current setup is based solely on average color. However, there is large potential for improvement is to integrate additional information commonly available in robotics such as geometric information in $\penaltyTerm_{ij}$. We are not able to get this information with \emph{Cityscapes dataset}. 
Therefore, for the tests in this section, we evaluated all compared techniques using % the same unary potentials and 
synthetic binary potentials obtained as follows: for a pair of superpixels $i$ and $j$, we set $\penaltyTerm_{ij}=0.2$ for nearby superpixels with the same ground truth labels and $\penaltyTerm_{ij}=0$ otherwise. Tables~\ref{table-statistics-1024-1K-Binary} and~\ref{table-statistics-1024-2K-Binary} 
% show that \fuses is still competitive with this setup and all compared techinques achieve better accuracy.
confirm our findings so far (\fuses is the fastest technique but it is slightly less accurate than competitors) but also stress that an improved choice of binary potentials can further boost accuracy. In particular, Table~\ref{table-statistics-1024-2K-Binary}  shows that \fuses is approximately $2\%$ better than Bonnet with this improved binary potentials. 
We remark that while in the Cityscapes dataset we can only use pixel information to create the binary potentials, in 
several robotics applications, one can leverage other sources of information, e.g., the geometry of the scene estimated from SLAM, to obtain better binary potentials.

%%%%%%%%%%%%%%%%%%%%%%%%%%%%%%%%%%%%%%%%%%%%%%%%
%!TEX root = main.tex
\begin{table}[h]
	\vspace{-0.2cm}
	\centering
	\resizebox{\columnwidth}{!}{%
	\begin{tabular}{|c|c|c|c|c|c|}
		\hline
		\multirow{3}{*}{Method} & \multicolumn{3}{c|}{Suboptimality} & \multirow{2}{*}{Accuracy}  & \multirow{2}{*}{Runtime} \\
		\cline{2-4} & Optimal & Relax &  Round &   &   \\
		&  Labels (\%) &  Gap (\%) & Gap (\%) & (\% IoU) &  (ms)  \\
		\hline
		\fuses & $99.54$  & $0.899$ & $0.059$  & $72.23$ & $14.89$ \\
		\hline
		\aexp & $99.98$  & - & $4.20\text{e-}4$  & $72.73$ & $34.56$ \\
		\hline
		\LBP & $99.94$  & - & $1.30\text{e-}3$  & $72.70$ & $64.40$ \\
		\hline
		\TRW & $100.00$  & - & $-7.18\text{e-}6$  & $72.71$ & $27.10$ \\
		\hline
		\emph{Bonnet} (SP)  & - & - & -& $69.53$ & - \\
		\hline
		\emph{Bonnet} (P)  & - & - & -& $61.23$ & - \\
		\hline
	\end{tabular}
	}
	\caption{Performance on the Cityscapes dataset (1000 superpixels, \emph{Bonnet} \modelTwo model, improved binary and superpixel).\label{table-statistics-1024-1K-SP}}
\end{table}

\begin{table}[h]
	\vspace{-0.4cm}
	\centering
	\resizebox{\columnwidth}{!}{%
	\begin{tabular}{|c|c|c|c|c|c|}
		\hline
		\multirow{3}{*}{Method} & \multicolumn{3}{c|}{Suboptimality} & \multirow{2}{*}{Accuracy}  & \multirow{2}{*}{Runtime} \\
		\cline{2-4} & Optimal & Relax &  Round &   &   \\
		&  Labels (\%) &  Gap (\%) & Gap (\%) & (\% IoU) &  (ms)  \\
		\hline
   		\fuses & $99.47$  & $0.898$ & $0.072$  & $72.50$ & $34.91$ \\
		\hline
		\aexp & $99.98$  & - & $3.48\text{e-}4$  & $73.34$ & $74.36$ \\
		\hline
		\LBP & $99.94$  & - & $1.99\text{e-}3$  & $73.28$ & $171.95$ \\
		\hline
		\TRW & $100.00$  & - & $-5.80\text{e-}7$  & $73.33$ & $74.11$ \\
		\hline
		\emph{Bonnet} (SP)  & - & - & -& $68.89$ & - \\
		\hline
		\emph{Bonnet} (P)  & - & - & -& $61.23$ & - \\
		\hline
	\end{tabular}
	}
	\caption{Performance on the Cityscapes dataset (2000 superpixels, \emph{Bonnet} \modelTwo model, improved binary and superpixel).\label{table-statistics-1024-2K-SP}}
	\vspace{-5mm}
\end{table}
%%%%%%%%%%%%%%%%%%%%%%%%%%%%%%%%%%%%%%%%%%%%%%%%
\myparagraph{Improved Superpixels} We can boost performance even further by considering an enhanced  superpixel segmentation. For these tests, we consider a ``ground truth'' superpixel segmentation where each superpixel is constrained to include pixels from a single class; for this purpose we obtain the superpixel segmentation from the ground truth labeling of the Cityscapes images. 
With the % combination of improved unary and binary potentials, as well as the 
improved superpixels, we attain an accuracy higher than $72\%$ for all the MRF-based techniques, with an improvement of $11\%$ with respect to Bonnet ($61.23\%$). 
While this setup leverages ground truth labels, hence it is not implementable in practice, one can envision to inform the superpixel creation with geometric information (e.g., 3D reconstruction from SLAM), so to only cluster together pixels  picturing nearby points in 3D. %This would allow obtaining a more accurate segmentation, as 
}
}{}

\isExtended{%!TEX root = main.tex

\section{Related Work}
\label{sec:relatedWork}

This section reviews inference techniques (\prettyref{sec:rw-techniques-exact}-\ref{sec:rw-techniques-approx}) and applications
 (\prettyref{sec:rw-applications}) for pairwise MRFs including
 work on semantic segmentation. % (\prettyref{sec:rw-ss}).
Our presentation is based on~\cite{Szeliski08pami-surveyMRF,Blake11book-MRF,Kappes15ijcv-energyMin} but also covers more recent work on MRFs and semantic segmentation. 
% In particular, we refer the reader to the comprehensive survey of Kappes\setal~\cite{Kappes15ijcv-energyMin}, which also 
% reports implementation details of state-of-the-art techniques.

%%%%%%%%%%%%%%%%%%%%%%%%%%%%%%%%%%%%%%%%%%%%%%%%%%%%%%%%%%%%%%%%%%%%%%%%%%%%%
\subsection{Exact Inference in MRFs}
\label{sec:rw-techniques-exact}

% We organize our review of inference techniques for .

%%%%%%%%%%%%%%%%%%%%%%%% ========================= %%%%%%%%%%%%%%%%%%%%%%%%% 
\myparagraph{Efficient Algorithms}
Inference in MRFs is intractable in general. However, particular instances of the problem 
are solvable in polynomial time. 
In particular, the Ising model %(binary pairwise MRF with attractive potentials) 
can be solved exactly in polynomial time via graph cut~\cite{Ivanescu65or-graphCut,Greig89jrss-graphCut}. Note that graph cut algorithms are exact when binary potentials are ``attractive'', i.e., $\penaltyTerm_{ij} \geq 0$ in~\eqref{eq:binary} (priors encourage nearby nodes to have the same label). MRFs with \emph{repulsive} potentials ($\penaltyTerm_{ij} < 0$) are intractable in general~\cite{Kolmogorov04pami-graphCut}. 
A more general (necessary and sufficient) condition that ensures optimality of graph cut for binary pairwise MRFs with classes $\classSet=\{0;1\}$ is 
the \emph{regularity} condition: 
\beq
\label{eq:regularity}
E_{ij}(0,0) + E_{ij}(1,1) \leq E_{ij}(0,1) + E_{ij}(1,0)
\eeq
for any $(i,j)\in \calB$, see Lemma 3.2 and Theorem 3,1 in~\cite{Kolmogorov04pami-graphCut}. 
The regularity condition in eq.~\eqref{eq:regularity} is a special case of \emph{submodularity}, and indeed the 
corresponding potentials are also called \emph{submodular}~\cite{Kolmogorov04pami-graphCut,Jegelka11cvpr,Felzenszwalb11pami}. 

For multi-label pair-wise MRFs, exact solutions exist for the case when the binary potentials are convex functions of the labels~\cite{Ishikawa03pami-MRF,Ishikawa98eccv-MRF,Boykov98cvpr-MRF} and for the case where the binary potentials are linear and the unary potentials are convex~\cite{Hochbaum01acm-MRF}. 
We remark that these approaches assume a linear ordering of the labels, where the 
potentials penalize node labels depending on their label distance $|\class_i - \class_j|$; 
this means that choosing $\class_i=1$ and $\class_j=3$ incurs a larger penalty than choosing 
$\class_i=1$ and $\class_j=2$; on the other hand, the Potts model in eq.~\eqref{eq:binary}
 penalizes in the same way any class mismatch $\class_i \neq \class_j$.
 Assuming a linear ordering is often unrealistic in practice; for instance, in semantic segmentation the classes (e.g., cat, table, car) do not admit a linear order in general. 
 Moreover, convexity is a strong assumption for several MRF applications, such as depth reconstruction, where nonconvex costs have the desirable property of being \emph{discontinuity-preserving}~\cite{Kolmogorov04pami-graphCut} contrarily to convex ones, which tend to smooth out depth discontinuities. 
 Inference in multi-class MRF based on the Potts model is NP-hard, see~\cite{Boykov01pami-graphCut}. 
 % the intractability result in~\cite{Boykov01pami-graphCut}. 

In the special case where the topology of the MRF is a \emph{chain} 
(e.g., when the MRF describes a 1D signal or sequence), or more generally a \emph{tree},  
\emph{Dynamic Programming} provides an optimal MAP estimate in polynomial time, see~\cite{Felzenszwalb11pami,Veksler05cvpr}. 
%, which also provides an overview of the use of DP in vision. 
Related work~\cite{Amit96pami,Felzenszwalb05pami} also extends dynamic programming to certain families of graphs with cycles and small cliques.
% The basic idea of dynamic programming is to decompose a problem into a set of subproblems such that 1) given a solution to the subproblems, we can quickly compute a solution to the original problem and 2) the subproblems can be efficiently solved recursively in terms of each other.

%%%%%%%%%%%%%%%%%%%%%%%%% ========================= %%%%%%%%%%%%%%%%%%%%%%%%% 
\myparagraph{Global Integer Solvers} 
 The energy minimization problem~\eqref{eq:MRF1} is a quadratic integer program and can be 
 easily reformulated as a binary optimization problem~\cite{Boros02dam-booleanOpt,Schrijver86book-integerProgramming,Boros06tr-binaryOpt}. 
 Integer programming is NP-hard in general, but one may still resort to 
 state-of-the-art integer solvers (e.g., \CPLEX~\cite{CPLEXwebsite}) for moderate-size instances.
 For quadratic and linear programs, integer solvers based on cutting plane methods or branch \& bound are able to produce 
 solutions for problems with few hundred variables relatively quickly (i.e., in few seconds), but 
 become unacceptably slow for larger problems.  
 A Branch-and-Cut approach is proposed in~\cite{Wang15ijcv}.
 An evaluation and a broader review of integer programming for MRFs is given in~\cite{Kappes15ijcv-energyMin}.
  % Typically, integer solvers are unable to solve MRFs with more than a few hundreds of nodes 
  % in an acceptable time.
% QPBOP *Boros et al. ’06, Rother et al. ‘07+

%%%%%%%%%%%%%%%%%%%%%%%%%%%%%%%%%%%%%%%%%%%%%%%%%%%%%%%%%%%%%%%%%%%%%%%%%%%%%
\subsection{Approximate and Local Inference in MRFs}
\label{sec:rw-techniques-approx}

\myparagraph{Iterative Local Solvers and Meta-heuristics} 
Local solvers start at a given initial guess and iteratively try to converge to 
a local optimum of the cost function. Early work includes 
the \emph{Iterative Local Modes} (ICM) of Besag~\cite{Besag86jrss-MRF}, which 
% starts with a given guess and 
at each iteration greedily changes the label of a node in order to 
get the largest decrease in the cost. ICM is known to be very sensitive to the quality of the 
initial guess~\cite{Szeliski08pami-surveyMRF}. 
In order to improve convergence, Geman and Geman~\cite{Geman84pami-MRF} use Simulated Annealing to perform inference in MRFs.
% compute a MAP estimate in MRFs.
% Simulated Annealing: accept move even if
% energy increases (wi{}th certain probability)
Simulated Annealing requires exponential time to converge in theory 
and is notoriously slow in practice~\cite{Birchfield98pami-MRF}.
%%%%%%%%%%%%%%%%%%%%%%%%% ========================= %%%%%%%%%%%%%%%%%%%%%%%%% 
% \myparagraph{(Meta)Heuristic Methods} 
% ICM: Very local moves get stuck in local
% minima [Boykov , Veksler and Zabih 2001]

%%%%%%%%%%%%%%%%%%%%%%%%% ========================= %%%%%%%%%%%%%%%%%%%%%%%%% 
\myparagraph{Graph Cuts and Move-Making Algorithms}
While graph cut methods are able to compute globally optimal solutions in binary pairwise MRFs with submodular potentials (Section~\ref{sec:rw-techniques-exact}), 
they are only able to converge to local minima in non-submodular binary MRFs or in multi-class MRFs. 
For the binary case, related works~\cite{Kolmogorov05pami,Jegelka11cvpr} develop schemes to approximately solve MRFs with non-submodular potentials. 
Regarding the multi-class case, popular graph cut methods include the \emph{swap-move} ($\alpha$-$\beta$-swap) and the \emph{expansion-move} ($\alpha$-expansion) algorithms, both proposed in~\cite{Boykov01pami-graphCut}.
At each inner iteration, these algorithms solve a binary segmentation problem using graph cut, while the outer loop attempts to reconcile the binary results into 
a coherent multi-class segmentation. 
Boykov\setal\cite{Boykov01pami-graphCut}  show that
the swap-move algorithm is applicable whenever
the smoothness potentials are \emph{semi-metric} (i.e., $E_{ij}(\class_i,\class_j) = E_{ij}(\class_i,\class_j) \geq 0$ and $E_{ij}(\class_i,\class_j) = 0 \iff \class_i=\class_j$),
and the expansion-move algorithm is applicable whenever the 
smoothness potentials are \emph{metric}\footnote{Note that both the Potts model and the truncated $\ell_2$ distance are metrics.} 
(i.e., they are semi-metric and also satisfy the triangle inequality
$E_{ij}(\class_i,\class_j) \leq E_{ik}(\class_i,\class_k) + E_{kj}(\class_k,\class_j)$); these conditions are further generalized in~\cite{Kolmogorov04pami-graphCut}. 
Under these conditions, Boykov\setal\cite{Boykov01pami-graphCut} show that these graph cut methods produce ``strong'' local minima, i.e., local minima where no allowed move is able to further reduce the cost.
Moreover, these techniques produce a local solution with is proven to be within a known factor from the global minimum~\cite{Boykov01pami-graphCut}.  
When these conditions are not satisfied, approximations of the cost function can be used~\cite{Rother05cvpr-MRF,Boykov01pami-graphCut}.
Komodakis and Tziritas~\cite{Komodakis05cvpr} draw connections between move-making algorithms and the dual of linear programming relaxations.
Kumar and Koller~\cite{Kumar09uai,Torr09nips} propose a move-making approach that applies to the semi-metric case and attains the same guarantees of the 
linear relaxation (see paragraph below) in the metric case. 
Faster algorithmic variants are proposed by Alahari\setal~\cite{Alahari08cvpr}.
Lempitsky\setal~\cite{Lempitsky07iccv} provide a low-complexity algorithm (\emph{LogCut}) that requires an offline learning step.
% Multiplicative factor is at least 2.
A summary of the MRF formulations that can be solved 
 exactly or within a constant factor from the global minimum via graph cut is given in~\cite{Kolmogorov04pami-graphCut}. 
 When the potentials do not satisfy the conditions for applicability of graph cut methods, approximate versions of these techniques can be 
 still applied~\cite{Rother05cvpr-MRF} but the corresponding performance bounds no longer hold. 

%%%%%%%%%%%%%%%%%%%%%%%%% ========================= %%%%%%%%%%%%%%%%%%%%%%%%% 
\myparagraph{Message-Passing Techniques} 
Message passing techniques adjust the MAP estimate at each node in the MRF via local information exchange between neighboring nodes.
A popular message passing technique is \emph{belief propagation}~\cite{Pearl88book}, which results in exact inference in graphs 
without loops, but is also applicable to generic graphs~\cite{Felzenszwalb06ijcv-beliefPropagation,Freeman99ijcv} (\emph{loopy belief propagation}, or \LBP in short).
\LBP is not guaranteed to converge in presence of cycles, but if convergence is attained \LBP returns ``strong'' local minima~\cite{Weiss01itis-beliefPropagation,Wainwright04sc-beliefPropagation}. 
% tutorial ICCV ’09, CVPR ‘10
% Message Passing Chain:
\emph{Tree-Reweighted Message Passing}~\cite{Wainwright05itis} (\TRW) is another popular message-passing algorithm which 
is also able to estimate a lower-bound on the cost that can be used to assess the quality of the solution. 
% relies on a different belief update rule, while estimating a lower-bound on the cost that can be used to assess the quality of the solution. 
Also in this case the estimate is not guaranteed to converge and may oscillate. 
Message-passing techniques do not necessarily return integer solutions, hence the resulting estimates 
need to be rounded, see~\cite[Section 4.5]{Kappes15ijcv-energyMin}.
Kr\"{a}henb\"{u}hl and Koltun~\cite{Krahenbuhl11nips} use message passing to perform inference in a mean field approximation of a fully-connected 
Conditional Random Fields (CRFs).\footnote{Conditional Random Fields (CRFs) are a special case of MRFs, where the binary terms, rather than being 
smoothness priors, are data driven.} 
% Loopy graphs: many techniques: BP, TRW, TRW-S, Diffusion:
% – Message update rules differ
% – Compute (approximate) MAP or marginals $P(xi | x_{V\backslash{i}} )$
% – Higher-order MRFs: Factor graph BP
% - Connections to LP-relaxation (TRW tries to solve MAP LP)
% \myparagraph{Dual/Problem Decomposition?} 
% – Decompose (NP-)hard problem into tractable once.
% Solve with e.g. sub-gradient technique
%%%%%%%%%%%%%%%%%%%%%%%%% ========================= %%%%%%%%%%%%%%%%%%%%%%%%% 
% \myparagraph{Submodular Optimization}

%%%%%%%%%%%%%%%%%%%%%%%%% ========================= %%%%%%%%%%%%%%%%%%%%%%%%% 
\myparagraph{Linear Programming (LP) Relaxations} 
These techniques relax the optimization to work on continuous labels rather than discrete ones. 
Early relaxation techniques include the LP relaxation of the \emph{local polytope}~\cite{Wainwright05itis}, which 
is typically applicable only to small problem instances~\cite{Kappes15ijcv-energyMin}. 
Kleinberg and Tardos~\cite{Kleinberg99focs} provide suboptimality guarantees for LP relaxations with metric potentials.
Gupta and Tardos~\cite{Gupta00acm} extend these results considering a truncated linear metric. 
Chekuri\setal~\cite{Chekuri01siam} and Werner~\cite{Werner07pami} further refine the suboptimality bounds.
Komodakis and Tziritas~\cite{Komodakis07pami} consider the case of semi-metric and non-metric potentials and derive primal-dual methods to 
efficiently solve the resulting LP relaxations.
Sontag and Jaakkola~\cite{Sontag08nips} propose a
cutting-plane algorithm for optimizing over the \emph{marginal} polytope.
Other specialized solvers to attack larger instances have also been proposed, including block-coordinate ascent~\cite{Kolmogorov06pami-MRF}, 
subgradient methods based on dual decomposition~\cite{Komodakis11pami-MRF,Kappes13cvpr-MRF,Guignard87mp}, 
Alternating Directions Dual Decomposition~\cite{Martins11icml}, and others~\cite{Savchynskyy12uai-MRF,Globerson07nips-MRF,Sontag12uai}. 
The performance of these techniques is typically sensitive to the choice of the parameters (e.g., stepsize) and can only ensure local convergence~\cite{Kappes15ijcv-energyMin}.
For binary pairwise MRFs, LP relaxation over the local polytope can be solved 
efficiently by reformulating it as a maximum flow problem, see the \emph{roof duality} (or QPBO) approach of Rother\setal~\cite{Rother07cvpr-MRF}. 
LP relaxations typically do not produce an integer solution, therefore the corresponding solutions need to be rounded. 
Moreover, they are tightly coupled with message-passing algorithms, see~\cite[Section 4.3]{Kappes15ijcv-energyMin}.  
Kumar\setal~\cite{Kumar08nips} provide a comparison between linear, quadratic, and second-order cone programming 
relaxations, showing that the linear relaxation dominates the others.
%  LP-relaxation: e.g. Cutting-plane
% Write MAP as Integer Program (IP)
% • Relax to Liner Program (LP relaxation)
% • Solve LP (polynomial time algorithms)
% • Round LP to get best IP solution (no guarantees)
%  – Relax original problem (e.g. {0,1} to [0,1])
% and solve with existing techniques (e.g. sub-gradient)
% – Can be applied any model (dep. on solver used)
% – Connections to message passing (TRW) and
% combinatorial optimization (QPBO)
 % Problem decomposition + subgradient

%%%%%%%%%%%%%%%%%%%%%%%%% ========================= %%%%%%%%%%%%%%%%%%%%%%%%% 
\myparagraph{Spectral and Semidefinite Relaxations} 
These techniques typically rephrase inference over an MRF in terms of a binary quadratic optimization problem~\cite{Olsson08cviu}, 
which can be then relaxed to a convex program (more details in Section~\ref{sec:standardSDPrelax}). 
% The key insight is that a binary variable $b\in\{0;1\}$ can be reparametrized as $\tilde{b}\in\{-1;+1\}$, and the 
% corresponding domain can be expressed in terms of the quadratic equality constraint $\tilde{b}^2 = 1$. 
% The resulting quadratic problem with quadratic equality constraints can be then relaxed using semidefinite or spectral relaxations~\cite{Olsson08cviu}.
Shi and Malik~\cite{Shi00pami} propose a spectral relaxation for image segmentation; more recently, spectral segmentation is used by Aksoy\setal~\cite{Aksoy18siggraph}.
Keuchel\setal~\cite{Keuchel03pami} introduce SDP relaxations to several computer vision applications and use 
interior-point methods and randomized hyperplane techniques to obtain integer solutions, leveraging the 
celebrated result of Goemans and Williamson~\cite{Goemans95acm}, which bounds the suboptimality of the resulting solutions.
SDP relaxations are known to provide better solutions than spectral methods~\cite{Keuchel03pami,Olsson08cviu}. 
% first paper
% based on celebrated result of Goemans and Williamson
% for some problems SDP relaxations provide a solution at least 0.87  times the optimal value.
While early approaches also recognized the accuracy of SDP relaxations with respect to commonly used alternatives (e.g.,~\cite{Kumar08nips}), 
the computational cost of general-purpose SDP solvers prevented widespread use of this technique beyond problems with few hundred  
variables~\cite{Keuchel03pami}. 
% In a subsequent work, 
Keuchel\setal~\cite{Keuchel04dagm} propose an approach to reduce the dimension of the problem via image preprocessing and superpixel segmentation.
Concurrently, Torr~\cite{Torr03aistats} proposes the use of SDP relaxations for pixel matching problems.
Schellewald and C. Schn{\"o}rr~\cite{Schellewald05cvpr} suggest a similar SPD relaxation for subgraph matching in the context of object 
recognition.
Heiler\setal~\cite{Heiler05pr} propose to add constraints in the SDP relaxation to enforce priors (e.g., constrain the number of pixels in a class, or force set of  pixels to belong to the same class).
Olsson\setal~\cite{Olsson08cviu} develop a spectral subgradient method which is shown to reduce the relaxation gap of spectral relaxations.
%provides a more efficient solution 
Huang\setal~\cite{Huang14icml} use an Alternating Direction Methods of Multipliers to speed up computation, while
 Wang\setal~\cite{Wang13cvpr,Wang16pami} develop a specialized dual solver. 
 Frostig\setal~\cite{Frostig14nips} resort to non-convex optimization to approximate the SDP solution, while
Wang\setal~\cite{Wang15cvpr} consider fully-connected CRFs and propose fast solvers for the case where the pairwise potentials admit a low-rank decomposition.
% a low-rank approximation to speed up semidefinite relaxations 
We remark that the approach to derive the SDP relaxation is common to all papers above and follows the line of Section~\ref{sec:standardSDPrelax}.
% is essentially the one outlined 
% in Section~\ref{sec:standardSDPrelax}, and relies on expressing binary constraints (i.e.,  $b\in\{-1;+1\}$) as quadratic 
% equality constraints ($b^2=1$).
 Wainwright and Jordan~\cite{Wainwright03nips} use semidefinite programming to approximately compute the marginal distributions in a graphical model.
More generally, semidefinite programming has been a popular way to relax combinatorial integer programming problems~\cite{Alizadeh95siam,Poljak95jgo}
and assignment problems~\cite{Zhao98co,Anstreicher01mp}. 

\subsection{Applications}
\label{sec:rw-applications}

%%%%%%%%%%%%%%%%%%%%%%%%% ========================= %%%%%%%%%%%%%%%%%%%%%%%%% 
\myparagraph{Overview}
% http://www.di.ens.fr/willow/events/cvml2010/materials/INRIA_summer_school_2010_Carsten.pdf
MRFs have been  successfully used in several application domains including computer vision, computer graphics, 
machine learning, and robotics.
Popular applications include
image denoising, inpainting, and super-resolution~\cite{Boykov98cvpr-MRF,Boykov01pami-graphCut,Greig89jrss-graphCut,Felzenszwalb06ijcv-beliefPropagation},
% compression, restoration,
% denoising and inpainting
image segmentation (reviewed below),
stereo reconstruction~\cite{Birchfield99cvpr-MRF,Boykov98cvpr-MRF,Boykov01pami-graphCut,Ishikawa98cvpr-MRF,Ishikawa98eccv-MRF,Kolmogorov01iccv-MRF,Roy99ijcv-MRF,Roy98iccv-MRF,Scharstein02ijcv},
panorama stitching and digital photomontages~\cite{Agarwala04acm},
% image registration ??, 
image/video/texture synthesis~\cite{Kwatra03tg},
multi-camera scene reconstruction~\cite{Kolmogorov02b},
voxel occupancy estimation~\cite{Snow00cvpr},
non-rigid point matching and registration~\cite{Komodakis08eccv-MRF,Olsson08cviu},
medical imaging~\cite{Boykov00mic,Kim00med}.
In stereo reconstruction, the labels are the disparities at each pixel and the binary potentials are function of the absolute color differences at nearby pixels.
Birchfield and Tomasi~\cite{Birchfield98pami-MRF} provide a comparison of graph-cut methods for stereo reconstruction,
 while  Tappen and Freeman~\cite{Tappen03iccv-MRF} compare graph cut and \LBP; 
 Kolomogorow and Rother~\cite{Kolmogorov06eccv-MRF} evaluate \TRW, \LBP, and graph cut.
% the comparison between graph cut and LBP by Tappen and Freeman~\cite{Tappen03iccv-MRF}, 
% and the evaluation of TRW, LBP, and graph cut by Kolomogorow and Rother~\cite{Kolmogorov06eccv-MRF}.
% from: \cite{Szeliski08pami-surveyMRF} : image restoration [5], texture modeling [20], image labeling [13], and stereo matching [4], [11] to applications such as interactive photo segmentation [8], [39] and the automatic placement of seams in digital photomontages [1].
% image compression and restoration, image segmentation, surface reconstruction, image registration, texture synthesis, super-resolution, stereo matching and information retrieval.
% medical imaging
% In panorama stitching the labels assign, for each pixel of the panorama.
Szeliski\setal~\cite{Szeliski08pami-surveyMRF} compare several techniques on stereo reconstruction, 
photomontage, image segmentation, and image denoising benchmarks. 
 The study concludes the the expansion move algorithm typically outperforms the swap move algorithm, 
 while ICM performs poorly in practice.
% and are both often outperformed by \LBP and \TRW in depth reconstruction problems. 
 % All algorithms in general are able to 
 In general, the best approach may depend on the application:
 for instance, the expansion move algorithm is the best performer for the photomontage benchmark,
 while expansion move  and \TRW  perform the best on the depth reconstruction benchmark.
 %, while \LBP typically is typically outperformed by those techniques.
  A broader evaluation is presented in~\cite{Kappes15ijcv-energyMin}, which also provides a C++ library, \openGM~\cite{OpenGM2website},
   that implements several inference algorithms. 
  % The study also considers higher-order MRF and 
  %  concludes that recent progress in integer programming makes Integer Liner Programs a competitive algorithmic choice 
  % for inference.
%    \todo{
% Stereo Matching:
% -- discontinuity preserving potentials *Blake \&Zisserman'83,'87+
% -- Olga Veksler PhD thesis, Daniel Cremers et al.
% -- Stereo with occlusion *Kolmogrov et al. ‘02+: higher connectivity
% }

%%%%%%%%%%%%%%%%%%%%%%%%% ========================= %%%%%%%%%%%%%%%%%%%%%%%%% 
\myparagraph{Semantic Segmentation}
Semantic segmentation methods assign a semantic label to each ``region'' in an RBG image (2D segmentation),
RBG-D image, or 3D model (3D segmentation). 
Depending on the approach, labels can be assigned to single pixels/voxels, superpixels, or keypoints~\cite{Kappes15ijcv-energyMin}; 
% note that when grouping multiple pixels in superpixels, these are constrained to belong to the same class. 
Since semantic segmentation is typically modeled as an MRF, the literature review in Sections~\ref{sec:rw-techniques-exact}-\ref{sec:rw-techniques-approx} 
already covers several work in segmentation, and indeed segmentation (together with depth reconstruction) is a typical benchmark for 
inference in MRFs, see~\cite{Szeliski08pami-surveyMRF,Blake11book-MRF,Kappes15ijcv-energyMin,Felzenszwalb11pami,Zhu16jcvir} and the references therein.
Therefore, the goal of this section is to (i) provide a brief taxonomy of semantic segmentation problems, and 
(ii) review semantic segmentation techniques that do not directly use MRFs. 
The corresponding literature is vast, and we refer the reader to the excellent survey of Zhu\setal~\cite{Zhu16jcvir} for a broader 
review of related work.

%% difference with other problems
\emph{Taxonomy.}
% First of all, it is worth noticing that the task of 
Semantic segmentation is different from \emph{clustering}, which groups pixels based on similarities 
without necessarily associating a given semantic label to each group (this is sometimes called \emph{non-semantic},
 \emph{unsupervised}, or \emph{bottom-up} segmentation~\cite{Thoma16arxiv,Zhu16jcvir}).
While semantic segmentation classifies image regions into semantic classes, \emph{instance segmentation} also attempts to discern multiple objects belonging to the same class.
%% binary vs multi-class
In full analogy with MRFs, segmentation problems can be divided in \emph{binary} segmentation problems (where only two classes, foreground and background, are segmented) and \emph{multi-class} segmentation problems, where more than two labels are allowed.
%% type of sensor data
We can further divide the literature depending on the type of input data the segmentation operates on, including 
isolated RGB images (most common setup in computer vision), 
stereo images~\cite{Boykov01pami-graphCut}, 
RGB-D images~\cite{Deng15iccv,Gupta15ijcv},
volumetric 3D data (e.g., volumetric X-ray CT images~\cite{Hu01tmi}, or 3D voxel-based models~\cite{Kundu14eccv}),
or multiple RBG images; the latter setup is typically referred to as \emph{co-segmentation}~\cite{Rother06cvpr,Kim11iccv,Zhu16jcvir} (for 
generic unordered images), or temporal (or video) segmentation~\cite{Chen11wacv} (if images are collected over time).
%% operation mode
Thoma~\cite{Thoma16arxiv} also categorizes the segmentation problems into \emph{active} (where one can influence the data collection 
mechanism, as it happens in robotics), \emph{passive} (where the input data is given), and \emph{interactive} (where 
a human user provides coarse information to the segmentation algorithms).
% One example would be a system where the user clicks on the background or marks a coarse segmentation and the algorithm finds a fine-grained segmentation.
%% approaches

\emph{Other Approaches.} 
Traditional approaches for semantic segmentation work by extracting and classifying features in the 
input data, and then enforcing consistency of the classification across pixels (e.g., using MRFs or other models).
Common features include pixel color, histogram of oriented gradients, SIFT, or textons, to mention a few~\cite{Thoma16arxiv,Zhu05ijcv}.
Shotton\setal~\cite{Shotton08cvpr,Schroff08bmvc} use textons and Random Decision Forests for semantic segmentation.
Yang\setal~\cite{Yang12pami} use Support Vector Machine (SVM) demonstrating competitive performance in the 
PASCAL segmentation challenge~\cite{Everingham09ijcv}. 
A \emph{latent SVM} model is used by Felszenzwalb\setal~\cite{Felszenzwalb10pami} to detect objects using deformable part models.
Winn and Shotton~\cite{Winn06cvpr} use a CRF-based algorithm, 
named the \emph{Layout Conditional Random Field} (LayoutCRF), to detect and segment objects from their parts;
 the approach is further generalized by Hoiem\setal~\cite{Hoiem07cvpr}; 
 Shotton\setal~\cite{Shotton06eccv} use textons within a CRF model for object segmentation.
% LayoutCRF Winn et al. ’06+
% Many other examples: ObjCut Kumar et. al. ’05; 
% Deformable Part Model Felzenszwalb et al.; CVPR ’08; 
% PoseCut Bray et al. ’06,
% LayoutCRF Winn et al. ’06
Kumar\setal~\cite{Kumar05cvpr} use MRFs to detect and segment objects in an image.
Bray\setal~\cite{Bray06eccv} concurrently segment and estimate the 3D pose of a human body from multiple views.
 Higher-order MRF formulations are also used for semantic segmentation, 
 see the work by Kohli and co-authors~\cite{Kohli09ijcv,Kohli07cvpr,Kohli09cvpr} 
 and the review~\cite{Kappes15ijcv-energyMin}.
 Approaches for interactive segmentation include \emph{intelligent scissors}~\cite{Mortensen95siggraph}, 
 active contour models~\cite{Kass87ijcv,Amini90pami} (based on dynamic programming), and graph cut methods (\emph{GrabCut}~\cite{Rother04siggraph}).
While most of the work mentioned so far operates on a discrete set of nodes of a graphical model, 
related work in multi-class segmentation also includes contributions modeling the problem over a continuous domain; 
examples of such efforts include the variational method of Lellmann\setal~\cite{Lellmann09iccv}, 
and the anisotropic diffusion method of Kim\setal~\cite{Kim11iccv}; see the chapter by Cremers\setal~\cite{Cremers11chapter} for a recent survey.
More recently, deep convolutional neural networks have become a popular solution for semantic segmentation, see the 
recent review of Garcia-Garcia\setal~\cite{GarciaGarcia17arxiv}. % and the references therein.
\change{}{State-of-the-art methods, such as \emph{DeepLab}~\cite{Chen18pami-deepLab}, 
refine the results of a deep convolutional network with a fully-connected conditional random field in order to improve the localization accuracy of object boundaries.}
% To include:
% TextonBoost; Shotton et al, ‘06
% TextonBoost for Image Understanding:
% Multi-Class Object Recognition and Segmentation by
% Jointly Modeling Texture, Layout, and Context
% \todo{
% deep learning
% https://arxiv.org/pdf/1709.07158.pdf - ICRA'18
% unsupervised segmentation
% co-segmentation
% interactive segmentation
% Kohli09ijcv
% % -----------
% Kohli07cvpr
% %  Kohli, P., Kumar, M., & Torr, P. (2007). P3 and beyond: solving energies
% % with higher order cliques. In IEEE conference on computer
% % vision and pattern recognition.
% Kohli09cvpr
% Kohli, P., Ladicky, L., & Torr, P. (2008). Robust higher order potentials
% for enforcing label consistency. In CVPR.

% %%%%%%%%%%%%%%%%%%%%%%%%% ========================= %%%%%%%%%%%%%%%%%%%%%%%%% 
% \myparagraph{Higher-order and dense MRF} 
% \todo{
% Order is the number of nodes involved in smoothness factors, 2 in examples above.
% density refers to the connectivity: figure X shows 4 connected lattice structure, but more pair-wise connections can be added.
% We can deal with dense, but not with higher-order.
% }
}{%!TEX root = main.tex

% \vspace{-4mm}
\section{Related Work}\label{sec:relatedWork}

This section provides a short review of MRF inference and semantic segmentation. 
The reader can find a more comprehensive review in the supplemental material~\cite{Hu18tr-fuses}.

%%%%%%%%%%%%%%%%%%%%%%%% ========================= %%%%%%%%%%%%%%%%%%%%%%%%% 
\myparagraph{Exact MRF inference}
Inference in MRFs is intractable in general. However, particular instances of the problem 
are solvable in polynomial time, e.g., binary MRF such as the Ising model~\cite{Greig89jrss-graphCut}.
% MRFs with \emph{repulsive} potentials ($\penaltyTerm_{ij} < 0$) are intractable in general~\cite{Kolmogorov04pami-graphCut}. 
For multi-label MRFs, exact solutions exist for the case when the binary terms are convex functions of the labels~\cite{Ishikawa03pami-MRF} 
or when the binary terms are linear and the unary terms are convex~\cite{Hochbaum01acm-MRF}. 
These assumptions are unrealistic in semantic segmentation and do not apply to the Potts model of Section~\ref{sec:preliminaries}.
When polynomial-time methods are not applicable,~\eqref{eq:MRF1} 
can be solved using integer programming (e.g., \CPLEX~\cite{CPLEXwebsite}) for moderate-sized problems, 
although this approach does not scale to large instances.

%%%%%%%%%%%%%%%%%%%%%%%%%%%%%%%%%%%%%%%%%%%%%%%%%%%%%%%%%%%%%%%%%%%%%%%%%%%%%
\myparagraph{Approximate and local MRF inference}
When MRF inference is NP-hard, one has to resort to approximation techniques. 
%, which attempt to efficiently compute an approximate MAP estimate. %We briefly review some of these techniques. 
% Early approximation techniques include 
% greedy methods~\cite{Besag86jrss-MRF}, which start at a given initial guess and iteratively try to converge to 
% a local optimum of the cost function. 
\emph{Move-making algorithms}~\cite{Kolmogorov05pami,Boykov01pami-graphCut},
at each inner iteration, solve a binary MRF using graph cut, while the outer loop attempts to reconcile the binary results into 
a coherent multi-class segmentation. 
%Under technical conditions, these techniques produce ``strong'' local minima, i.e., local minima where no allowed move is able to further reduce the cost.
\emph{Message-passing techniques} adjust the MAP estimate at each node in the MRF via local information exchange between neighboring nodes.
Popular message passing techniques are \emph{Loopy Belief Propagation}~\cite{Weiss01itis-beliefPropagation} (\LBP) and 
 \emph{Tree-Reweighted Message Passing}~\cite{Wainwright05itis} (\TRW).
 Both \LBP and \TRW are not guaranteed to converge in generic MRFs and may oscillate. 
\emph{Linear Programming Relaxations} relax MRF inference to work on continuous labels~\cite{Wainwright05itis,Kleinberg99focs,Chekuri01siam}.
% Related work also provides suboptimality guarantees for LP relaxations~\cite{Kleinberg99focs,Gupta00acm,Chekuri01siam}.
% For binary pairwise MRFs, LP relaxation over the local polytope can be solved 
% efficiently by reformulating it as a maximum flow problem, see the \emph{roof duality} (or QPBO) approach of Rother\setal~\cite{Rother07cvpr-MRF}. 
\emph{Spectral and semidefinite relaxations} 
rephrase MRF inference in terms of a binary (pseudo-boolean) quadratic optimization~\cite{Shi00pami,Olsson08cviu}, 
which can be then relaxed to a convex program (see Section~\ref{sec:standardSDPrelax}). 
% The key insight is that a binary variable $b\in\{0;1\}$ can be reparametrized as $\tilde{b}\in\{-1;+1\}$, and the 
% corresponding domain can be expressed in terms of the quadratic equality constraint $\tilde{b}^2 = 1$. 
% The resulting quadratic problem with quadratic equality constraints can be then relaxed using semidefinite or spectral relaxations~\cite{Olsson08cviu}.
% Spectral relaxation techniques include~\cite{Shi00pami,Olsson08cviu}.
SDP relaxations are known to provide better solutions than spectral methods~\cite{Keuchel03pami,Olsson08cviu} 
 but  the computational cost of general-purpose SDP solvers prevented widespread use of this technique beyond problems with few hundred  
variables~\cite{Keuchel03pami}. 
% In a subsequent work, 
Keuchel\setal~\cite{Keuchel04dagm} reduce the dimensionality of the problem via image preprocessing and superpixel segmentation.
%Concurrently, Torr~\cite{Torr03aistats} proposes the use of SDP relaxations for pixel matching problems.
Heiler\setal~\cite{Heiler05pr} add constraints in the SDP relaxation to enforce priors (e.g., to bound the number of pixels in a class).
Olsson\setal~\cite{Olsson08cviu} develop a spectral subgradient method which is shown to reduce the relaxation gap of spectral relaxations.
%provides a more efficient solution 
Huang\setal~\cite{Huang14icml} use an Alternating Direction Methods of Multipliers to speed up computation, while
 Wang\setal~\cite{Wang16pami} develop a specialized dual solver.

%%%%%%%%%%%%%%%%%%%%%%%%%%%%%%%%%%%%%%%%%%%%%%%%%%%%%%%%%%%%%%%%%%%%%%%%%%%%%
\myparagraph{Application to semantic segmentation}
Semantic segmentation methods assign a semantic label to each ``region'' in an RBG image,
RBG-D image, or 3D model~\cite[Section 6]{Hu18tr-fuses}. 
% Depending on the approach, labels can be assigned to single pixels/voxels, superpixels, or keypoints~\cite{Kappes15ijcv-energyMin}; 
% note that when grouping multiple pixels in superpixels, these are constrained to belong to the same class. 
% Since semantic segmentation is typically modeled as an MRF, the literature review in Sections~\ref{sec:rw-techniques-exact}-\ref{sec:rw-techniques-approx} 
% already covers several work in segmentation, and indeed segmentation (together with depth reconstruction) is a typical benchmark for 
% inference in MRFs, see~\cite{Szeliski08pami-surveyMRF,Blake11book-MRF,Kappes15ijcv-energyMin,Felzenszwalb11pami,Zhu16jcvir} and the references therein.
Traditional approaches for semantic segmentation work by extracting and classifying features in the 
input image, and then enforcing consistency of the classification across pixels,  using MRFs or other models.
Common features include pixel color, histogram of oriented gradients, SIFT, or textons, to mention a few~\cite{Zhu16jcvir}. %~\cite{Thoma16arxiv}.
% Shotton\setal~\cite{Shotton08cvpr} use textons and Random Decision Forests for semantic segmentation.
% Yang\setal~\cite{Yang12pami} use Support Vector Machine (SVM). 
% A \emph{latent SVM} model is used by Felszenzwalb\setal~\cite{Felszenzwalb10pami} to detect objects using deformable part models.
% Winn and Shotton~\cite{Winn06cvpr} use a Conditional Random Field (CRF)\footnote{Conditional Random Fields (CRFs) are a special case of MRFs, where the binary terms, rather than being smoothness priors, are data driven.} to detect and segment objects from their parts;
 % the approach is further generalized by Hoiem\setal~\cite{Hoiem07cvpr}; 
 % Shotton\setal~\cite{Shotton06eccv} use textons within a CRF model for object segmentation.
% LayoutCRF Winn et al. ’06+
% Many other examples: ObjCut Kumar et. al. ’05; 
% Deformable Part Model Felzenszwalb et al.; CVPR ’08; 
% PoseCut Bray et al. ’06,
% LayoutCRF Winn et al. ’06
% Kumar\setal~\cite{Kumar05cvpr} use MRFs to detect and segment objects in an image.
% Bray\setal~\cite{Bray06eccv} concurrently segment and estimate the 3D pose of a human body from multiple views.
 % Higher-order MRF formulations are also used for semantic segmentation, 
More recently, deep convolutional neural networks (CNNs) have become a popular segmentation approach, see~\cite{GarciaGarcia17arxiv}. % and the references therein.
\change{}{State-of-the-art methods, such as \emph{DeepLab}~\cite{Chen18pami-deepLab}, 
refine the results of a deep convolutional network with a fully-connected conditional random field in order to improve the localization accuracy of object boundaries.}
%, which are typically  
% combine the flexibility of CNNs with accuracy of conditional random field models.}
}
%!TEX root = main.tex

\section{Conclusion}
\label{sec:conclusion}

We propose fast optimization techniques to solve two semidefinite relaxations of maximum a posteriori inference in Markov Random Fields (MRFs).
The first technique, named \dars (\emph{Dual Ascent Riemannian Staircase}), provides a scalable solution for the
standard SDP relaxation proposed in the literature. 
The second technique,  named \fuses (\emph{Fast Unconstrained SEmidefinite Solver}), is based on a novel relaxation.
% Both approaches 
% leverage recent results on low-rank SDP solvers via optimization on smooth Riemannian manifolds. 
% In a nutshell,
\omitted{ 
Our contribution is to exploit the geometric structure of (new and old) semidefinite relaxations to obtain fast 
solutions via optimization on smooth Riemannian manifolds. }
We test the proposed approaches in semantic segmentation problems and compare them against state-of-the-art MRF solvers, including move-making 
and message-passing methods. Our experiments show that
(i) \fuses and \dars produce near-optimal solutions, attaining an objective within \percSubopt of the optimum,
%with a suboptimality gap between \tocheck{1-2\%} from the optimum,
(ii) our approaches are remarkably faster than general-purpose SDP solvers, 
while  \fuses is more than two orders of magnitude faster than \dars, % while attaining similar solution quality,
(iii) \fuses is \claimTime  local search methods while being a global solver.
\isExtended{
\change{}{While the evaluation in this paper focuses on the MRF solver (rather than attempting to outperform state-of-the-art deep learning methods in semantic segmentation), we believe \fuses can be used \emph{in conjunction} with existing deep learning methods, as done in~\cite{Chen18pami-deepLab}, to refine the segmentation results.}
}{}

% \addtolength{\textheight}{-5cm}   % This command serves to balance the column lengths

% \bibliographystyle{spmpsci} %{spphys} %{splncs_srt} %NO{spbasic} %{splncs03} % {spmpsci}
\bibliographystyle{IEEEtran}
\bibliography{refs,myRefs}

\isExtended{
\appendix
%!TEX root = main.tex

\section*{Appendix A: Equivalence between Problems~\eqref{eq:MRFmp2} and~\eqref{eq:MRF1}}
\label{sec:proof:eq:MRFmp1}

Here we prove that solving Problem~\eqref{eq:MRFmp2}  is equivalent to solving~\eqref{eq:MRF1}, in the sense that the solution set of a problem
is in 1-to-1 correspondence with the solution set of the other. 
Towards this goal, we show that~\eqref{eq:MRFmp2} can be simply obtained as a reparametrization of~\eqref{eq:MRF1}. 

We first rewrite each node variable $\class_i \in \classSet \doteq \{1,\ldots,\nrClasses\}$ in~\eqref{eq:MRF1} as a vector 
$\Classesmp_i \in \{-1,+1\}^\nrClasses$, such that $\Classesmp_i$ has a single entry equal to $+1$ (all the others are $-1$), and if the $j$-th entry of $\Classesmp_i$ is $+1$, then the corresponding node has label $j$. Each vector $\Classesmp_i \in \{-1,+1\}^\nrClasses$ is a valid label assignment 
as long as there is a unique entry equal to $+1$, or, equivalently, $\ones\tran \Classesmp_i = \cumsum, \quad i=1,\ldots,\nrNodes$.
Using this vector parametrization we  rewrite the unary and binary potentials~\eqref{eq:binary} as:
\bea
\label{eq:potentialsmp}
{E}_i(\Classesmp_i) = \frac{\penaltyTerm_{i}}{2} (1-\ve_{\measuredClass_i}\tran \Classesmp_i) \;,% -1 is disagrees
% \qquad\qquad 
\\
{E}_{ij}(\Classesmp_i, \Classesmp_j) = \frac{\penaltyTerm_{ij}}{2} (\Classesmp_i\tran \Classesmp_j - \nrClasses + 2) \nonumber
\eea
where $\ve_{\measuredClass_i}$ is a vector of all zeros, except the entry in position $\measuredClass_i$ (measured class label for node $i$), which is equal to $+1$.
The reparametrization of the unary potentials in~\eqref{eq:potentialsmp} can be  seen to be the same as~\eqref{eq:binary} by observing that $\ve_{\measuredClass_i}\tran \Classesmp_i = +1$ if $\class_i = \measuredClass_i$ 
or $-1$ otherwise; similarly, the reparametrization of the binary potentials follows from the fact that
$\Classesmp_i\tran \Classesmp_j = \nrClasses$ if $\class_i = \class_j$ or $\nrClasses-2$ otherwise.

Using~\eqref{eq:potentialsmp}, we rewrite Problem~\eqref{eq:MRF1} as: 
\beq
\begin{array}{rl}
\label{eq:MRFmp1}
\displaystyle
\min_{\substack{\Classesmp_i, i =1,\ldots,\nrNodes}}  &
\displaystyle\sum_{i \in \unarySet} - \frac{\penaltyTerm_{i}}{2} \ve_{\measuredClass_i}\tran \Classesmp_i
+ \displaystyle\sum_{(i,j)\in\binarySet} \frac{\penaltyTerm_{ij}}{2} \Classesmp_i\tran \Classesmp_j 
\\
\subject & \Classesmp_i \in \{-1,+1\}^\nrClasses
\\
         & \ones\tran \Classesmp_i = \cumsum, \quad i=1,\ldots,\nrNodes
\end{array}
\eeq
where we dropped the constant terms from~\eqref{eq:potentialsmp} (which are irrelevant for the optimization), and 
where the constraint $\ones\tran \Classesmp_i = \cumsum$ enforces each vector to have at most one entry equal to $+1$
(i.e., we assign a single label to each node).

In order to obtain Problem~\eqref{eq:MRFmp2}, we adopt a more compact notation by stacking  all vectors $\Classesmp_i$, with $i =1,\ldots,\nrNodes$, in a single 
$\nrNodes\nrClasses$-vector $\Classesmp = [\Classesmp_1\tran \; \Classesmp_2\tran \; \ldots \; \Classesmp_\nrNodes\tran]\tran$, and note that the cost
 function~\eqref{eq:MRFmp1}
is quadratic in the entries of $\Classesmp$. Therefore, we rewrite problem~\eqref{eq:MRFmp1} as:
\beq
\label{eq:MRFmp111}
\begin{array}{rl}
\min_{\Classesmp}  &  2 \vb\tran \Classesmp + \Classesmp\tran \MA \Classesmp
\\
\subject & \Classesmp \in \{-1,+1\}^{\nrNodes\nrClasses}, %\quad k = 1,\ldots, \nrNodes\nrClasses 
\\
         & \vu_i\tran \Classesmp = \cumsum, \quad i=1,\ldots,\nrNodes   % \MU \Classesmp = (\nrClasses - 2) \ones_{\nrNodes}
\end{array}
\eeq
where $\MA$ is an $\nrNodes\nrClasses \times \nrNodes\nrClasses$ symmetric block matrix, 
and $\vb$ is an $\nrNodes\nrClasses$-vector, and $\vu_i \doteq \ve_i\tran \kron \ones_\nrClasses\tran$, where $\ve_i$ is an $\nrNodes$-vector which is 
all zero, except the $i$-th entry which is one, $\ones_\nrClasses$ is a $\nrClasses$-vector of ones, 
and $\kron$ is the Kronecker product. The constraint $\vu_i\tran \Classesmp = \cumsum$ simply rewrites the constraint $\ones\tran \Classesmp_i = \cumsum$ in~\eqref{eq:MRFmp1}.
 The reader can also verify by inspection that the following choice of $\MA$ and $\vb$ ensures that the objective in~\eqref{eq:MRFmp111} is the same as~\eqref{eq:MRFmp1}:
\bea
\vb \doteq &
\left\{
\begin{array}{ll}
\,[\vb]_{i} =  -\frac{\penaltyTerm_{i}}{4} \ve_{\measuredClass_i}, & \text{ if } i \in \unarySet 
\\
\, [\vb]_{i}  = \zero_{\nrClasses}, & \text{ otherwise}
\end{array}
\right. 
% \matTwo{
% -\penaltyTerm_{1} \ve_{\measuredClass_2}\tran \\
% -\penaltyTerm_{2} \ve_{\measuredClass_2}\tran \\
% \vdots \\
% -\penaltyTerm_{\nrNodes} \ve_{\measuredClass_\nrNodes}\tran
%  },
\qquad \\
%\\ %\text{ and }
% \begin{array}{c}
\MA \doteq &
\left\{
\begin{array}{ll}
\,[\MA]_{ij} =  -\frac{\penaltyTerm_{ij}}{2} \cdot \eye_{\nrClasses},&  \text{ if } (i,j) \in \binarySet\\
\,[\MA]_{ij} = \zero_{\nrClasses \times \nrClasses}, &\text{ otherwise}
\end{array}
\right. \hspace{-3mm}
\nonumber
% \\
% (\MH \in \Real{\nrNodes \times \nrNodes})
% \end{array}
\eea
where $\vb$ stacks $\nrNodes$ subvectors of size $\nrClasses$, $[\vb]_{i}$ is the $i$-th subvector of $\vb$,
 % in $\vb$ ($\vb$ stacks $\nrNodes$ subvectors of si)
$[\MA]_{ij}$ is the $\nrClasses \times \nrClasses$ block of $\MA$  in block row $i$ and block column $j$, 
and $\eye_{\nrClasses}$ is the identity matrix of size $\nrClasses$. %in position $(i,j)$.

Now we observe that for a scalar $\omega$, we can equivalently write $\omega \in \{-1,+1\}$ as $\omega \in \setdef{\Real{}}{\omega^2 = 1}$.
Moreover, we note that the diagonal of the matrix $\Classesmp\Classesmp\tran$ contains the squares of every entry of $\Classesmp$.
Combining these two observations, we rewrite problem~\eqref{eq:MRFmp1} equivalently as:
\beq
\begin{array}{rl}
\min_{\Classesmp}  & \Classesmp\tran \MA \Classesmp + 2 \vb\tran \Classesmp
\\
\subject & \diag{\Classesmp\Classesmp\tran} = \ones_{\nrNodes\nrClasses}, %\quad k = 1,\ldots, \nrNodes\nrClasses 
\\
         & \vu_i\tran \Classesmp = \cumsum, \quad i=1,\ldots,\nrNodes   % \MU \Classesmp = (\nrClasses - 2) \ones_{\nrNodes}
\end{array}
\eeq
which is the same as~\eqref{eq:MRFmp2}, concluding the proof.
% where $\MA$ and $\vb$ are a suitable symmetric matrix and a suitable vector collecting the coefficients of the 
% binary terms and the unary terms in~\eqref{eq:binary}, respectively;
% $\diag{\Classesmp\Classesmp\tran}$ is the diagonal of the matrix $\Classesmp\Classesmp\tran$, 
% and 
% Moreover, they stack all vectors $\Classesmp_i$, $i =1,\ldots,\nrNodes$, in a single 
% $\nrNodes\nrClasses$-vector $\Classesmp = [\Classesmp_1\tran \; \Classesmp_2\tran \; \ldots \; \Classesmp_\nrNodes\tran]\tran$. %\in 
% The second step towards a standard SDP relaxation consists in rephrasing~\eqref{eq:MRFmp1} as a quadratic problem with quadratic (and linear) equality constraints.
%expressing the binary variables as continuous variable with quadratic equality constraints.

%!TEX root = main.tex

%--!TEX root = main.tex

\section*{Appendix B: Proof of Proposition~\ref{prop:dars:guarantees}}
\label{sec:proof:dars:guarantees}

Here we assume that the dual ascent iterations of \dars converged to a value $\vlambda^\star$ (i.e., the dual iterations reach a solution where the gradient in~\eqref{eq:gradAscent} is zero) and prove the two claims of Proposition~\ref{prop:dars:guarantees}.

Let us start by proving the first claim. 
Let us consider the latest dual ascent iteration, where the dual variable is $\vlambda^\star$ 
and we need to solve~\eqref{eq:RRT2} to compute an updated primal variable. Moreover, assume that the Riemannian Staircase, 
which solves~\eqref{eq:RRT2} for increasing rank $r$, finds a (column) rank-deficient second-order critical point~\eqref{eq:RRT2},
 which we call $\MR^\star$. Then, Proposition~\ref{prop:boumal} guarantees that the matrix $\MY^\star \doteq (\MR^\star)(\MR^\star)\tran$
 is an optimal solution for the SDP~\eqref{eq:primalDescent}. Moreover, since~\eqref{eq:primalDescent} simply rewrites~\eqref{eq:DA3}, 
 $\MY^\star$ is also an optimal solution for~\eqref{eq:DA3}:
\bea
\label{eq:dual0}
d(\vlambda^\star) \doteq
&  \min_{\MY}  g(\MY) + \sum_{i=1}^{\nrNodes} \vlambda^\star_i (\trace{\MU_i \MY} - \nrClasses + 2) = \\
& g(\MY^\star) + \sum_{i=1}^{\nrNodes} \vlambda^\star_i (\trace{\MU_i \MY^\star} - \nrClasses + 2)
\eea
Now we are left to prove that $\MY^\star$ is also an optimal solution for the standard SDP relaxation~\eqref{eq:SDPstandard}. 
Towards this goal, we observe that~\eqref{eq:dual0} is the \emph{dual function}~\cite[Section 5.1.2]{Boyd04book} of the optimization problem~\eqref{eq:DA1}, 
and standard results in optimization theory~\cite[Section 5.1.3]{Boyd04book} 
 ensure that $d(\vlambda)$ is a lower bound for the optimal cost of~\eqref{eq:DA1} for any $\vlambda$. 
 Since~\eqref{eq:DA1} is the same as~\eqref{eq:SDPstandard} and both attain the optimal objective $\fsdpmp$, it then holds
\beq
\label{eq:dual1}
d(\vlambda^\star) \leq \fsdpmp
\eeq
Since we assumed that the dual iterations reached a solution where the gradient in~\eqref{eq:gradAscent} is zero, 
and since the gradient with respect to the $i$-th dual variable is 
$\nabla \calL_{\vlambda_i}(\MY,\vlambda) = \trace{\MU_i \MY} - \nrClasses + 2$, it follows $\trace{\MU_i \MY^\star} = \nrClasses - 2$, for $i=1,\ldots,\nrNodes$.
This implies that $d(\vlambda^\star) = \trace{\ML \MY^\star}$\footnote{The indicator functions also disappear since $\MY^\star$ 
satisfies the corresponding constraints by construction.} and~\eqref{eq:dual1} becomes:
\beq
\label{eq:dual2}
\trace{\ML \MY^\star} \leq \fsdpmp
\eeq
Now we note that since $\MY^\star$ satisfies $\trace{\MU_i \MY^\star} = \nrClasses - 2$, for $i=1,\ldots,\nrNodes$, it is a feasible solution
for~\eqref{eq:SDPstandard}. Therefore by optimality of  $\fsdpmp$ it holds:
\beq
\label{eq:dual3}
\fsdpmp \leq \trace{\ML \MY^\star} 
\eeq
Combining~\eqref{eq:dual2} and~\eqref{eq:dual3} it follows $\trace{\ML \MY^\star} = \fsdpmp$, i.e., $\MY^\star$ attains the optimal objective in~\eqref{eq:SDPstandard}, proving the first claim.

To prove the second claim, we observe that~\eqref{eq:SDPstandard} is a relaxation of~\eqref{eq:MRFmp3}, hence it holds:
\beq
\label{eq:gap1}
\fsdpmp \leq \foptmp   \iff  -\foptmp \leq -\fsdpmp 
\eeq
Adding $\froundedmp$ to both sides of~\eqref{eq:gap1} yields the desired inequality $\froundedmp - \foptmp \leq \froundedmp - \fsdpmp$, concluding the proof 
of Proposition~\ref{prop:dars:guarantees}.

%!TEX root = main.tex

\section*{Appendix C: Proof of Proposition~\ref{prop:binaryMat}}
\label{sec:proof:prop:binaryMat}

Here we prove that solving~\eqref{eq:binaryMat} is equivalent to solving~\eqref{eq:MRF1}, in the sense that the solution set of a problem
is in 1-to-1 correspondence with the solution set of the other.
Towards this goal, we show that~\eqref{eq:binaryMat} can be simply obtained as a reparametrization of~\eqref{eq:MRF1}. 
Let us first consider the constraints in~\eqref{eq:binaryMat}, and note that the equality $\diag{\MX \MX\tran} = \ones_{\nrNodes}$ can 
be written explicitly as $\| \MX_i \|^2 = 1$ for $i=1,\ldots,\nrNodes$, where $\MX_i$ is the $i$-th row of $\MX$; since the $i$-th row $\MX_i$ 
 describes the label assignment of node $i$ and is a binary vector, the constraint $\| \MX_i \|^2 = 1$ enforces $\MX_i$ to have a unique
 nonzero element, hence matching the requirement of assigning a unique label to each node in~\eqref{eq:MRF1}.  

We are only left to prove that the two objective functions are equivalent. 
For this purpose, we observe that the objective in~\eqref{eq:MRF1}, 
with the choice of potentials in~\eqref{eq:binary}, can be written as a function of the rows of $\MX$ as:
\bea
E(\MX) = \sum_{i \in \unarySet} \penaltyTerm_{i} (1 - \ve_{\measuredClass_i}\tran \MX_i\tran)  
+ \sum_{(i,j)\in\binarySet} \penaltyTerm_{ij} (1 - \MX_i \MX_j\tran) 
\\
\equiv
\sum_{i \in \unarySet} - \penaltyTerm_{i} \ve_{\measuredClass_i}\tran \MX_i\tran 
+ \sum_{(i,j)\in\binarySet} -\penaltyTerm_{ij} \MX_i \MX_j\tran 
\eea
where $\equiv$ denotes equality up to constant (irrelevant for the optimization). 
% , we note that the set of constraints in~\eqref{eq:binaryMat} enforces 
% each node to have a unique label, since $\diag{\MX \MX\tran} = \ones_{\nrNodes} \Leftrightarrow \MX \MX\tran$
% , in both problems enforces that each node is assigned 
% a unique label and that all nodes are labeled; therefore we only have 
% we have to prove that the objective functions are 
Using definitions~\eqref{eq:GHdefs}, 
by inspection we can verify that 
$
\trace{\MG \MX\tran} = \sum_{i \in \unarySet} - \penaltyTerm_{i} \ve_{\measuredClass_i}\tran \MX_i\tran
$ 
and that 
\beal
\trace{\MX\tran \MH \MX} = \trace{\MH \MX \MX\tran} = \textstyle\sum_{\substack{i = 1,\ldots,\nrNodes \\ j = 1,\ldots,\nrNodes}} \MH_{ij} [\MX \MX\tran]_{ij} 
\\ = 
\textstyle\sum_{\substack{i = 1,\ldots,\nrNodes \\ j = 1,\ldots,\nrNodes}} \MH_{ij} \MX_i \MX_j\tran 
\overset{\tiny \text{using}~\eqref{eq:GHdefs}}{=}  \sum_{(i,j)\in\binarySet} -\penaltyTerm_{ij} \MX_i \MX_j\tran \nonumber
\eeal
which demonstrates the equality (up to constant) of the objective functions in~\eqref{eq:MRF1} and~\eqref{eq:binaryMat}, concluding the proof.
}{}
\end{document}